\documentclass[twoside,11pt]{article}
\usepackage[nohyperref,preprint]{jmlr2e}

\usepackage{lastpage}
\jmlrheading{23}{2026}{1-\pageref{LastPage}}{1/21; Revised 5/22}{9/22}{21-0000}{Chen and Thi\'{e}ry}

\ShortHeadings{Flow Ensemble Filter}{Chen and Thi\'{e}ry}
\firstpageno{1}

\usepackage[hypertexnames=false]{hyperref}

\graphicspath{{./}}

\usepackage{amsfonts, amsmath, amssymb}
\usepackage[
        nameinlink,
    ]{cleveref}
\usepackage{natbib}
\usepackage[dvipsnames]{xcolor}
\usepackage{colortbl}
\usepackage{mathtools,empheq}
\usepackage{algorithm,algorithmic}
\usepackage{bbm,dsfont}

\makeatletter
\@ifundefined{theorem}{\newtheorem{theorem}{Theorem}}{}
\@ifundefined{lemma}{\newtheorem{lemma}[theorem]{Lemma}}{}
\@ifundefined{proposition}{\newtheorem{proposition}[theorem]{Proposition}}{}
\@ifundefined{corollary}{\newtheorem{corollary}[theorem]{Corollary}}{}
\@ifundefined{remark}{\newtheorem{remark}[theorem]{Remark}}{}
\@ifundefined{assumption}{\newtheorem{assumption}[theorem]{Assumption}}{}
\@ifundefined{proof}{%
  \providecommand{\BlackBox}{\rule{1.5ex}{1.5ex}}
  \newenvironment{proof}{\par\noindent{\bf Proof\ }}{\hfill\BlackBox\\[2mm]}
}{}

\newcommand{\enflow@retype}[1]{%
  \expandafter\let\csname enflow@orig@#1\expandafter\endcsname\csname #1\endcsname
  \expandafter\let\csname enflow@origend@#1\expandafter\endcsname\csname end#1\endcsname
  \renewenvironment{#1}[1][]{%
    \crefalias{theorem}{#1}%
    \def\enflow@thmtitle{##1}%
    \ifx\enflow@thmtitle\@empty
      \csname enflow@orig@#1\endcsname
    \else
      \csname enflow@orig@#1\endcsname[##1]%
    \fi
  }{\csname enflow@origend@#1\endcsname}%
}
\enflow@retype{lemma}
\enflow@retype{proposition}
\enflow@retype{corollary}
\enflow@retype{remark}
\enflow@retype{assumption}
\makeatother

\crefname{theorem}{Theorem}{Theorems}
\Crefname{theorem}{Theorem}{Theorems}
\crefname{lemma}{Lemma}{Lemmas}
\Crefname{lemma}{Lemma}{Lemmas}
\crefname{proposition}{Proposition}{Propositions}
\Crefname{proposition}{Proposition}{Propositions}
\crefname{corollary}{Corollary}{Corollaries}
\Crefname{corollary}{Corollary}{Corollaries}
\crefname{assumption}{Assumption}{Assumptions}
\Crefname{assumption}{Assumption}{Assumptions}
\crefname{remark}{Remark}{Remarks}
\Crefname{remark}{Remark}{Remarks}
\crefname{remark}{Remark}{Remarks}
\Crefname{remark}{Remark}{Remarks}
\usepackage{graphicx,array,tabularray,multirow,geometry,booktabs}
\usepackage{longtable}
\usepackage{placeins}

\newcommand{\papertitle}{Improving Ensemble Filters with Flow Matching}
\newcommand{\paperauthorone}{Haoyuan Chen}
\newcommand{\paperauthoroneemail}{hchen@nus.edu.sg}
\newcommand{\paperauthortwo}{Alexandre Thi\'{e}ry}
\newcommand{\paperauthortwoemail}{a.h.thiery@nus.edu.sg}
\newcommand{\paperaffiliation}{Department of Statistics and Data Science\\National University of Singapore\\6 Science Drive 2, Singapore 117546}

\newcommand{\algref}[1]{\hfill\textcolor{blue}{\(\triangleright\) #1}}

\definecolor{figgray}{HTML}{808080}
\definecolor{figblue}{HTML}{2E75B6}
\definecolor{figorange}{HTML}{ED7D31}
\definecolor{figtruth}{HTML}{A02334}
\newcommand{\mkfcst}{\textcolor{figgray}{$\bm{\times}$}}
\newcommand{\mksrc}{\textcolor{figblue}{$\bullet$}}
\newcommand{\mkend}{\textcolor{figorange}{$\bullet$}}
\newcommand{\mktruth}{\textcolor{figtruth}{$\star$}}

\newcommand{\bm}[1]{\boldsymbol{#1}}

\DeclareMathOperator*{\argmin}{arg\,min} 

\newcommand{\makebfcommands}[1]{%
  \forcsvlist{\makebfcommand}{#1}%
}

\newcommand{\makebfcommand}[1]{%
  \expandafter\newcommand\csname bf#1\endcsname{\mathbf{#1}}%
}

\makebfcommands{a,b,c,d,e,f,g,h,i,j,k,l,m,n,o,p,q,r,s,t,u,v,w,x,y,z,A,B,C,D,E,F,G,H,I,J,K,L,M,N,O,P,Q,R,S,T,U,V,W,X,Y,Z}

\newcommand{\makecalcommands}[1]{%
  \forcsvlist{\makecalcommand}{#1}%
}

\newcommand{\makecalcommand}[1]{%
  \expandafter\newcommand\csname cal#1\endcsname{\mathcal{#1}}%
}

\makecalcommands{A,B,C,D,E,F,G,H,I,J,K,L,M,N,O,P,Q,R,S,T,U,V,W,X,Y,Z}

\newcommand{\makebfgreek}[1]{%
  \forcsvlist{\makebfgreekletter}{#1}%
}

\newcommand{\makebfgreekletter}[1]{%
  \expandafter\newcommand\csname bf#1\endcsname{\boldsymbol{\csname #1\endcsname}}%
}

\makebfgreek{alpha,beta,gamma,delta,epsilon,varepsilon,zeta,eta,theta,vartheta,iota,kappa,lambda,mu,nu,xi,pi,varpi,rho,varrho,sigma,varsigma,tau,upsilon,phi,varphi,chi,psi,omega,Gamma,Delta,Theta,Lambda,Xi,Pi,Sigma,Upsilon,Phi,Psi,Omega}

\usepackage{enumitem}
\setlistdepth{9} 

\setlist[itemize,1]{label=\textbullet}
\setlist[itemize,2]{label=\textendash}
\setlist[itemize,3]{label=\textasteriskcentered}
\setlist[itemize,4]{label=\textperiodcentered}
\setlist[itemize,5]{label=$\diamond$}
\setlist[itemize,6]{label=$\star$}
\setlist[itemize,7]{label=$\triangleright$}
\setlist[itemize,8]{label=$\circ$}
\setlist[itemize,9]{label=$\cdot$}

\setlist[enumerate,1]{label=\arabic*.}
\setlist[enumerate,2]{label=(\alph*)}
\setlist[enumerate,3]{label=\roman*.}
\setlist[enumerate,4]{label=\Alph*.}
\setlist[enumerate,5]{label=\arabic*)}
\setlist[enumerate,6]{label=\alph*)}
\setlist[enumerate,7]{label=(\roman*)}
\setlist[enumerate,8]{label=(\Alph*)}
\setlist[enumerate,9]{label=\arabic*:}

\setlist[description]{font=\normalfont\bfseries}

\begin{document}
\title{\papertitle}
\author{\name \paperauthorone \email \paperauthoroneemail \\       
\addr \paperaffiliation       
\AND       
\name \paperauthortwo \email \paperauthortwoemail \\       
\addr \paperaffiliation}
\editor{My editor}
\maketitle
\begin{abstract}%
Data assimilation estimates a dynamical state from partial and noisy
observations. Classical ensemble filters are efficient but 
restrict analysis updates through finite sample covariance 
and affine Gaussian distribution. 
We introduce the Flow Ensemble Filter (FlowEF), 
which uses conditional flow matching to transport 
the forecast ensemble from a classical baseline filter 
to an analysis ensemble. 
FlowEF uses a localized Gaussian source during training, 
transports forecast ensemble members from a baseline filter at deployment, 
and conditions its velocity field on ensembles from that baseline filter and the observation. 
The proposed model therefore learns a nonlinear update while 
mapping each baseline ensemble independently. 
For sparsely observed dynamical systems, 
FlowEF improves both deterministic and probabilistic metrics 
over all four classical ensemble filters. 
It also achieves the best performance 
among the state-of-the-art generative data assimilation models.
\end{abstract}

\begin{keywords}
  Data assimilation, Flow matching, Ensemble filtering
\end{keywords}

\section{Introduction}
\label{sec:introduction}

Data assimilation (DA) estimates the evolving state of a dynamical system from a numerical model with partial and noisy observations. At each assimilation time,
the forecast distribution is updated by the new observation to form an analysis
distribution. In geophysical systems, the state dimension is
large, observations cover only part of the state, and exact filtering updates
are not available in closed form \citep{carrassi2018review}.
Practical methods therefore work with finite representations of uncertainty,
and the output of interest is often an analysis ensemble rather than a single
state estimate.

Two families of methods are widely used in practical applications. First, Ensemble DA methods use a finite collection of forecast particles to represent predictive uncertainty and update this collection after each new observation. For example, Ensemble Kalman filters (EnKFs) \citep{evensen1994sequential} and
their variants \citep{bishop2001adaptive,sakov2008deterministic,hunt2007efficient} provide computationally tractable analysis
updates in high-dimensional systems.
Secondly, variational methods instead define the
analysis through prior-plus-observation objectives and covariance models \citep{lorenc1986analysis,ledimet1986variational}. However, the majority of these algorithms are fundamentally based on locally linear or Gaussian assumptions, rendering them inaccurate when these conditions are not met.

Recent learned DA methods relax these restrictions by fitting
parts of the analysis rule to empirical or simulated data. Some keep the classical update and learn a
correction, gain, localization rule, or deterministic analysis map
\citep{revach2022kalmannet}.
Others define a stochastic analysis rule through a conditional score,
transport map, flow, or related sampler \citep{rozet2023sda, bao2024ensf,aljarrah2024otf,wang2025fbf, chen2025flowdas}.

For stochastic learned DA, a key design choice is how the current forecast ensemble enters the analysis update. It can be supplied as conditioning information to a learned sampler, or it can also be used to construct the distribution from which the sampler is initialized. In the latter case, the initial particles already encode selected information about the current forecast uncertainty before the learned transformation is applied. This provides a natural way to anchor the learned analysis update to the forecast ensemble available at each assimilation time.

We propose the Flow Ensemble Filter (FlowEF), a learned nonlinear analysis update that augments a classical ensemble filter. At each assimilation time, FlowEF constructs a source distribution from the current forecast ensemble and transports particles from this source to an analysis ensemble through a velocity field learned by conditional flow matching \citep{lipman2023flow}. The velocity field is conditioned on the forecast ensemble, the analysis ensemble obtained from the classical filter, and the current observation. The classical filter therefore provides a forecast-informed reference update, while the learned transport allows the resulting analysis update to depart from the affine and Gaussian structure underlying classical ensemble methods.

We make four contributions. First, we formulate the ensemble analysis step as a conditional transport problem with a source distribution constructed from the current forecast ensemble and with the classical-filter analysis incorporated as additional conditioning information. Second, we introduce a forecast-informed Gaussian source with localized covariance and an observation residual that is updated along the transport. Third, we establish a Wasserstein error bound that separates the contributions of the principal approximation errors in the proposed analysis update. Fourth, across sparsely observed Lorenz--96, Kuramoto--Sivashinsky, and Kolmogorov flow experiments, FlowEF yields lower RMSE and marginal CRPS than tuned classical filters and, where direct comparisons are available, the learned DA baselines considered here.

\section{Preliminaries}
\label{sec:preliminary}

\subsection{State-space model}
\label{sec:state-space-model}

We consider a discrete-time state-space model
\begin{align}
    \bm{x}_t &= f(\bm{x}_{t-1}) + \bm{\eta}_t,
    \qquad
    \bm{y}_t = h(\bm{x}_t) + \bm{\epsilon}_t.
    \label{eq:ssm}
\end{align}
where $\bm{x}_t\in\mathbb{R}^{d_x}$ is the latent state, $\bm{y}_t\in\mathbb{R}^{d_y}$ is the observation, $f$ is the dynamics, and $h$ is the observation operator. Unless otherwise stated, the noise terms $\bm{\eta}_t$ and $\bm{\epsilon}_t$ are additive Gaussian and we write $p_t(\bm{y}_t\mid \bm{x})$ for the observation likelihood.

\subsection{Bayesian forecast and analysis laws}
\label{sec:bayesian-filtering}

We denote $\mathbb{P}$ as the underlying probability measure
and $\mathbb{E}$ as expectation with respect to it. At time $t$, the forecast and analysis probability distributions are:
\begin{align}
    \pi_t^f(d\bm{x})
    =
    \mathbb{P}(\bm{x}_t\in d\bm{x}\mid \bm{y}_{1:t-1}),
    \qquad
    \pi_t^a(d\bm{x})
    =
    \mathbb{P}(\bm{x}_t\in d\bm{x}\mid \bm{y}_{1:t}).
\end{align}
When densities exist, we use the same notation for their densities. Bayes' rule then gives
\begin{align}
    \pi_t^a(\bm{x})
    =
    \frac{
        p_t(\bm{y}_t\mid \bm{x})\pi_t^f(\bm{x})
    }{
        \int p_t(\bm{y}_t\mid \bm{x}')\pi_t^f(\bm{x}')\,d\bm{x}'
    }.
    \label{eq:bayes-analysis}
\end{align}
In the linear-Gaussian setting, the forecast law is characterized by its mean and covariance; outside this setting, these moments generally do not determine the predictive distribution.

\subsection{Ensembles and analysis rules}
\label{sec:ensemble-da}

Ensemble filters approximate forecast and analysis laws by finite particle
collections \citep{evensen2003ensemble}. Let the analysis ensemble from
the previous time $t-1$ be
\begin{align}
    \bm{X}_{t-1}^a
    =
    \big[
        \bm{x}_{t-1}^{a,(1)},\ldots,\bm{x}_{t-1}^{a,(N)}
    \big]
    \in\mathbb{R}^{d_x\times N},
\end{align}
The forecast step propagates each member through the dynamics,
\begin{align}
    \bm{x}_t^{f,(i)}
    =
    f(\bm{x}_{t-1}^{a,(i)})+\bm{\eta}_t^{(i)},
    \quad i=1,\ldots,N,
    \quad
    \bm{X}_t^f
    =
    \big[
        \bm{x}_t^{f,(1)},\ldots,\bm{x}_t^{f,(N)}
    \big]
    \in\mathbb{R}^{d_x\times N},
    \label{eq:forecast-step}
\end{align}
The corresponding empirical forecast law of the filter is
$\widehat{\pi}_t^{f}
=
\frac{1}{N}\sum_{i=1}^N \delta_{\bm{x}_t^{f,(i)}}$.
The empirical mean and covariance of the forecast ensemble are
\begin{align}
    \bar{\bm{x}}_t^f
    =
    \frac{1}{N}\sum_{i=1}^N \bm{x}_t^{f,(i)},
    \qquad
    \bm{P}_t^f
    =
    \frac{1}{N-1}
    \sum_{i=1}^N
    (\bm{x}_t^{f,(i)}-\bar{\bm{x}}_t^f)
    (\bm{x}_t^{f,(i)}-\bar{\bm{x}}_t^f)^\top,
    \label{eq:forecast-moments}
\end{align}
with $\bar{\bm{x}}_t^a$ and $\bm{P}_t^a$ defined analogously. Ensemble Kalman methods base their analysis updates on these moments, typically with localization and inflation in high-dimensional settings. The full ensemble, however, retains information not captured by its first two empirical moments.

After observing $\bm{y}_t$, the analysis step returns an analysis ensemble
\begin{align}
    \bm{X}_t^a
    =
    \big[
        \bm{x}_t^{a,(1)},\ldots,\bm{x}_t^{a,(N)}
    \big]
    \in\mathbb{R}^{d_x\times N},
    \qquad
    \widehat{\pi}_t^{a}
    =
    \frac{1}{N}\sum_{i=1}^N \delta_{\bm{x}_t^{a,(i)}}.
\end{align}
Sequential ensemble filtering alternates the forecast and analysis steps,
\begin{align}
    \bm{X}_{t-1}^a
    \xrightarrow{\;\mathrm{forecast}\;}
    \bm{X}_t^f
    \xrightarrow{\;\mathrm{analysis}\;}
    \bm{X}_t^a .
\end{align}
Different ensemble methods share this recursion but differ in the analysis rule. We write the finite ensemble analysis step as
\begin{align}
    \bm{X}_t^a
    \sim
    \mathcal{A}(\cdot\mid \bm{X}_t^f,\bm{y}_t),
    \label{eq:analysis-operator}
\end{align}
where $\mathcal{A}$ denotes a possibly stochastic analysis rule; deterministic updates are included as a degenerate case. Fixed ingredients such as the dynamics, observation model, localization scheme, and filter hyperparameters are left implicit.
\section{Related work}
\label{sec:analysis-rules-related-methods}
A DA method specifies how the forecast ensemble and the new
observation are turned into an analysis ensemble.

\paragraph{Classical ensemble filters}
Classical ensemble filters construct the analysis rule $\mathcal{A}(\cdot\mid \bm{X}_t^f,\bm{y}_t)$ from forecast moments and anomalies.
Under the local linear-Gaussian analysis model
\begin{align}
    \bm{y}_t
    =
    \bm{H}_t\bm{x}_t+\bm{\epsilon}_t,
    \qquad
    \bm{\epsilon}_t\sim\mathcal{N}(\bm{0},\bm{R}),
    \label{eq:linear-gaussian-analysis-model}
\end{align}
a Gaussian forecast approximation is updated by an affine Kalman map
\citep{kalman1960new}. Ensemble Kalman filters (EnKFs) replace the forecast mean and
covariance in this map by ensemble estimates
\citep{evensen1994sequential}. For nonlinear observation operators, ensemble implementations typically use sample covariances between the state and predicted observations, or local linearizations of $h$.

The stochastic EnKF \citep{evensen2003ensemble} uses perturbed observations, whereas deterministic variants update the ensemble mean and transform the forecast anomalies. These include the ensemble transform Kalman filter (ETKF) \citep{bishop2001adaptive}, deterministic ensemble Kalman filter (DEnKF) \citep{sakov2008deterministic}, and local ensemble transform Kalman filter (LETKF) \citep{hunt2007efficient}. A representative update is
\begin{align}
    \bm{X}_t^a
    =
    \bar{\bm{x}}_t^a\bm{1}_N^\top
    +
    \bm{A}_t^f\bm{T}_t,
    \qquad
    \bm{A}_t^f
    =
    \bm{X}_t^f-\bar{\bm{x}}_t^f\bm{1}_N^\top ,
    \label{eq:square-root-update-v2}
\end{align}
where $\bm{T}_t$ is an ensemble-space transform whose construction depends on the filter.
In spatially extended systems, localization and inflation are commonly used to stabilize finite-ensemble covariance estimates: localization suppresses poorly estimated long-range dependencies or performs local analyses \citep{gaspari1999construction}, while inflation mitigates ensemble underdispersion \citep{anderson2009spatially}.

\paragraph{Variational and hybrid methods}
Variational methods define the analysis through an optimization problem rather than a direct ensemble transform. 3DVar \citep{lorenc1986analysis} and 4DVar
\citep{ledimet1986variational}
minimize prior-plus-observation mismatch over states or trajectories.
Ensemble and hybrid variants use ensemble subspaces or ensemble covariance estimates, often in combination with static background covariances.

\paragraph{Deterministic learned analysis rules}
Many ML-enhanced DA methods retain the conventional forecast--analysis structure while learning part of the analysis update. Learned components include model-error corrections, gain estimators
\citep{revach2022kalmannet,chen2022autodifferentiable}, localization rules, and residual corrections. Other approaches learn the analysis map more directly from forecast and observational information
\citep{fablet2021learning,revach2022kalmannet,boudier2023data,bach2025learning}.
Conditional on their inputs, these updates are deterministic; uncertainty is represented through the ensemble or distributional quantities they produce rather than through sampling within the analysis map.

\paragraph{Stochastic learned analysis rules}
Stochastic learned analysis rules instead generate an analysis particle from an auxiliary random variable. Let $\bm{z}_{t,0} \sim q_{t,0}$ denote the starting variable and let
$\bm{c}_t=C(\bm{X}_t^f,\bm{y}_t)$ collect the forecast and observational information supplied to the learned update. A generic sampled analysis update can be written as
\begin{align}
    \bm{z}_{t,0}\sim q_{t,0},
    \qquad
    \hat{\bm{x}}_t^a
    =
    T_\theta(\bm{z}_{t,0};\bm{c}_t).
    \label{eq:generic-learned-stochastic-analysis-v2}
\end{align}
Score and diffusion DA methods use this form through conditional scores or guided sampling dynamics. Score-based DA (SDA) \citep{rozet2023sda} and ScoreFilter \citep{bao2024scorefilter} use score-based conditional
sampling, while the ensemble score filter
(EnSF) \citep{bao2024ensf} and iterative EnSF (IEnSF) \citep{zhang2025iensf} estimate the required score information from ensembles without training.
Normalizing-flow filters instead learn invertible conditional transformations \citep{wang2025fbf}.

\paragraph{Transport and flow analysis rules}
Transport and flow methods represent the stochastic update as a map from a source law to an analysis law.
Flow-based methods parameterize this map through a continuous-time dynamics. Introducing pseudo-time $\tau\in[0,1]$, distinct from the assimilation time $t$, let $\bm{z}_{t,\tau}$ satisfy
\begin{align}
    \frac{d\bm{z}_{t,\tau}}{d\tau}
    =
    \bm{v}_\theta(\bm{z}_{t,\tau},\tau;\bm{c}_t),
    \qquad
    \tau\in[0,1],
    \label{eq:generic-transport-ode-v2}
\end{align}
with $\bm{z}_{t,0}\sim q_{t,0}$.
Integrating this dynamics defines
$T_\theta(\bm{z}_{t,0};\bm{c}_t)=\bm{z}_{t,1}$.
Flow matching, stochastic interpolants, and related samplers learn such continuous or amortized transports
\citep{lipman2023flow,transue2025enff,chen2025flowdas,andrae2026daisi}.
Coupling and optimal-transport filters \citep{aljarrah2024otf,aljarrah2025aotf} instead construct an analysis map directly from a coupling or optimality criterion.

An important design choice in stochastic transport methods is the source law. It may be fixed across assimilation times or constructed from the current forecast ensemble:
\begin{align}
    q_{t,0}=q_0
    \qquad\text{versus}\qquad
    q_{t,0}=q_{t,0}(\cdot\mid \bm{X}_t^f).
    \label{eq:starting-law-comparison-v2}
\end{align}
With a fixed source, dependence on the current forecast ensemble enters through the learned transformation, for example through conditioning or forecast-dependent score estimates. A forecast-derived source instead incorporates selected forecast information before the transport begins. The distinction therefore concerns where forecast information enters the stochastic analysis rule, rather than whether the forecast is used.
FlowEF takes the latter approach. It constructs a source distribution from the current forecast ensemble and transports particles through a learned continuous-time flow conditioned on the forecast ensemble, the corresponding baseline analysis ensemble, and the observation. The source families, conditioning variables, and supervised training construction are specified in \Cref{sec:flow-da}.

\section{FlowEF}
\label{sec:flow-da}

FlowEF represents each analysis update as a conditional transport from
forecast-informed starting particles to analysis particles. At assimilation time $t$, a classical ensemble filter baseline $B$ provides the forecast and analysis ensembles $\{ \bm{X}_t^f, \bm{X}_t^a \}$ together with the current observation $\bm{y}_t$. FlowEF defines a starting law from the forecast ensemble,
\begin{align}
    \bm{z}_{t,0}^{(i)}
    \sim
    q_{t,0}(\cdot\mid \bm{X}_t^f),
    \qquad i=1,\ldots,N,
    \label{eq:flowda-starting-particles}
\end{align}
and transports them through a learned pseudo-time dynamics
\begin{align}
    \frac{d\bm{z}_{t,\tau}^{(i)}}{d\tau}
    =
    \bm{v}_\theta(\bm{z}_{t,\tau}^{(i)},\tau;\bm{c}_{t,\tau}),
    \qquad
    \tau\in[0,1],
    \label{eq:flowda-analysis-ode}
\end{align}
where the conditioning variables $\bm{c}_{t,\tau}$ comprises static features $\bm{c}_t^{\mathrm{static}}=C(\bm{X}_t^f,\bm{X}_t^a,\bm{y}_t)$, fixed throughout the transport, and a pseudo-time-dependent dynamic field $\bm{c}_{t,\tau}^{\mathrm{dyn}}$; both are defined in \Cref{sec:cond}.
The transport endpoints $\bm{z}_{t,1}^{(i)}$ form the raw analysis ensemble and are subsequently calibrated by inflation (\Cref{sec:infl}) to obtain the FlowEF analysis particles
$\hat{\bm{x}}_t^{a,(i)}$.
Together, the starting law, conditioning variables, and velocity field define the FlowEF analysis rule
$\mathcal{A}(\cdot\mid\bm{X}_t^f,\bm{X}_t^a,\bm{y}_t)$.
We present the diagram of FlowEF in \Cref{fig:flowef-diagram}
and summarize the algorithm in \Cref{alg:flowda}.
Implementation details are given in \Cref{sec:flowda-implementation-details}.

\begin{figure}[!htb]
\centering
\includegraphics[width=\linewidth]{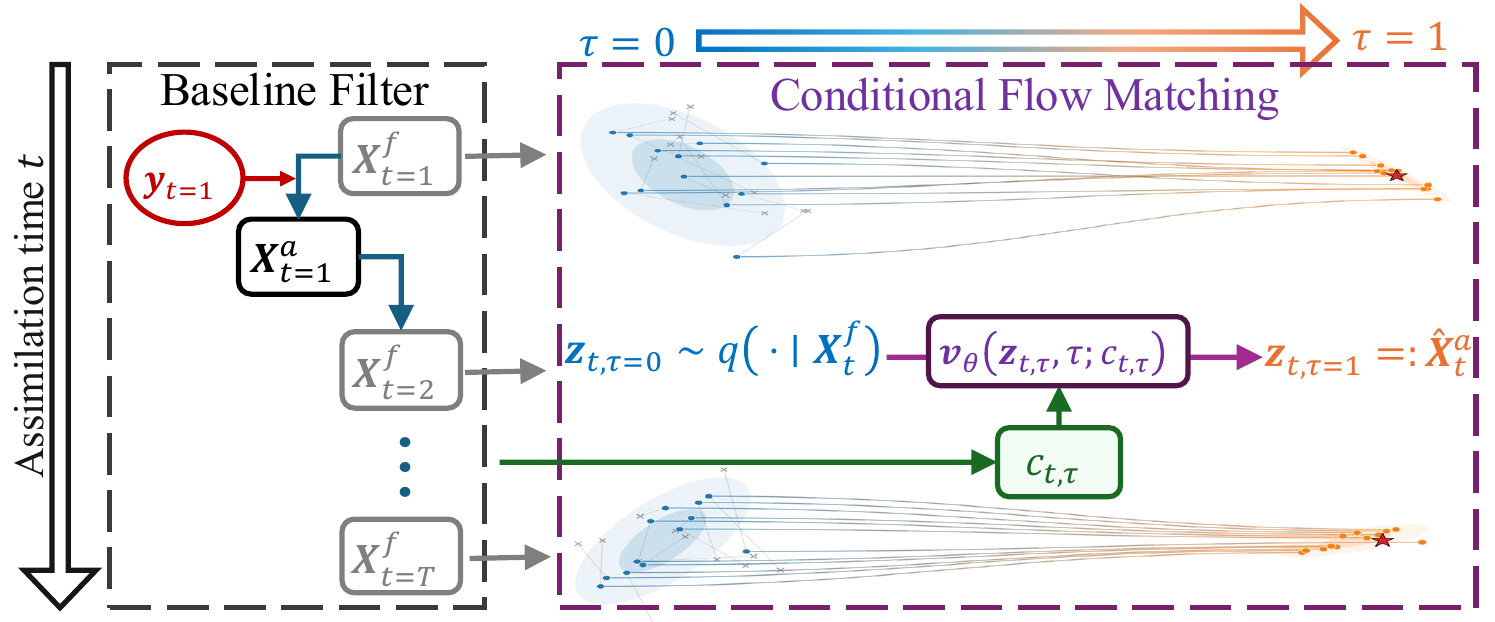}
\caption{\small Schematic diagram of FlowEF. Gray crosses (\mkfcst): forecast members $\bm{x}_t^{f,(i)}$,
who are the columns of $\bm{X}_t^f$. Blue dots (\mksrc): starting particles
$\bm{z}_{t,\tau=0}^{(i)}$ drawn from
the localized Gaussian
$q_{t,0}^{\mathrm{train}}(\cdot\mid\bm{X}_t^f)=\mathcal{N}(\bar{\bm{x}}_t^f,
\bm{C}_{\mathrm{GC}}(r_{\mathrm{loc}})\circ\bm{P}_t^f)$
(\cref{eq:locgauss-dist}), whose covariance contours are the shaded blue
ellipses. Orange dots (\mkend): endpoint particles
$\bm{z}_{t,\tau=1}^{(i)}$, forming the raw analysis ensemble.
Red star (\mktruth): ground truth state $\bm{x}_t^\star$.}
\label{fig:flowef-diagram}
\end{figure}

\subsection{Conditional flow matching}
\label{sec:cfm}

The velocity field is trained by conditional flow matching
\citep[CFM]{lipman2023flow}. Training data consist of simulation tuples
$\{\bm{X}_t^f,\bm{X}_t^a,\bm{y}_t,\bm{x}_t^\star\}$: the baseline filter $B$
provides the baseline forecast and analysis ensembles
$\{ \bm{X}_t^f, \bm{X}_t^a \}$, and the simulator provides
the ground truth state $\bm{x}_t^\star$.
The details are given in \Cref{app:data-splits}.

For a starting particle $\bm{z}_{t,0}\sim q_{t,0}(\cdot\mid\bm{X}_t^f)$ and
$\tau\sim\mathrm{Uniform}(0,1)$, we use the straight-line interpolant:
\citep{mccann1997convexity}:
\begin{align}
    \bm{z}_{t,\tau}^{\mathrm{train}}
    :=
    (1-\tau)\bm{z}_{t,0}
    +
    \tau\bm{x}_t^\star,
    \qquad
    \tau\in[0,1],
    \label{eq:flowda-interpolant}
\end{align}
whose target velocity is
$d\bm{z}_{t,\tau}^{\mathrm{train}}/d\tau
=\bm{x}_t^\star-\bm{z}_{t,0}$.
The CFM objective is
\begin{align}
    \mathcal{L}_{\mathrm{CFM}}(\theta)
    =
    \mathbb{E}
    \left[
    \left\|
    \bm{v}_\theta(\bm{z}_{t,\tau}^{\mathrm{train}},\tau;\bm{c}_{t,\tau})
    -
    (\bm{x}_t^\star-\bm{z}_{t,0})
    \right\|_2^2
    \right],
    \label{eq:cfm-loss}
\end{align}
where the expectation is over training cycles $t \in \mathcal{T}_{\mathrm{train}}$, source samples $\bm{z}_{t,0}\sim q_{t,0}(\cdot\mid\bm{X}_t^f)$ and interpolation
times $\tau \sim \operatorname{Uniform}(0,1)$. In this loss, the dynamic
conditioning $\bm{c}_{t,\tau}^{\mathrm{dyn}}$ is evaluated at
$\bm{z}_{t,\tau}^{\mathrm{train}}$.

At inference, an analysis ensemble is obtained by integrating
\cref{eq:flowda-analysis-ode} from $\tau=0$ to $\tau=1$. We approximate this
ODE with $K$ forward Euler steps,
\begin{align}
    \bm{z}_{t,(k+1)/K}
    =
    \bm{z}_{t,k/K}
    +
    \frac{1}{K}
    \bm{v}_\theta
    \left(
    \bm{z}_{t,k/K},\frac{k}{K};\bm{c}_{t,k/K}
    \right),
    \qquad k = 0, \ldots, K - 1.
    \label{eq:ode-euler}
\end{align}
Applying this transport to the $N$ starting particles $\bm{z}_{t,0}^{(i)} \sim q_{t,0}(\cdot \mid \bm{X}_t^f)$ gives the raw analysis
ensemble $[\bm{z}_{t,1}^{(1)},\ldots,\bm{z}_{t,1}^{(N)}]$. This raw ensemble can then be
calibrated with inflation \cref{eq:flowda-inflation} to produce the analysis
ensemble $\hat{\bm{X}}_t^a=[\hat{\bm{x}}_t^{a,(1)},\ldots,\hat{\bm{x}}_t^{a,(N)}]$.

\subsection{Forecast ensemble starting laws}
\label{sec:sources}

The starting law determines which information from the baseline forecast is
present before the learned transport begins.
Let $\bm{X}_t^f$, $\bar{\bm{x}}_t^f$, and $\bm{P}_t^f$ be the forecast ensemble, mean, and covariance from the baseline filter $B$.
During training, we use the localized Gaussian training source
\begin{align}
    q_{t,0}^{\mathrm{train}}(\cdot\mid\bm{X}_t^f)
    =
    \mathcal{N}\!\left(
        \bar{\bm{x}}_t^f,
        \bm{C}_{\mathrm{GC}}(r_{\mathrm{loc}})\circ\bm{P}_t^f
    \right),
    \label{eq:locgauss-dist}
\end{align}
where $\bm{C}_{\mathrm{GC}}(r_{\mathrm{loc}}) \in \mathbb{R}^{d_x \times d_x}$ is the compact-support
Gaspari--Cohn correlation matrix \citep{gaspari1999construction} with localization radius $r_{\mathrm{loc}}$, and $\circ$
denotes the Hadamard product.
The localization suppresses poorly estimated long-range correlations, as in
standard ensemble filtering \citep{houtekamer2001sequential}.
In our experiments, we reused the localization radius $r_{\mathrm{loc}}$ already tuned for the baseline filter $B$.

At inference, we initialize the $i$-th trajectory with the
corresponding forecast member,
\begin{align}
    \bm{z}_{t,0}^{(i)}=\bm{x}_t^{f,(i)},
    \qquad i=1,\ldots,N.
    \label{eq:ordered-deploy-source}
\end{align}
Thus every forecast member is transported exactly once without resampling.

\subsection{Conditioning variables}
\label{sec:cond}

\paragraph{Static conditioning}
The static conditioning variables are constructed from the baseline forecast
and analysis ensembles and the current observation:
\begin{align}
    \bm{c}_t^{\mathrm{static}}
    =C(\bm{X}_t^f,\bm{X}_t^a,\bm{y}_t)
    =[\bm{c}_t^{f,y},\bm{c}_t^{a,f}]
    \in\mathbb{R}^{d_x\times d_{\mathrm{static}}},
    \label{eq:condition-tensor}
\end{align}
where $\bm{c}_t^{f,y}$ contains forecast--observation features and
$\bm{c}_t^{a,f}$ contains analysis--forecast features.

Under the linear observation model of \cref{eq:linear-gaussian-analysis-model},
the forecast--observation features are constructed from the innovation field
$\tilde{\bm{d}}_t\in\mathbb{R}^{d_x}$
(\cref{eq:static-innovation}) that measures \emph{how much the observations
currently disagree with the forecast}, the
forecast mean $\bar{\bm{x}}_t^f$ and standard deviation $\bm{\sigma}_t^f=\sqrt{\operatorname{diag}(\bm{P}_t^f)}\in\mathbb{R}^{d_x}$
from a baseline filter $B$ (\cref{eq:forecast-moments}),
a normalized observation weight map
$\bm{m}_t\in[0,1]^{d_x}$ (\cref{eq:observation-weight-map}) that marks
\emph{which state coordinates are observed} and \emph{how strongly}.
The analysis--forecast features use the baseline analysis increment
$\bm{\Delta}_t=\bar{\bm{x}}_t^a-\bar{\bm{x}}_t^f$, the corresponding change in
ensemble spread, and the increment relative to the innovation.
Thus, $\bm{c}_t^{f,y}$ describes the observational discrepancy relative to the
forecast, while $\bm{c}_t^{a,f}$ describes the response of the baseline filter.
Conditioning on these channels, the model can learn an update relative to a known, competent update instead of rediscovering classical gain behavior from scratch. This facilitates training and data efficiency.

\paragraph{Dynamic observation conditioning}
\label{sec:obsres}
The static conditioning variables $\bm{c}_t^{\mathrm{static}}$
remain fixed throughout the
transport. To provide information about the current transported state, we
compute the back-projected, whitened observation residual
\begin{align}
    \bm{r}_{t,\tau}
    =
    \bm{H}_t^\top
    \bm{R}^{-1/2}
    \left(\bm{y}_t-\bm{H}_t\bm{z}_{t,\tau}\right)
    \in\mathbb{R}^{d_x}.
    \label{eq:obs-residual-field}
\end{align}
During training, $\bm{z}_{t,\tau}$ is the interpolant
$\bm{z}_{t,\tau}^{\mathrm{train}}$ of \cref{eq:flowda-interpolant}; at
inference, it is the current numerical ODE state $\bm{z}_{t,\tau_k}^{(i)}$
from (\cref{eq:flowda-analysis-ode,eq:ode-euler}).

We use the pseudo-time embedding
\begin{align}
    \bm{\phi}(\tau)
    = [\tau,\sin(2\pi\tau),\cos(2\pi\tau),\sin(4\pi\tau)]
    \in\mathbb{R}^{d_{\phi}},
    \label{eq:time-embedding}
\end{align}
and broadcast it over the state coordinates as
$\bm{\Phi}(\tau)=\bm{1}_{d_x}\bm{\phi}(\tau)^\top\in\mathbb{R}^{d_x\times d_{\phi}}$.
The input to the observation CNN is therefore
$[\bm{r}_{t,\tau},\bm{m}_t,\bm{\Phi}(\tau)]
\in\mathbb{R}^{d_x\times(d_{\phi}+2)}$.
A periodic convolutional neural network maps the residual, observation weights,
and pseudo-time embedding to the dynamic conditioning field:
\begin{align}
    \bm{c}_{t,\tau}^{\mathrm{dyn}}
    =\operatorname{CNN}^{\mathrm{period}}_{\theta_{\mathrm{obs}}}
    \big( [\bm{r}_{t,\tau},\bm{m}_t,\bm{\Phi}(\tau)] \big)
    \in\mathbb{R}^{d_x\times d_{\mathrm{dyn}}}.
    \label{eq:dynamic-conditioning}
\end{align}
Consequently, the conditioning variable $\bm{c}_{t,\tau}^{\mathrm{dyn}}$ encodes the current observation residual together with the stage of the transport. Architectural details are given in
\cref{eq:implementation-observation-input,eq:implementation-observation-cnn}.

\paragraph{Full conditioning}
The full conditioning field is
\begin{align}
    \bm{c}_{t,\tau}
    =\left[\bm{c}_t^{\mathrm{static}},\bm{c}_{t,\tau}^{\mathrm{dyn}}\right]
    \in\mathbb{R}^{d_x\times d_c},
    \qquad d_c=d_{\mathrm{static}}+d_{\mathrm{dyn}}.
    \label{eq:full-conditioning}
\end{align}
The transported particle, pseudo-time embedding, and conditioning field are
then concatenated channelwise and passed to the velocity network:
\begin{align}
    \bm{v}_\theta(\bm{z}_{t,\tau},\tau;\bm{c}_{t,\tau})
    =\operatorname{CNN}^{\mathrm{period}}_{\theta_{\mathrm{field}}}
    \big(
        [\bm{z}_{t,\tau},\bm{\Phi}(\tau),\bm{c}_{t,\tau}]
    \big).
    \label{eq:velocity-cnn}
\end{align}
The full parameter set is $\theta={\theta_{\mathrm{obs}},\theta_{\mathrm{field}}}$ and the network architecture details are given in \cref{eq:implementation-trunk-input,eq:implementation-velocity-cnn}.

\subsection{Training}
\label{sec:train}

Let $\mathcal{T}_{\mathrm{train}}$ denote the set of assimilation
times used for training. Given a classical ensemble filter baseline $B$, the training data are
\begin{align}
    \mathcal{D}_{\mathrm{train}}
    =\left\{
      \{ \bm{X}_t^f,\bm{X}_t^a,\bm{y}_t,\bm{x}_t^\star \}
      :t\in\mathcal{T}_{\mathrm{train}}
    \right\}.
    \label{eq:training-data}
\end{align}
These quantities are generated before training FlowEF. In particular, constructing
the training tuple at time $t$ does not require a FlowEF analysis at time
$t-1$.

For each $t\in\mathcal{T}_{\mathrm{train}}$, we first draw
$\bm{z}_{t,0}\sim q_{t,0}^{\mathrm{train}}(\cdot\mid\bm{X}_t^f)$ and
$\tau\sim\operatorname{Uniform}(0,1)$ , and construct the interpolant in \cref{eq:flowda-interpolant}, whose straight line target velocity is simply
$d\bm{z}_{t,\tau}^{\mathrm{train}}/d\tau
=\bm{x}_t^\star-\bm{z}_{t,0}$.
We train the velocity field by minimizing
\begin{align}
    \mathcal{L}_{\mathrm{train}}(\theta)
    =\frac{1}{|\mathcal{T}_{\mathrm{train}}|}
      \sum_{t\in\mathcal{T}_{\mathrm{train}}}
      \mathbb{E}_{\bm{z}_{t,0},\tau}
      \left[
        \left\lVert
            \bm{v}_\theta(
                \bm{z}_{t,\tau}^{\mathrm{train}},\tau;
                \bm{c}_{t,\tau})
            -(\bm{x}_{t}^\star-\bm{z}_{t,0})
        \right\rVert_2^2
      \right]
    \label{eq:finite-training-loss}
\end{align}
where the expectation is over the training source and interpolation time. We
denote the fitted parameters by $\hat{\theta}=\argmin_\theta \mathcal{L}_{\mathrm{train}}(\theta)$

\subsection{Inference and calibration}
\label{sec:infl}
Let $\mathcal{T}_{\mathrm{test}}$ denote the test assimilation times. Conditional
on the baseline forecast and analysis ensembles, FlowEF applies the analysis rule $\mathcal{A}(\cdot\mid\bm{X}_t^f,\bm{X}_t^a,\bm{y}_t)$ independently to every
$t\in\mathcal{T}_{\mathrm{test}}$, so these
transport calculations can be performed in parallel.

For each $t$, we construct the source ensemble via
$\bm{z}_{t,0}^{(i)}=\bm{x}_t^{f,(i)}$
(\cref{eq:ordered-deploy-source}) and compute the static conditioning tensor
$\bm{c}_t^{\mathrm{static}}$ from \cref{eq:condition-tensor}.
We perform $K$ Euler steps
with $\tau_k:=k/K$. At step $k$, for each $i=1,\ldots,N$,
the current transported particle determines the observation
residual and dynamic conditioning field,
\begin{align}
    \bm{r}_{t,\tau_k}^{(i)}
    =\bm{H}_t^\top\bm{R}^{-1/2}
    \big(\bm{y}_t-\bm{H}_t\bm{z}_{t,\tau_k}^{(i)}\big),
    \quad
    \bm{c}_{t,\tau_k}^{\mathrm{dyn},(i)}
    =\operatorname{CNN}^{\mathrm{period}}_{\hat{\theta}_{\mathrm{obs}}}
    \big([\bm{r}_{t,\tau_k}^{(i)},\bm{m}_t,\bm{\Phi}(\tau_k)]\big).
\end{align}
Combining the dynamic field with the static conditioning gives
$\bm{c}_{t,\tau_k}^{(i)}=[\bm{c}_t^{\mathrm{static}},
\bm{c}_{t,\tau_k}^{\mathrm{dyn},(i)}]$, and the velocity and Euler update in
\cref{eq:ode-euler} are then
\begin{align}
    \bm{v}_{t,\tau_k}^{(i)}
    := \bm{v}_{\hat{\theta}}
      \big(\bm{z}_{t,\tau_k}^{(i)},\tau_k;
      \bm{c}_{t,\tau_k}^{(i)}\big)
    =\operatorname{CNN}^{\mathrm{period}}_{\hat{\theta}_{\mathrm{field}}}
      \big([\bm{z}_{t,\tau_k}^{(i)},\bm{\Phi}(\tau_k),
      \bm{c}_{t,\tau_k}^{(i)}]\big),
    \quad
    \bm{z}_{t,\tau_{k+1}}^{(i)}
    =\bm{z}_{t,\tau_k}^{(i)}+\frac{1}{K}\bm{v}_{t,\tau_k}^{(i)}.
\end{align}
Thus, the dynamic conditioning is recomputed from the current transported
particle before each velocity evaluation.

The ODE transport endpoints $\bm{z}_{t,1}^{(i)}$ form the raw FlowEF analysis
ensemble.
To calibrate its spread, we apply multiplicative inflation about the ensemble
mean:
\begin{align}
    \hat{\bm{x}}_t^{a,(i)}
    =
    \hat{\bar{\bm{x}}}_t^a
    +
    \alpha_{\mathrm{flow}}
    \big(\bm{z}_{t,1}^{(i)}-\hat{\bar{\bm{x}}}_t^a\big),
    \qquad
    \hat{\bar{\bm{x}}}_t^a
    =
    \frac{1}{N}\sum_{i=1}^{N}\bm{z}_{t,1}^{(i)}.
    \label{eq:flowda-inflation}
\end{align}
The inflation factor $\alpha_{\mathrm{flow}}$ for FlowEF is selected on
held-out validation data by minimizing the marginal ensemble CRPS
\citep{gneiting2007strictly}. For the analysis ensemble
$\hat{\bm{X}}_t^{a} = [\hat{\bm{x}}_t^{a,(1)}, \ldots, \hat{\bm{x}}_t^{a,(N)}] \in\mathbb{R}^{d_x \times N}$ and target
$\bm{x}_t^\star\in\mathbb{R}^{d_x}$, we average scalar CRPS over
state coordinates first, then average this score over validation times. The
finite-ensemble CRPS at time $t$ is given by
\begin{align}
    \operatorname{CRPS}_t(\hat{\bm{X}}_t^a,\bm{x}_t^\star)
    =\frac{1}{d_x}\sum_{\ell=1}^{d_x}
    \left[
      \frac{1}{N}\sum_{i=1}^{N}
      \left|\hat{x}_{t,\ell}^{a,(i)}-x_{t,\ell}^\star\right|
      -\frac{1}{2N^2}\sum_{i=1}^{N}\sum_{j=1}^{N}
      \left|\hat{x}_{t,\ell}^{a,(i)}-\hat{x}_{t,\ell}^{a,(j)}\right|
    \right].
    \label{eq:ensemble-crps}
\end{align}
This is the marginal CRPS averaged over state coordinates, rather than a
multivariate energy score.

The baseline inflation factor $\alpha_B$ and localization radius
$r_{\mathrm{loc}}$ are jointly tuned on the validation data before FlowEF
training. During training, we reuse the localization radius $r_{\mathrm{loc}}$
from the baseline filter for FlowEF source distribution.
After training, $\alpha_{\mathrm{flow}}$ is tuned
separately on the validation data.

\subsection{Convergence analysis}
\label{sec:convergence}

The assumptions and proofs for the following results are deferred to \Cref{app:theory}.

\begin{theorem}[FlowEF error bound]
\label{thm:convergence-informal}
For any $t$, let $\mathcal{C}_t^N=[\bm c_t^{\mathrm{static}},\bm o_t]$ and
$\mathcal{C}_t^\infty=[\bm c_t^{\mathrm{static},\infty},\bm o_t]$ be
finite-ensemble and infinite-population conditioning variables, respectively.
Let $\pi_t^a(d\bm{x})=\mathbb{P}(\bm{x}_t^a\in d\bm{x}\mid\bm{y}_{1:t})$
be the exact filtering law and $\widehat{\pi}_t^a=N^{-1}\sum_{i=1}^N
\delta_{\hat{\bm{x}}_t^{a,(i)}}$ be the FlowEF empirical analysis
law. Here $N$ is the ensemble size, $K$ is the number of Euler
steps, and $W_2$ is the $2$-Wasserstein distance. Under the assumptions in
\Cref{app:theory-assumptions},
\begin{align}
    \mathbb{E}\, W_2\big(\widehat{\pi}_t^a,\pi_t^a\big)
    \le
    &\underbrace{B_{c,t}}_{\varepsilon_{\mathrm{bias}}}
    + \underbrace{L_{K,t}\big(\sigma_{c,t}\sqrt{D_{c,t}/N}+b_{c,N,t}\big)}_{\varepsilon_{\mathrm{feat}}}
    + \underbrace{C_{\mathrm{FM}}\sqrt{\varepsilon_{\mathrm{vel}}}}_{\varepsilon_{\mathrm{fm}}}
    + \underbrace{C_{\mathrm{Euler}}K^{-1}}_{\varepsilon_{\mathrm{ode}}} \notag\\
    &+ \underbrace{e^{\widetilde{L}_{v,z,t}}\big(O(r_{f,N})+\delta_{\mathrm{gauss},t}\big)}_{\varepsilon_{\mathrm{src}}}
    + \underbrace{|\alpha_{\mathrm{flow}}-1|\,\mathbb{E}S_{t,N}}_{\varepsilon_{\mathrm{cal},N}}.
    \label{eq:convergence-overview}
\end{align}
\end{theorem}
The first two terms quantify, respectively, information lost by the static conditioning
summary and its finite-ensemble estimation. The next two
capture velocity-field approximation and numerical integration error.
The source term accounts for the discrepancy between the forecast-member
initialization used at inference and the localized Gaussian source used for
training, while the final term quantifies the perturbation introduced by
calibration inflation. The quantities appearing in the bound are summarized
in
\Cref{tab:theory-dictionary}.

\section{Experiments}
\label{sec:experiments}
We evaluate FlowEF with four classical ensemble filters as baselines:
EnKF \citep{evensen2003ensemble}, ETKF \citep{bishop2001adaptive},
DEnKF \citep{sakov2008deterministic}, and LETKF
\citep{hunt2007efficient}. Experiments are conducted on three chaotic
dynamical systems: the Lorenz--96 ODE system, the one-dimensional
Kuramoto--Sivashinsky PDE, and two-dimensional Kolmogorov flow.
For each baseline filter, the inflation factor and localization radius
${\alpha_B,r_{\mathrm{loc}}}$ are jointly tuned on the validation set;
FlowEF is then trained and evaluated using that tuned baseline.

All results are averaged over 10 different seeds. Each seed determines
the train/validation/test split, the ground truth trajectory, the
observation noise, the baseline ensemble's initialization and
analysis draws, FlowEF initialization, minibatch order,
and stochastic inference sampling.
These quantities vary together across seeds.

\subsection{Experimental setup}
\label{sec:setup}
The three systems all use the linear observation model
$\bm{y}_t = \bm{H}_t\bm{x}_t + \bm{\epsilon}_t$,
where $\bm{H}_t \in \mathbb{R}^{d_y \times d_x}$
is a sparse and binary matrix
that selects an evenly spaced
subset of state coordinates $\bm{x}_t$,
which means we partially observe the system.
The state dimension $d_x$, observation dimension $d_y$ and other settings
for each system are given in \Cref{tab:setup}.
The three systems also share the same training size $n_{\mathrm{train}}=10{,}000$,
same test set size $n_{\mathrm{test}}=2{,}000$,
and identical FlowEF architecture, training, and inference hyperparameters,
see \Cref{tab:implementation-settings} for complete details.

For comparisons with learned DA methods, we select, for each system, the
classical baseline with the lowest validation CRPS among EnKF, ETKF, DEnKF,
and LETKF. This gives DEnKF for Lorenz--96 and Kuramoto--Sivashinsky and LETKF
for Kolmogorov flow. We report the results for all three systems below; the
full learned-method benchmark comparisons are collected in \Cref{app:benchmarks}.

\begin{table}[t]
\centering
\caption{Experimental setup. $d_x$ denotes the state dimension, $d_y$ the
observation dimension, $\sigma_{\mathrm{obs}}$ the observation noise standard deviation,
$\sigma_{\mathrm{proc}}$ the process noise standard deviation, $\Delta t$ the assimilation
interval, and $N$ the ensemble size.}
\label{tab:setup}
\small
\begin{tabular}{lcccccc}
\toprule
System & $d_x$ & $d_y$ & $\sigma_{\mathrm{obs}}$ & $\sigma_{\mathrm{proc}}$ & $\Delta t$ & $N$ \\
\midrule
Lorenz--96 (low-dim) & $40$       & $10$   & $0.1$ & $0.2$ & $0.05$ & $30$ \\
Lorenz--96 (high-dim) & $100$      & $20$   & $0.1$ & $0.2$ & $0.05$ & $30$ \\
Kuramoto--Sivashinsky & $128$      & $16$ & $0.7$ & $0.0$ & $1.0$  & $20$ \\
Kolmogorov flow& $64{\times}64$ & $8{\times}8$  & $0.1$ & $0.2$ & $0.2$  & $30$ \\
\bottomrule
\end{tabular}
\end{table}

\paragraph{Metrics}
We report root mean squared error (RMSE, \cref{eq:test-rmse})
of the ensemble mean,
the continuous ranked probability score (CRPS, \cref{eq:ensemble-crps}),
the spread-skill ratio (SSR, \cref{eq:test-ssr}),
and the empirical coverage probability
($\operatorname{Cov}_{95}$, \cref{eq:test-cov95}) of the $95\%$ ensemble interval.
We compute each metric per seed and aggregate it as mean
$\pm$ standard deviation over seeds.
RMSE measures point accuracy only, whereas
CRPS is a proper probabilistic score that measures both accuracy and how
tightly the predictive distribution concentrates around the truth.
SSR is the square root of the ratio between averaged ensemble variance and
averaged squared error, values below or above $1$ indicate under- or over-dispersion.
$\operatorname{Cov}_{95}$ checks each coordinate against its own marginal
interval and averages over assimilation times and coordinates,
values below or above $0.95$ indicate under- or over-coverage.

\subsection{Lorenz--96 model}
\label{sec:lorenz96}
The Lorenz--96 model \citep{lorenz1996predictability} is a
standard chaotic model of an atmospheric variable on a latitude circle, and is
widely used to evaluate data assimilation methods.
Its state evolves according to
\begin{align}
    \frac{\mathrm{d}x_i(t)}{\mathrm{d}t}
    &= \bigl(x_{i+1}(t)-x_{i-2}(t)\bigr)x_{i-1}(t)
       - x_i(t) + F,
    \qquad i=1,\ldots,d_x,
    \qquad F=8,
    \label{eq:lorenz96-dynamics}
\end{align}
where $x_i(t)$ is the $i$-th state coordinates
at time $t$ with periodic indexing $x_{i+d_x}(t)=x_i(t)$,
and $F$ is the constant forcing.
We integrate the system with
the fourth-order Runge--Kutta method at the
internal time step $\Delta t_{\mathrm{RK}}=\Delta t/4=0.0125$. We generate
the ground truth trajectory through four such substeps
over one assimilation interval, with $\Delta t=0.05$ as specified in \Cref{tab:setup}.
We evaluate two sizes of the Lorenz--96 system at the same noise levels
$\sigma_{\mathrm{obs}}=0.1$ and $\sigma_{\mathrm{proc}}=0.2$,
assimilation interval $\Delta t=0.05$, and ensemble size $N=30$,
as shown in \Cref{tab:setup}.

\paragraph[]{Low-dimensional system}
For the low-dimensional Lorenz--96 system, we set $d_x=40$, $d_y=10$,
and choose $N=30$, which observes every fourth state coordinate.
\Cref{tab:headline-lorenz96} reports the results for
the proposed method FlowEF on top of each of four baselines
averaged over 10 seeds.
FlowEF improves both the deterministic and probabilistic
metrics for every baseline: the improvement is modest for the already
strong DEnKF, but substantially larger for EnKF, ETKF, and LETKF. FlowEF
also brings SSR and Cov$_{95}$ closer to their reference values for all
four baselines.

\begin{table}[t]
\centering
\caption{Comparison of FlowEF against each baseline on test set
for low-dimensional Lorenz--96 system ($d_x{=}40$, $d_y{=}10$, $N{=}30$).
Bold marks the best value per metric.}
\label{tab:headline-lorenz96}
\small
\setlength{\tabcolsep}{3pt}
\resizebox{\linewidth}{!}{%
\begin{tabular}{llcccc}
\toprule
Baseline & Method & RMSE $\downarrow$ & CRPS $\downarrow$ & SSR ($\to1$) & Cov$_{95}$ ($\to0.95$) \\
\midrule
\multirow{2}{*}{EnKF}
& Baseline           & $1.42 \pm 0.208$ & $0.701 \pm 0.102$ & $0.365 \pm 0.0926$ & $0.741 \pm 0.0527$ \\
& \cellcolor{gray!15}FlowEF (ours)      & \cellcolor{gray!15}$\mathbf{1.25 \pm 0.130}$ & \cellcolor{gray!15}$\mathbf{0.581 \pm 0.0566}$ & \cellcolor{gray!15}$\mathbf{0.861 \pm 0.0747}$ & \cellcolor{gray!15}$\mathbf{0.901 \pm 0.0216}$ \\
\midrule
\multirow{2}{*}{ETKF}
& Baseline           & $1.41 \pm 0.211$ & $0.692 \pm 0.106$ & $0.457 \pm 0.137$ & $0.813 \pm 0.0479$ \\
& \cellcolor{gray!15}FlowEF (ours)      & \cellcolor{gray!15}$\mathbf{1.24 \pm 0.157}$ & \cellcolor{gray!15}$\mathbf{0.577 \pm 0.0726}$ & \cellcolor{gray!15}$\mathbf{0.910 \pm 0.0920}$ & \cellcolor{gray!15}$\mathbf{0.905 \pm 0.0209}$ \\
\midrule
\multirow{2}{*}{DEnKF}
& Baseline           & $0.894 \pm 0.154$ & $0.438 \pm 0.0693$ & $0.753 \pm 0.166$ & $0.893 \pm 0.0262$ \\
& \cellcolor{gray!15}FlowEF (ours)      & \cellcolor{gray!15}$\mathbf{0.873 \pm 0.127}$ & \cellcolor{gray!15}$\mathbf{0.420 \pm 0.0534}$ & \cellcolor{gray!15}$\mathbf{0.948 \pm 0.145}$ & \cellcolor{gray!15}$\mathbf{0.905 \pm 0.0202}$ \\
\midrule
\multirow{2}{*}{LETKF}
& Baseline           & $1.32 \pm 0.112$ & $0.661 \pm 0.0581$ & $0.315 \pm 0.0545$ & $0.722 \pm 0.0429$ \\
& \cellcolor{gray!15}FlowEF (ours)      & \cellcolor{gray!15}$\mathbf{1.18 \pm 0.0680}$ & \cellcolor{gray!15}$\mathbf{0.551 \pm 0.0299}$ & \cellcolor{gray!15}$\mathbf{0.835 \pm 0.0928}$ & \cellcolor{gray!15}$\mathbf{0.900 \pm 0.0219}$ \\
\bottomrule
\end{tabular}}
\end{table}

\Cref{tab:bench-lorenz96} in \Cref{app:benchmarks} compares FlowEF
against representative ML-based DA benchmarks
on top of the DEnKF baseline.
FlowEF achieves the best RMSE, CRPS, SSR, and Cov$_{95}$
among all benchmarks, which demonstrates its
outperformance in both predictive accuracy
and ensemble calibration.

\paragraph[]{High-dimensional system}
For the high-dimensional Lorenz--96 system, we increase the state dimension to
$d_x=100$ while keeping $N=30$, and observe $d_y=20$ coordinates. The sample
covariance therefore remains of rank at most $N-1$ while the state dimension
increases substantially, making finite-ensemble covariance estimation more
challenging.
As shown in \Cref{tab:headline-lorenz96-large}, FlowEF again improves all four
reported metrics relative to each baseline. The reductions in RMSE and CRPS
are also larger than in the $d_x=40$ experiment for each of the four
baselines.

\begin{table}[t]
\centering
\caption{Comparison of FlowEF against each baseline on test set
for high-dimensional Lorenz--96 system ($d_x{=}100$, $d_y{=}20$, $N{=}30$).
Bold marks the best value per metric.}
\label{tab:headline-lorenz96-large}
\small
\setlength{\tabcolsep}{3pt}
\resizebox{\linewidth}{!}{%
\begin{tabular}{llcccc}
\toprule
Baseline & Method & RMSE $\downarrow$ & CRPS $\downarrow$ & SSR ($\to1$) & Cov$_{95}$ ($\to0.95$) \\
\midrule
\multirow{2}{*}{EnKF}
& Baseline           & $2.94 \pm 0.160$ & $1.51 \pm 0.121$ & $0.370 \pm 0.0745$ & $0.631 \pm 0.0672$ \\
& \cellcolor{gray!15}FlowEF (ours)      & \cellcolor{gray!15}$\mathbf{2.23 \pm 0.0880}$ & \cellcolor{gray!15}$\mathbf{1.06 \pm 0.0519}$ & \cellcolor{gray!15}$\mathbf{0.889 \pm 0.0218}$ & \cellcolor{gray!15}$\mathbf{0.883 \pm 0.00890}$ \\
\midrule
\multirow{2}{*}{ETKF}
& Baseline           & $2.86 \pm 0.102$ & $1.44 \pm 0.0575$ & $0.584 \pm 0.0955$ & $0.763 \pm 0.0272$ \\
& \cellcolor{gray!15}FlowEF (ours)      & \cellcolor{gray!15}$\mathbf{2.26 \pm 0.0726}$ & \cellcolor{gray!15}$\mathbf{1.08 \pm 0.0422}$ & \cellcolor{gray!15}$\mathbf{0.933 \pm 0.0434}$ & \cellcolor{gray!15}$\mathbf{0.877 \pm 0.0126}$ \\
\midrule
\multirow{2}{*}{DEnKF}
& Baseline           & $2.15 \pm 0.102$ & $1.03 \pm 0.0564$ & $0.758 \pm 0.234$ & $0.848 \pm 0.0590$ \\
& \cellcolor{gray!15}FlowEF (ours)      & \cellcolor{gray!15}$\mathbf{1.89 \pm 0.0809}$ & \cellcolor{gray!15}$\mathbf{0.874 \pm 0.0439}$ & \cellcolor{gray!15}$\mathbf{0.969 \pm 0.0830}$ & \cellcolor{gray!15}$\mathbf{0.886 \pm 0.0197}$ \\
\midrule
\multirow{2}{*}{LETKF}
& Baseline           & $2.84 \pm 0.0945$ & $1.45 \pm 0.0646$ & $0.279 \pm 0.0927$ & $0.608 \pm 0.0613$ \\
& \cellcolor{gray!15}FlowEF (ours)      & \cellcolor{gray!15}$\mathbf{2.13 \pm 0.0695}$ & \cellcolor{gray!15}$\mathbf{0.998 \pm 0.0420}$ & \cellcolor{gray!15}$\mathbf{0.895 \pm 0.0244}$ & \cellcolor{gray!15}$\mathbf{0.893 \pm 0.0124}$ \\
\bottomrule
\end{tabular}}
\end{table}

The learned-method comparison for the high-dimensional system is reported in
\Cref{tab:bench-lorenz96-large} of \Cref{app:benchmarks}, again using DEnKF as
the common baseline. FlowEF has the lowest RMSE and CRPS and the closest SSR
and Cov$_{95}$ to their reference values among the methods reported there.
The limited improvements in  the low-dimensional case does not apply:
increasing $d_x$ from $40$ to $100$ at fixed $N=30$
inflates the rank deficiency $d_x/N$, which is exactly
the regime where the classical filter is most limited by sampling error in
the forecast covariance, so there is far more space for a learned
map to exploit.

\begin{figure}[t]
\centering
\begin{minipage}[t]{0.49\linewidth}\centering
\includegraphics[width=\linewidth]{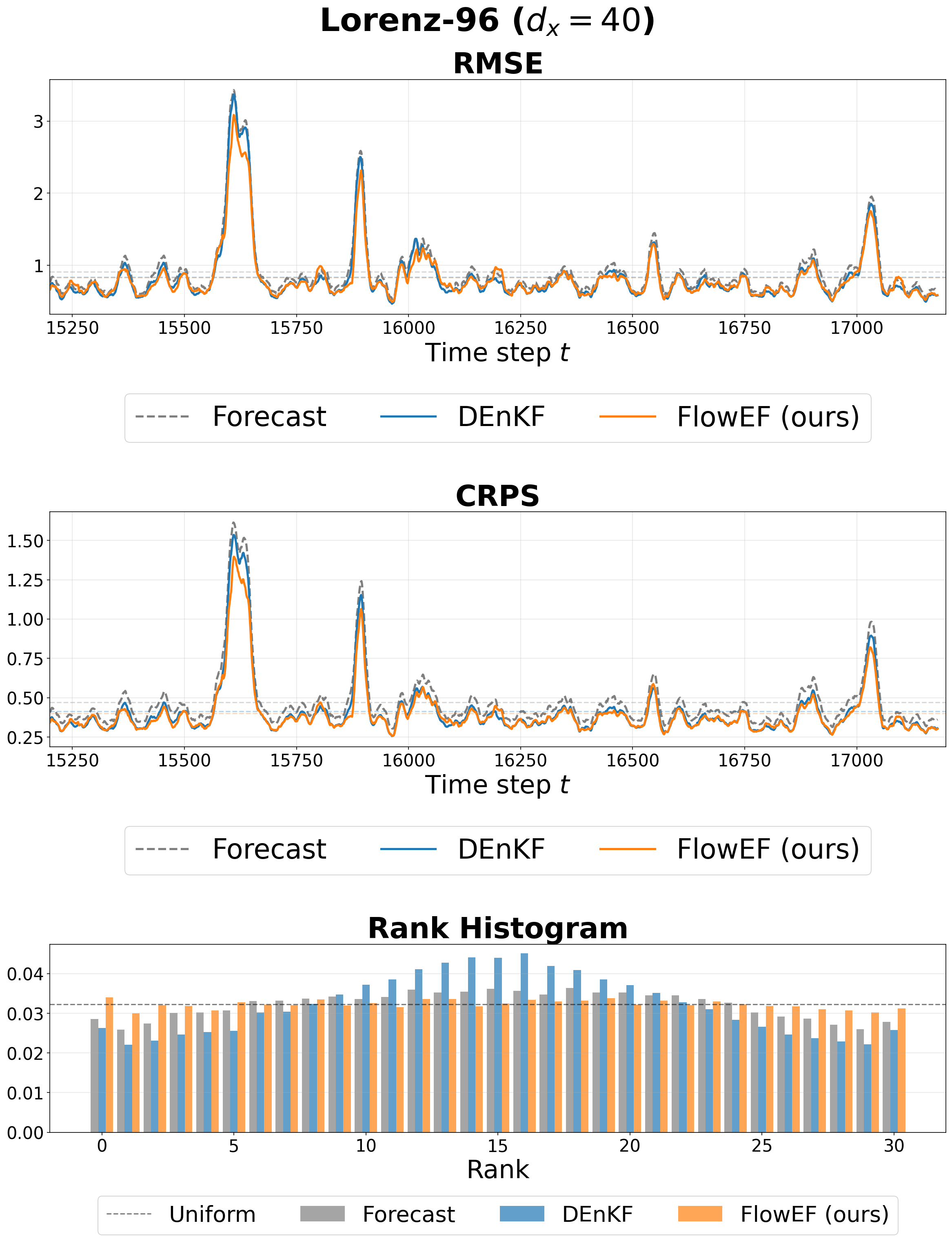}\end{minipage}\hfill
\begin{minipage}[t]{0.49\linewidth}\centering
\includegraphics[width=\linewidth]{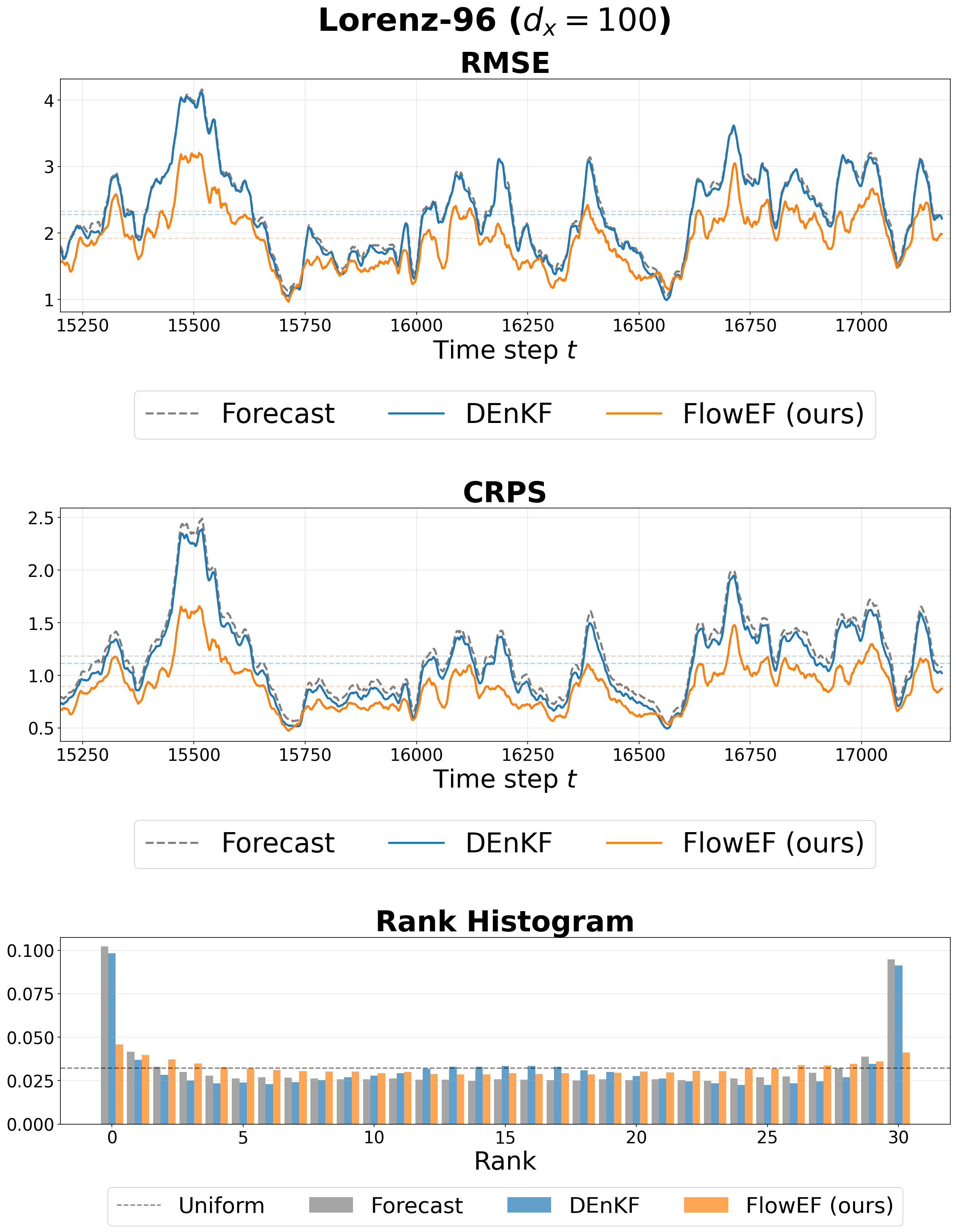}\end{minipage}
\caption{Metrics comparison for Lorenz--96 using DEnKF as the baseline filter.
\textit{Left}: low-dimensional Lorenz--96 ($d_x{=}40$).
\textit{Right}: high-dimensional Lorenz--96 ($d_x{=}100$).
Each subfigure shows RMSE, CRPS, and rank histogram over test set.
Gray dashed, blue, and orange curves denote DEnKF forecast ensemble,
DEnKF analysis ensemble, and FlowEF analysis ensemble, respectively.
The horizontal dashed line in rank histogram denotes the uniform rank reference $1/(N{+}1)$.}
\label{fig:best-seed-metrics-lorenz96}
\end{figure}

\Cref{fig:best-seed-metrics-lorenz96}
shows the same qualitative pattern
for both low-dimensional and high-dimensional Lorenz--96 systems.
FlowEF produces lower RMSE and CRPS than DEnKF across
the entire trajectory and a flatter rank histogram.
The figure suggests that the advantage of FlowEF is stronger
when the forecast ensemble is more rank-deficient,
which is consistent with the results in
\Cref{tab:headline-lorenz96,tab:headline-lorenz96-large}.
\subsection{Kuramoto--Sivashinsky equation}
\label{sec:ks}
The Kuramoto--Sivashinsky (KS) equation
\citep{kuramoto1978diffusion, sivashinsky1977nonlinear} describes phase
turbulence in reaction--diffusion systems, 
and is a classical model of spatiotemporal chaos.
The KS field $u(x,t)$ on $x \in [0,L]$ with
$L=32\pi$ and periodic boundary conditions $u(x+L,t)=u(x,t)$ satisfies
\begin{align}
    \partial_t u(x,t) + u(x,t)\,\partial_x u(x,t) + \partial_{xx}u(x,t) + \partial_{xxxx}u(x,t)
    &= 0,
    \label{eq:ks-dynamics}
\end{align}
where $u(x,t)\in\mathbb{R}$
is the field at position $x\in[0,L]$ and time $t$,
discretized on $d_x = 128$ equispaced grid points.
The high-order derivative terms create rapidly changing 
fine-scale features, so simple time-stepping methods require
very small stable steps. 
We therefore integrate the ground truth
trajectory with the standard ETDRK4 method \citep{kassam2005fourth} 
at an internal time step of $0.25$. 
Each assimilation interval consists of four such steps, with
$\Delta t=1.0$ as specified in \Cref{tab:setup}. 
We select observations at every eighth point ($d_y = 16$) 
with $\sigma_{\mathrm{obs}} = 0.7$ and no process noise $\sigma_{\mathrm{proc}} = 0$.
We follow the KS equation
setup of \citet{bach2025learning} using the same grid resolution, domain
extent, observation dimension and noise, and ensemble size $N=20$.

\begin{table}[t]
\centering
\caption{Comparison of FlowEF against each baseline on the test set
for Kuramoto--Sivashinsky ($d_x{=}128$, $d_y{=}16$, $N{=}20$). 
Bold marks the best value per metric.}
\label{tab:headline-ks}
\small
\setlength{\tabcolsep}{3pt}
\resizebox{\linewidth}{!}{%
\begin{tabular}{llcccc}
\toprule
Baseline & Method & RMSE $\downarrow$ & CRPS $\downarrow$ & SSR ($\to1$) & Cov$_{95}$ ($\to0.95$) \\
\midrule
\multirow{2}{*}{EnKF}
& Baseline           & $0.607 \pm 0.0331$ & $0.309 \pm 0.0187$ & $0.777 \pm 0.0610$ & $0.835 \pm 0.0196$ \\
& \cellcolor{gray!15}FlowEF (ours)      & \cellcolor{gray!15}$\mathbf{0.538 \pm 0.0193}$ & \cellcolor{gray!15}$\mathbf{0.273 \pm 0.0111}$ & \cellcolor{gray!15}$\mathbf{0.915 \pm 0.0416}$ & \cellcolor{gray!15}$\mathbf{0.856 \pm 0.0132}$ \\
\midrule
\multirow{2}{*}{ETKF}
& Baseline           & $0.608 \pm 0.114$ & $0.314 \pm 0.0704$ & $0.853 \pm 0.152$ & $0.853 \pm 0.0680$ \\
& \cellcolor{gray!15}FlowEF (ours)      & \cellcolor{gray!15}$\mathbf{0.538 \pm 0.0572}$ & \cellcolor{gray!15}$\mathbf{0.274 \pm 0.0336}$ & \cellcolor{gray!15}$\mathbf{0.938 \pm 0.0520}$ & \cellcolor{gray!15}$\mathbf{0.863 \pm 0.0184}$ \\
\midrule
\multirow{2}{*}{DEnKF}
& Baseline           & $0.502 \pm 0.0322$ & $0.257 \pm 0.0183$ & $0.826 \pm 0.0931$ & $0.840 \pm 0.0233$ \\
& \cellcolor{gray!15}FlowEF (ours)      & \cellcolor{gray!15}$\mathbf{0.483 \pm 0.0235}$ & \cellcolor{gray!15}$\mathbf{0.248 \pm 0.0132}$ & \cellcolor{gray!15}$\mathbf{0.939 \pm 0.0731}$ & \cellcolor{gray!15}$\mathbf{0.854 \pm 0.0214}$ \\
\midrule
\multirow{2}{*}{LETKF}
& Baseline           & $0.589 \pm 0.0403$ & $0.297 \pm 0.0213$ & $0.804 \pm 0.0877$ & $0.854 \pm 0.0232$ \\
& \cellcolor{gray!15}FlowEF (ours)      & \cellcolor{gray!15}$\mathbf{0.526 \pm 0.0321}$ & \cellcolor{gray!15}$\mathbf{0.269 \pm 0.0237}$ & \cellcolor{gray!15}$\mathbf{0.968 \pm 0.120}$ & \cellcolor{gray!15}$\mathbf{0.872 \pm 0.0226}$ \\
\bottomrule
\end{tabular}}
\end{table}

\Cref{tab:headline-ks} shows the performance of FlowEF against each baseline
on the test set for the Kuramoto--Sivashinsky equation. 
As on Lorenz--96, DEnKF is the strongest baseline, and
the margin between FlowEF and the DEnKF baseline is narrow,
since the DEnKF baseline is already close to the linear--Gaussian
optimum under the KS equation setting,
which leaves little room for any correction on top of it.
Accordingly, FlowEF improves least over DEnKF
and more over EnKF, ETKF, and LETKF.
Across all four baselines, FlowEF improves RMSE, CRPS, SSR, and
Cov$_{95}$, and reduces their variability across runs.
The left subfigure in \Cref{fig:best-seed-metrics-ks-kol}
visually demonstrates this on a representative seed.
RMSE and CRPS of FlowEF drop below the DEnKF baseline
across the entire trajectory,
and the rank histogram of FlowEF is flatter than the baseline,
which is consistent with the SSR improvement in \Cref{tab:headline-ks}.

\Cref{tab:bench-ks} in \Cref{app:benchmarks} compares FlowEF against
representative ML-based DA benchmarks on top of the DEnKF baseline.
FlowEF achieves the best RMSE, CRPS, and SSR among all benchmark methods,
while only slightly improving Cov$_{95}$ over the DEnKF baseline.
ADEnKF \citep{chen2022autodifferentiable}
instead attains the best Cov$_{95}$,
but yields higher RMSE and CRPS and worse SSR,
which indicates that its coverage gain
comes at the cost of worse predictive accuracy
and poorer spread--skill consistency.
Overall, FlowEF provides the best
predictive performance and ensemble uncertainty.

\begin{figure}[t]
\centering
\begin{minipage}[t]{0.49\linewidth}\centering
\includegraphics[width=\linewidth]{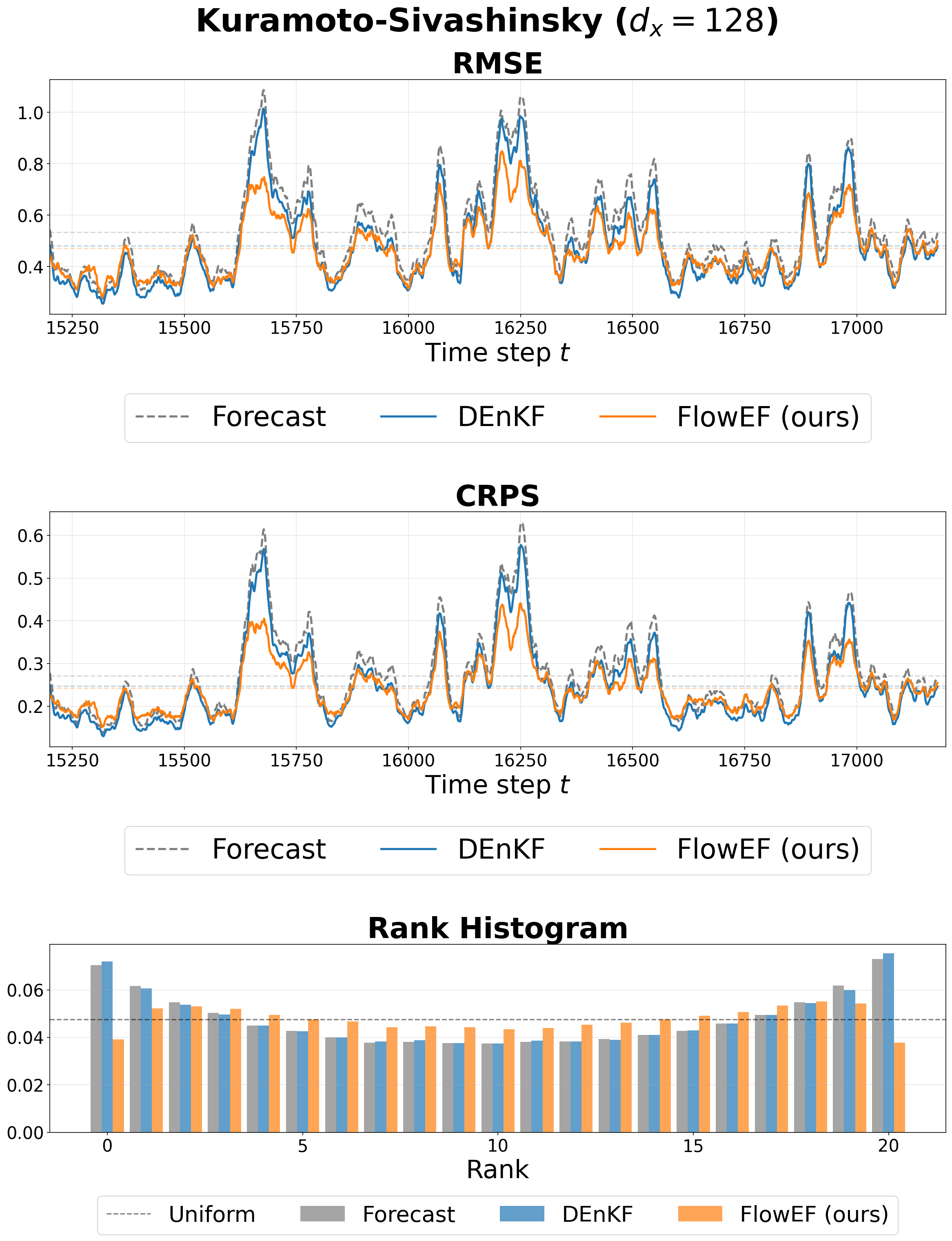}\end{minipage}\hfill
\begin{minipage}[t]{0.49\linewidth}\centering
\includegraphics[width=\linewidth]{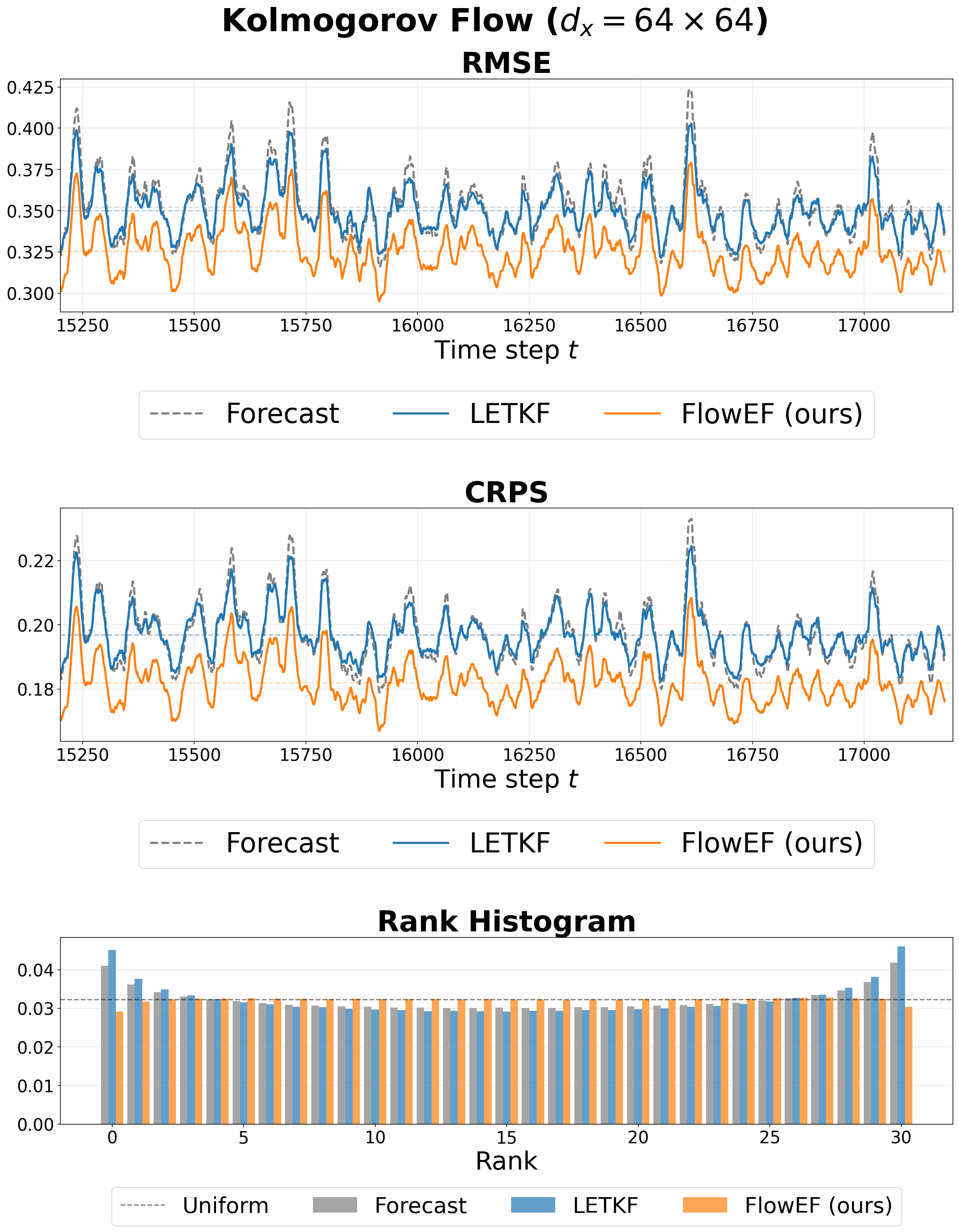}\end{minipage}
\caption{Metrics comparison for Kuramoto--Sivashinsky and Kolmogorov flow
on a representative seed. \emph{Left}: Kuramoto--Sivashinsky using
DEnKF as the baseline filter. \emph{Right}: Kolmogorov flow using LETKF as
the baseline filter. Each subfigure shows RMSE, CRPS, and the rank
histogram over test set. 
}
\label{fig:best-seed-metrics-ks-kol}
\end{figure}

\subsection{Kolmogorov flow}
\label{sec:kol}
Kolmogorov flow \citep{obukhov1983kolmogorov} 
is two-dimensional incompressible Navier--Stokes flow on a periodic domain
driven by a steady sinusoidal forcing.
It is a classical model of two-dimensional turbulence
\citep{boffetta2012two, chandler2013invariant} widely used in  
fluid dynamics and DA. In vorticity form, it is governed by
\begin{align}
    \partial_t \omega(x,y,t) + (\bm{u}(x,y,t) \cdot \nabla)\omega(x,y,t) &= \nu \Delta \omega(x,y,t) - \mu\omega(x,y,t) + g(x,y),\\
    \bm{u}(x,y,t) &= \nabla^\perp \Delta^{-1}\omega(x,y,t),
    \label{eq:kolmogorov-dynamics}
\end{align}
on the domain $(x,y)\in[0, 2\pi]^2$ with periodic boundary conditions
in both directions, where $\omega(x,y,t)\in\mathbb{R}$ is the 
vorticity field and $\bm{u}(x,y,t)\in\mathbb{R}^2$ is the divergence-free
velocity field induced by $\omega$. The stream function $\psi$ solves the
Poisson equation $\Delta\psi=\omega$, written in $\psi=\Delta^{-1}\omega$, and
$\bm{u}=\nabla^\perp\psi=(-\partial_y\psi,\partial_x\psi)$ is its
skew gradient, which is divergence-free by construction.
The system has viscosity $\nu = 0.01$, 
linear drag $\mu = 0.1$, and Kolmogorov forcing 
$g(x,y) = -k_f\cos(k_f y)$ at $k_f = 4$. 
The state is the vorticity discretized on a $64\times 64$ grid, 
so $\omega(t)\in\mathbb{R}^{64\times 64}$ is 
flattened into a state vector of dimension $d_x = 64^2 = 4096$. 
We integrate the ground truth trajectory with a
pseudo-spectral ETDRK4 scheme \citep{kassam2005fourth} 
at an internal time step of $0.02$. 
Each assimilation interval consists of ten such steps, with 
$\Delta t=0.2$ as specified in \Cref{tab:setup}. 
We take observations at every eighth grid point along each axis, 
a regular $8\times 8$ subsampling of the $64\times 64$ grid 
that gives $d_y = 64$ observations covering only $1.6\%$ of the state. 
With $N = 30$, the sample forecast covariance has rank at most $N-1 = 29$ 
while $d_x = 4096$. With $d_x/N \approx 137$, the forecast covariance is
severely rank-deficient; therefore, sampling error dominates its estimation.

\Cref{tab:headline-kol} reports results averaged over 10 seeds for FlowEF
on top of each of four baselines. FlowEF improves both RMSE and CRPS for
all baselines, with the largest gains for ETKF and the smallest for EnKF.
Although SSR and Cov$_{95}$ do not improve uniformly across all baselines,
FlowEF improves SSR for ETKF and LETKF and Cov$_{95}$ for DEnKF and
LETKF, while delivering consistently better overall predictive accuracy. This pattern is
consistent with FlowEF method that learns the analysis update conditioned on 
the ensembles from the baseline filter. 
The baselines with larger sampling error and inflation bias 
leave more room for improvement. 

\Cref{fig:kol-reconstruction} 
illustrates the state reconstruction over time for Kolmogorov flow using LETKF as the baseline filter 
against the ground truth. 
It's observed that FlowEF tracks the large-scale vortex structure 
of the ground truth throughout the trajectory and stays 
close to the truth even at later assimilation times. 
The right subfigure in \Cref{fig:best-seed-metrics-ks-kol} 
demonstrates the metrics comparison for Kolmogorov flow 
on a representative seed. 
FlowEF consistently outperforms the LETKF baseline filter 
in RMSE and CRPS across the entire trajectory,
and the rank histogram of FlowEF is flatter than the LETKF baseline, 
which is consistent with the SSR improvement in \Cref{tab:headline-kol}. 

\begin{table}[t]
\centering
\caption{Comparison of FlowEF against each baseline on the test set
for Kolmogorov flow ($d_x{=}64{\times}64{=}4096$, $d_y{=}8{\times}8{=}64$, $N{=}30$).
Bold marks the best value per metric.}
\label{tab:headline-kol}
\small
\setlength{\tabcolsep}{3pt}
\resizebox{\linewidth}{!}{%
\begin{tabular}{llcccc}
\toprule
Baseline & Method & RMSE $\downarrow$ & CRPS $\downarrow$ & SSR ($\to1$) & Cov$_{95}$ ($\to0.95$) \\
\midrule
\multirow{2}{*}{EnKF}
& Baseline           & $0.841 \pm 1.14$ & $0.565 \pm 0.855$ & $\mathbf{0.591 \pm 0.399}$ & $\mathbf{0.819 \pm 0.252}$ \\
& \cellcolor{gray!15}FlowEF (ours)      & \cellcolor{gray!15}$\mathbf{0.771 \pm 1.03}$ & \cellcolor{gray!15}$\mathbf{0.493 \pm 0.737}$ & \cellcolor{gray!15}$0.576 \pm 0.302$ & \cellcolor{gray!15}$0.803 \pm 0.211$ \\
\midrule
\multirow{2}{*}{ETKF}
& Baseline           & $1.48 \pm 0.0765$ & $0.784 \pm 0.0439$ & $1.28 \pm 0.0247$ & $\mathbf{0.965 \pm 0.00215}$ \\
& \cellcolor{gray!15}FlowEF (ours)      & \cellcolor{gray!15}$\mathbf{1.08 \pm 0.0565}$ & \cellcolor{gray!15}$\mathbf{0.557 \pm 0.0297}$ & \cellcolor{gray!15}$\mathbf{1.03 \pm 0.0244}$ & \cellcolor{gray!15}$0.885 \pm 0.00966$ \\
\midrule
\multirow{2}{*}{DEnKF}
& Baseline           & $0.773 \pm 1.26$ & $0.522 \pm 0.945$ & $\mathbf{0.868 \pm 0.467}$ & $0.844 \pm 0.261$ \\
& \cellcolor{gray!15}FlowEF (ours)      & \cellcolor{gray!15}$\mathbf{0.584 \pm 0.718}$ & \cellcolor{gray!15}$\mathbf{0.339 \pm 0.422}$ & \cellcolor{gray!15}$0.713 \pm 0.260$ & \cellcolor{gray!15}$\mathbf{0.864 \pm 0.0769}$ \\
\midrule
\multirow{2}{*}{LETKF}
& Baseline           & $0.547 \pm 0.209$ & $0.331 \pm 0.145$ & $0.548 \pm 0.331$ & $0.835 \pm 0.0539$ \\
& \cellcolor{gray!15}FlowEF (ours)      & \cellcolor{gray!15}$\mathbf{0.476 \pm 0.165}$ & \cellcolor{gray!15}$\mathbf{0.273 \pm 0.104}$ & \cellcolor{gray!15}$\mathbf{0.699 \pm 0.259}$ & \cellcolor{gray!15}$\mathbf{0.877 \pm 0.0288}$ \\
\bottomrule
\end{tabular}}
\end{table}

\begin{figure}[t]
\centering
\includegraphics[width=\linewidth]{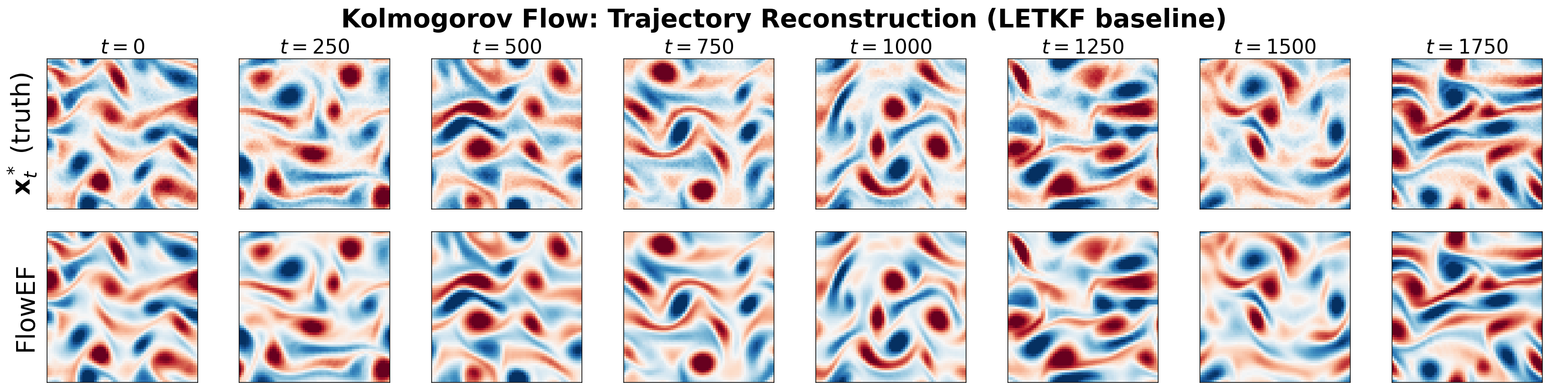}
\caption{Kolmogorov flow state reconstruction over time using LETKF as the
baseline filter on a representative seed. 
\emph{Top}: ground truth $\bm{x}_t^\star$. 
\emph{Bottom}: analysis ensembles $\hat{\bm{X}}_t^{a}$ of FlowEF. 
Columns show assimilation times $t=0,250,\ldots,1750$.}
\label{fig:kol-reconstruction}
\end{figure}

\Cref{tab:bench-kol} in \Cref{app:benchmarks} compares FlowEF against representative ML-based DA
benchmarks on top of the LETKF baseline. FlowEF achieves the best performance
across all four metrics among all benchmark methods.
FlowEF uses a periodic CNN velocity field directly on the vorticity field,
thereby encoding local spatial interactions and consistent behavior under
spatial shifts.
Overall, among the evaluated benchmark methods, FlowEF provides the most
accurate state estimates and the most reliable uncertainty quantification for
Kolmogorov flow.

\section{Conclusion}
\label{sec:conclusion}

We introduced the Flow Ensemble Filter (FlowEF), a learned nonlinear analysis update that augments a classical ensemble filter through conditional flow matching. The learned velocity field is conditioned on features constructed from the baseline forecast and analysis ensembles and the current observation. During training, FlowEF uses a localized Gaussian source derived from the forecast ensemble; at inference, it transports the forecast members themselves to form the analysis ensemble. This construction anchors the learned transport to the current forecast while allowing a nonlinear analysis update beyond the affine structure of classical ensemble filters.

We evaluated FlowEF with EnKF, ETKF, DEnKF, and LETKF baselines on Lorenz--96 with state dimensions $40$ and $100$, Kuramoto--Sivashinsky with dimension $128$, and $64\times64$ Kolmogorov flow. Across the reported settings, FlowEF reduced both RMSE and marginal CRPS relative to the corresponding tuned classical baselines. Where direct comparisons with learned DA methods were available, FlowEF also achieved lower RMSE and marginal CRPS than the evaluated alternatives. The improvements were smaller for already strong baseline analyses and larger in several of the more challenging finite-ensemble settings. FlowEF also generally improved the spread--skill ratio and marginal coverage, although empirical $95\%$ coverage remained below its nominal level in several settings. Overall, these results show that a forecast-informed learned transport can improve both point accuracy and probabilistic performance while retaining information supplied by a classical ensemble filter.

\paragraph{Limitations}
FlowEF is currently trained separately for each baseline filter and dynamical system, which limits transfer across settings without retraining. Training also relies on simulated trajectories for which the latent state is known, so performance depends on how well the training distribution represents the regime encountered at inference. In addition, numerical integration of the learned flow introduces computational cost beyond the underlying classical analysis, and the present conditioning construction is designed primarily for the observation models considered here.

\paragraph{Future work}
Future work could develop FlowEF models that transfer across baseline filters, dynamical regimes, and observation configurations without separate retraining. Other directions include extending the conditioning construction to nonlinear, non-Gaussian, or misspecified observation models; learning localization and calibration jointly with the transport; and reducing inference cost through more efficient flow parameterizations or numerical integration schemes. These extensions would broaden the applicability and robustness of the proposed framework.

\acks{The authors acknowledge support from the Singapore Ministry of Education
under grants MOE-000618-00 and MOE-000537-00.}

\vskip 0.2in
\bibliography{refs}
\newpage
\appendix
\crefname{appendix}{Appendix}{Appendices}
\Crefname{appendix}{Appendix}{Appendices}
\crefalias{section}{appendix}
\crefalias{subsection}{appendix}
\section{Implementation details}
\label{sec:flowda-implementation-details}
In this section, we present the implementation details of the proposed method FlowEF,
which includes the data generation and splits in \Cref{app:data-splits},
the localized Gaussian sampling in \Cref{app:localized-gaussian-sampling},
the network architecture in \Cref{app:architecture},
the training procedure in \Cref{app:training},
the inference and calibration in \Cref{app:inference},
and the evaluation metrics in \Cref{app:metrics}.
The complete algorithm is summarized in \Cref{alg:flowda},
and the implementation and hyperparameter settings
for the experiments are collected in \Cref{tab:implementation-settings}.

\subsection{Data generation and splits}
\label{app:data-splits}
For each system, we simulate one trajectory of $n_{\mathrm{total}}=25{,}000$
time steps. Following standard ensemble filtering practice,
we discard an initial spin-up segment of length $n_{\mathrm{spinup}}$
before collecting the data used for training, validation, and testing.
This spin-up serves two purposes:
(i) it lets the ground truth state trajectory move away from its
arbitrary starting point and settle into the system's typical chaotic
dynamics, and
(ii) it gives the baseline filter
time to converge from its arbitrary initialization, so that neither the
ground truth nor the filter's ensemble
we used for experiments is still unstable.
From the resulting trajectory we take three disjoint
windows of assimilation times:
a training set of size $n_{\mathrm{train}}=10{,}000$,
a validation set of size $n_{\mathrm{val}}=2{,}000$,
and a test set of size $n_{\mathrm{test}}=2{,}000$.
We use validation set $\mathcal{T}_{\mathrm{val}}$
to tune the baseline filter's inflation factor and localization radius
before training FlowEF,
and tune the FlowEF's inflation factor $\alpha_{\mathrm{flow}}$ (\Cref{app:inference})
after training FlowEF.
This pipeline is repeated over $10$ data seeds, each
of which regenerates the trajectory, the split, and every other source of
randomness listed in \Cref{tab:implementation-settings}.

\subsection{Localized Gaussian sampling}
\label{app:localized-gaussian-sampling}
The taper $\bm{C}_{\mathrm{GC}}(r_{\mathrm{loc}}) \in \mathbb{R}^{d_x\times d_x}$ of \cref{eq:locgauss-dist}
is a function of pairwise grid distance. For the one-dimensional systems
(Lorenz--96, KS) with periodic index set $\{0,\ldots,d_x-1\}$, the distance
between grid points $\ell,k$ is the ring distance
$
  \rho(\ell,k)=\min\big(|\ell-k|,\,d_x-|\ell-k|\big).
$
For Kolmogorov flow, discretized on an $n_x\times n_x$ periodic grid
($d_x=n_x^2$), points are indexed by $(\ell_1,\ell_2),(k_1,k_2)\in\{0,\ldots,n_x-1\}^2$
and the distance is the torus distance
$
  \rho\big((\ell_1,\ell_2),(k_1,k_2)\big)
  =\sqrt{\rho_1(\ell_1,k_1)^2+\rho_1(\ell_2,k_2)^2},
$
with $\rho_1(a,b)=\min(|a-b|,n_x-|a-b|)$ the coordinatewise ring distance.
The entries of the taper are the Gaspari--Cohn fifth-order piecewise
polynomial \citep{gaspari1999construction} evaluated at the grid distance,
\begin{align}
    C_{\mathrm{GC},\ell k}(r_{\mathrm{loc}})
    &=G\big(\rho(\ell,k)/r_{\mathrm{loc}}\big),
    \notag
    \\
    G(z)
    &=\begin{cases}
    1-\dfrac{5}{3}z^2+\dfrac{5}{8}z^3+\dfrac{1}{2}z^4-\dfrac{1}{4}z^5,
    & 0\leq z\leq 1,\\[4pt]
    4-5z+\dfrac{5}{3}z^2+\dfrac{5}{8}z^3-\dfrac{1}{2}z^4+\dfrac{1}{12}z^5
    -\dfrac{2}{3z},
    & 1<z<2,\\[4pt]
    0, & z\geq 2,
    \end{cases}
    \label{eq:gaspari-cohn}
\end{align}
so that $\bm{C}_{\mathrm{GC}}(r_{\mathrm{loc}})$ has compact support
$\rho(\ell,k)\leq 2r_{\mathrm{loc}}$.
We sample the localized Gaussian in \cref{eq:locgauss-dist}
using the factorization below. 
Let
\begin{align}
    \bm{C}_{\mathrm{GC}}
    &=\bm{U}\bm{\Lambda}\bm{U}^\top,
    &
    \bm{\Lambda}
    &=\operatorname{diag}(\lambda_1,\ldots,\lambda_{d_x}),
    \quad \lambda_1\geq\cdots\geq\lambda_{d_x}\geq 0,
    \label{eq:taper-eigendecomposition}
\end{align}
be the eigendecomposition of the taper, with eigenvalues clipped at zero to
remove negative numerical roundoff. Define the truncation rank $r$ and 
low-rank factor $\bm{L}$ as
\begin{align}
    r
    &=\min\Big\{j\in\{1,\ldots,d_x\}:
      \textstyle\sum_{i=1}^{j}\lambda_i
      \geq 0.999\sum_{i=1}^{d_x}\lambda_i\Big\},
    &
    \bm{L}
    &=\bm{U}_r\bm{\Lambda}_r^{1/2}\in\mathbb{R}^{d_x\times r},
    \label{eq:taper-rank-and-factor}
\end{align}
where $\bm{U}_r\in\mathbb{R}^{d_x\times r}$ collects the eigenvectors
associated with $\lambda_1,\ldots,\lambda_r$ and
$\bm{\Lambda}_r=\operatorname{diag}(\lambda_1,\ldots,\lambda_r)$, so that
$\bm{L}\bm{L}^\top=\bm{U}_r\bm{\Lambda}_r\bm{U}_r^\top\approx\bm{C}_{\mathrm{GC}}$.
We get the localized Gaussian sample 
$\bm{z}_{t,0} \sim \mathcal{N}(\bar{\bm{x}}_t^f,\bm{C}_{\mathrm{GC}}\circ\bm{P}_t^f)$ 
in \cref{eq:locgauss-dist} as follows:
\begin{align}
    \bm{W}
    &=[W_{ij}]\in\mathbb{R}^{N\times r},
    \quad
    W_{ij}
    \overset{\mathrm{iid}}{\sim}\mathcal{N}(0,1),
    \quad i=1,\ldots,N,\ j=1,\ldots,r,
    \notag
    \\
    \bm{V}
    &=\frac{\bm{L}\bm{W}^\top}{\sqrt{N-1}}
      \in\mathbb{R}^{d_x\times N},
    \notag
    \\
    \bm{A}_t^f
    &
      =\bm{X}_t^f-\bar{\bm{x}}_t^f\bm{1}_N^\top
      \in\mathbb{R}^{d_x\times N},
    \notag
    \\
    \bm{z}_{t,0}
    &=\bar{\bm{x}}_t^f+\big(\bm{A}_t^f\circ\bm{V}\big)\bm{1}_N,
    \label{eq:localized-source-implementation}
\end{align}
where $\bm{1}_N\in\mathbb{R}^N$ is the all-ones vector and $\circ$ denotes
the Hadamard (elementwise) product. 
The covariance of $\bm{z}_{t,0}$ conditional on the forecast ensemble
$\bm{X}_t^f$ is
\begin{align}
    \operatorname{Cov}(\bm{z}_{t,0}\mid\bm{X}_t^f)
    &=(\bm{L}\bm{L}^\top)\circ\bm{P}_t^f
    \approx\bm{C}_{\mathrm{GC}}\circ\bm{P}_t^f,
    \label{eq:localized-source-covariance}
\end{align}
where $\bm{P}_t^f=\frac{1}{N-1}\bm{A}_t^f(\bm{A}_t^f)^\top$ defined in 
\cref{eq:forecast-moments}.
Setting $r_{\mathrm{loc}}=0$ gives the identity taper
$\bm{C}_{\mathrm{GC}}=\bm{I}_{d_x}$ and yields the diagonal Gaussian source
\begin{align}
    q_{t,0}^{\mathrm{diag}}(\cdot\mid\bm{X}_t^f)
    =\mathcal{N}\!\left(
      \bar{\bm{x}}_t^f,
      \mathrm{diag}\big({(\bm{\sigma}_t^f)}^2\big)
    \right),
    \qquad
    \bm{\sigma}_t^f=\sqrt{\mathrm{diag}(\bm{P}_t^f)}.
    \label{eq:flow-gauss}
\end{align}
This source preserves marginal forecast spread but discards cross-coordinate
correlations. 
The Gaspari--Cohn factor and its eigendecomposition
\eqref{eq:taper-eigendecomposition} are computed once for each localization radius. 
In practice, we use a dense eigendecomposition 
when $d_x$ is small and a randomized low-rank eigensolver when $d_x$ is large
and the estimated rank $r$ of \cref{eq:taper-rank-and-factor} is small.

\subsection{Architecture}
\label{app:architecture}

\paragraph{Conditioning channels}
The innovation field $\tilde{\bm{d}}_t \in \mathbb{R}^{d_x}$ and 
observation weight map $\bm{m}_{t}\in[0,1]^{d_x}$
in \Cref{sec:cond} are defined as follows:
\begin{align}
    \bm{d}_t
    &=\bm{y}_t-\bm{H}_t\bar{\bm{x}}_t^f,
    \quad
    e_t
    =\frac{1}{d_y}\bm{d}_t^\top\bm{R}^{-1}\bm{d}_t,
    \quad
    \tilde{\bm{d}}_t
    =\bm{H}_t^\top\bm{d}_t,
    \label{eq:static-innovation}
    \\
    \bm{m}_{t}[\ell]
    &=\frac{\lVert\bm{H}_{t}[:,\ell]\rVert_2}
      {\max_{1\leq k\leq d_x}\lVert\bm{H}_{t}[:,k]\rVert_2},
    \quad
    \ell\in\{1,\ldots,d_x\},
    \label{eq:observation-weight-map}
\end{align}
with the linear observation operator $\bm{H}_t \in \mathbb{R}^{d_y \times d_x}$ 
and noise covariance $\bm{R}\in\mathbb{R}^{d_y\times d_y}$ defined in
\cref{eq:linear-gaussian-analysis-model}. 
The static conditioning variables $\bm{c}_t^{\mathrm{static}}$ include the
forecast--observation channels $\bm{c}_t^{f,y}$ and analysis--forecast
channels $\bm{c}_t^{a,f}$:
\begin{align}
    \bm{c}_t^{\mathrm{static}}
    =[\bm{c}_t^{f,y},\bm{c}_t^{a,f}].
    \label{eq:static-condition}
\end{align}
The forecast--observation channels $\bm{c}_t^{f,y}$ are given by
\begin{align}
    \bm{c}_t^{f,y}
    =\left[
      \frac{\bar{\bm{x}}_t^f}{\operatorname{std}(\bar{\bm{x}}_t^f)},
      \frac{\bm{\sigma}_t^f}{\operatorname{mean}(\bm{\sigma}_t^f)},
      \frac{\tilde{\bm{d}}_t}{\operatorname{std}(\bm{d}_t)},
      \bm{m}_t,
      \log(1+e_t)\bm{1}_{d_x}
    \right].
    \label{eq:forecast-observation-condition}
\end{align}
With $\bar{\bm{x}}_t^a$, $\bm{\sigma}_t^a$ the baseline analysis mean and
spread, $\bm{\Delta}_t=\bar{\bm{x}}_t^a-\bar{\bm{x}}_t^f$ the mean
increment, and $\sigma_{\mathrm{obs}}$ the observation noise scale, the
analysis--forecast channels $\bm{c}_t^{a,f}$ are given by 
\begin{align}
    \bm{c}_t^{a,f}
    =\left[
      \frac{\bm{\Delta}_t}{\operatorname{std}(\bm{\Delta}_t)},
      \frac{\bm{\Delta}_t}{\bm{\sigma}_t^f},
      \frac{\bm{\sigma}_t^a}{\operatorname{mean}(\bm{\sigma}_t^f)},
      \log\frac{\bm{\sigma}_t^a}{\bm{\sigma}_t^f},
      1-{\left(\frac{\bm{\sigma}_t^a}{\bm{\sigma}_t^f}\right)}^{2},
      \frac{\tilde{\bm{d}}_t}{\sigma_{\mathrm{obs}}},
      \frac{|\bm{\Delta}_t|}{|\tilde{\bm{d}}_t|}
    \right].
    \label{eq:baseline-action-condition}
\end{align}

The above setting forms $d_{\mathrm{static}}=12$ channels
for static conditioning variables in \Cref{sec:cond} in total. 
Forecast and analysis spreads in these channels use the
population standard deviation, whereas the localized source uses the unbiased
$1/(N-1)$ ensemble covariance. For binary selection observation matrix $\bm{H}_t$, 
the observation weight map $\bm{m}_t$ is the binary mask of observed coordinates.
For dense observation matrix $\bm{H}_t$, 
the observation weight map $\bm{m}_t$ is computed 
from \cref{eq:observation-weight-map}. 

Because $\bm{c}_t^{\mathrm{static}}$ depends only on the assimilation time
$t$ and not on the pseudo-time $\tau$ or the ensemble member $i$, it is
precomputed once per $t$ and cached as an array of shape
$(T,d_x,d_{\mathrm{static}})$, which avoids recomputation at each of the $K$
pseudo-time steps. Since assimilation times are mutually independent and
only these $K$ steps of \cref{eq:ode-euler} are sequential, 
we vectorize over $t$: over $t\in\mathcal{T}_{\mathrm{batch}}$
during training, and over both $t\in\mathcal{T}_{\mathrm{test}}$ and
ensemble member $i$ during inference.

\paragraph{Network Architecture}
We set 
$d_\phi=4$, $d_{\mathrm{static}}=12$,
$d_{\mathrm{dyn}}=16$, 
and total number of conditioning channels is then $d_c=d_{\mathrm{static}} + d_{\mathrm{dyn}}=28$. 
The pseudo-time embedding $\bm{\phi}(\tau_t) \in \mathbb{R}^{d_\phi}$ of \cref{eq:time-embedding} is
broadcast over the $d_x$ state grids to form  
$\bm{\Phi}(\tau_t)=\bm{1}_{d_x}\bm{\phi}(\tau_t)^\top\in\mathbb{R}^{d_x\times d_{\phi}}$.
For an input $\bm{h}\in\mathbb{R}^{d_x\times c}$ with $c$ channels at each
grid point, kernel width $k$ ($k$ is odd), $w$ output channels, weights
$\bm{\Theta}\in\mathbb{R}^{k\times c\times w}$, bias $\bm{b}\in\mathbb{R}^w$,
and half-width $p_{\max}=(k-1)/2$, we write $\operatorname{Conv}^{\mathrm{period}}_{k,w}$
for the periodic (circular) convolution
\begin{align}
    \big(\operatorname{Conv}^{\mathrm{period}}_{k,w}(\bm{h})\big)[\ell,:]
    &=\bm{b}+\sum_{p=-p_{\max}}^{p_{\max}}
    \bm{h}\big[(\ell+p)\bmod d_x,:\big] \bm{\Theta}[p+p_{\max}+1,:,:],
    \\
    \ell & =0,\ldots,d_x-1,
    \label{eq:implementation-periodic-conv}
\end{align}
which wraps around the ring index set of \Cref{app:localized-gaussian-sampling}.
For Kolmogorov flow, the same rule is applied along each axis of the
$n_x\times n_x$ grid with $n_x$-periodic wraparound.
The overall velocity $\bm{v}_\theta(\bm{z}_{t,\tau},\tau;\bm{c}_{t,\tau})$ in
\cref{eq:velocity-cnn} is a composition of two periodic CNNs: an observation CNN that builds the dynamic conditioning features, and a velocity field CNN that
combines them together with the state and static channels:
\begin{align}
    \bm{c}_{t,\tau}^{\mathrm{dyn}}
    &=\operatorname{CNN}^{\mathrm{period}}_{\theta_{\mathrm{obs}}}
    \big([\bm{r}_{t,\tau},\bm{m}_t,\bm{\Phi}(\tau)]\big),
    \\
    \bm{v}_\theta(\bm{z}_{t,\tau},\tau;\bm{c}_{t,\tau})
    &=\operatorname{CNN}^{\mathrm{period}}_{\theta_{\mathrm{field}}}
    \big([\bm{z}_{t,\tau},\bm{\Phi}(\tau),\bm{c}_{t,\tau}]\big),
    \quad
    \bm{c}_{t,\tau}=[\bm{c}_t^{\mathrm{static}},\bm{c}_{t,\tau}^{\mathrm{dyn}}].
\end{align}

\begin{figure}[!htb]
\centering
\includegraphics[width=0.92\linewidth]{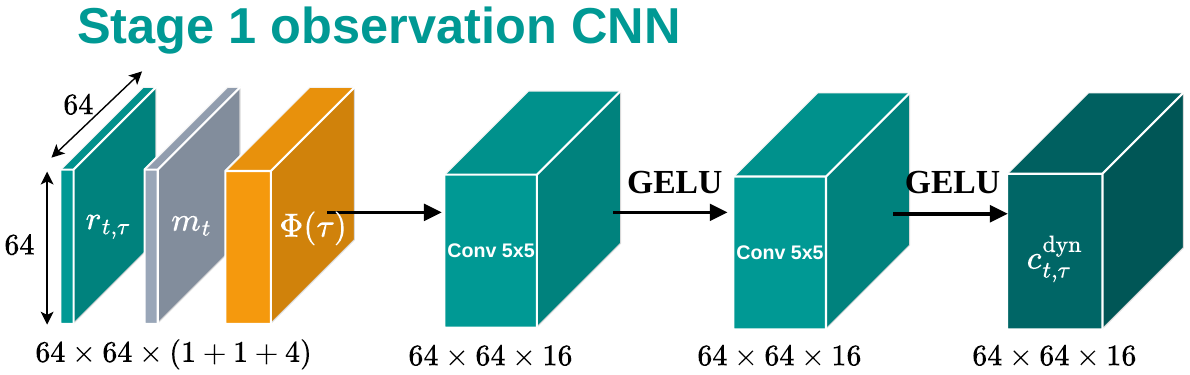}
\caption{Stage 1 observation CNN for Kolmogorov flow with $d_x=64\times64$.}
\label{fig:implementation-observation-cnn}
\end{figure}

\paragraph{Stage 1: observation CNN}
$\operatorname{CNN}^{\mathrm{period}}_{\theta_{\mathrm{obs}}}$ is a small
network that extracts dynamic features from the observation residuals $\bm{r}_{t,\tau}
=
\bm{H}_t^\top
\bm{R}^{-1/2}
\left(\bm{y}_t-\bm{H}_t\bm{z}_{t,\tau}\right)
\in\mathbb{R}^{d_x}$ of \cref{eq:obs-residual-field} and the
pseudo-time embedding $\bm{\phi}(\tau)
= [\tau,\sin(2\pi\tau),\cos(2\pi\tau),\sin(4\pi\tau)]
\in\mathbb{R}^{d_{\phi}}$ of \cref{eq:time-embedding}. Its layers alternate a periodic convolution with the
Gaussian error linear unit \citep[GELU]{hendrycks2016gaussian},
\begin{align}
    \operatorname{GELU}(x)=\frac{x}{2}\Big(1+\operatorname{erf}(x/\sqrt{2})\Big),
    \label{eq:implementation-gelu}
\end{align}
applied elementwise, where
$\operatorname{erf}(u)=\frac{2}{\sqrt{\pi}}\int_0^u e^{-s^2}\,\mathrm{d}s$
is the Gauss error function. Since the standard normal cumulative distribution
function satisfies
$\Phi_{\mathcal{N}}(x)=\frac{1}{2}\big(1+\operatorname{erf}(x/\sqrt{2})\big)$,
\cref{eq:implementation-gelu} is equivalently
$\operatorname{GELU}(x)=x\,\Phi_{\mathcal{N}}(x)=x\,\mathbb{P}(\xi\le x)$ with
$\xi\sim\mathcal{N}(0,1)$. Unlike the hard gate of ReLU, this
map is infinitely differentiable and has nonzero gradient everywhere, which
matters here because the velocity field must be differentiated through the
ODE solver during training. 
The architecture is given below and 
its implementation for Kolmogorov flow is illustrated in
\Cref{fig:implementation-observation-cnn}:
\begin{align}
    \bm{h}_{t,\tau}^{\mathrm{obs}}
    &= [\bm{r}_{t,\tau},\bm{m}_t,\bm{\Phi}(\tau)]
      \in\mathbb{R}^{d_x\times(2+d_\phi)},
    \label{eq:implementation-observation-input}\\
    \bm{u}^{(0)}_{t,\tau}
    &= \bm{h}_{t,\tau}^{\mathrm{obs}}, \\
    \bm{u}^{(j)}_{t,\tau}
    &= \operatorname{GELU}\!\left(
      \operatorname{Conv}^{\mathrm{period}}_{5,d_{\mathrm{dyn}},j}
      \big(\bm{u}^{(j-1)}_{t,\tau}\big)\right),
      \qquad j=1,2, \\
    \bm{c}_{t,\tau}^{\mathrm{dyn}}
    &= \bm{u}^{(2)}_{t,\tau}
      = \operatorname{CNN}^{\mathrm{period}}_{\theta_{\mathrm{obs}}}
      \big(\bm{h}_{t,\tau}^{\mathrm{obs}}\big)
      \in\mathbb{R}^{d_x\times d_{\mathrm{dyn}}}.
    \label{eq:implementation-observation-cnn}
\end{align}

\paragraph{Stage 2: velocity field CNN}
$\operatorname{CNN}^{\mathrm{period}}_{\theta_{\mathrm{field}}}$ 
is a larger network that computes the velocity field 
$\bm{v}_\theta(\bm{z}_{t,\tau},\tau;\bm{c}_{t,\tau})$ of \cref{eq:velocity-cnn}. 
It concatenates the transported particle $\bm{z}_{t,\tau}\in\mathbb{R}^{d_x}$, the pseudo-time embedding
$\bm{\Phi}(\tau)\in\mathbb{R}^{d_x\times d_\phi}$, the static channels $\bm{c}_t^{\mathrm{static}}\in\mathbb{R}^{d_x\times d_{\mathrm{static}}}$, and the dynamic observation features
$\bm{c}_{t,\tau}^{\mathrm{dyn}}\in\mathbb{R}^{d_x\times d_{\mathrm{dyn}}}$ channelwise into an input tensor
$\bm{h}_{t,\tau}^{\mathrm{in}}\in\mathbb{R}^{d_x\times(1+d_\phi+d_{\mathrm{static}}+d_{\mathrm{dyn}})}$, 
then applies a convolution to transform the input channels to a hidden width $w$, followed by $L=3$ residual blocks of width $w$ with layer normalization
\citep{ba2016layer}. For a hidden tensor $\bm{u}\in\mathbb{R}^{d_x\times w}$,
whose row $\bm{u}[\ell,:]\in\mathbb{R}^w$ is the channel vector at grid point
$\ell\in\{0,\ldots,d_x-1\}$, layer normalization (LN) rescales that grid point's
channels independently of every other grid point,
\begin{align}
    &\operatorname{LN}(\bm{u})[\ell,:]
    =\bm{\gamma}\circ\frac{\bm{u}[\ell,:]-\mu_\ell}{\sqrt{\sigma_\ell^2+\epsilon}}
      +\bm{\beta},
    \\
    &\mu_\ell=\frac{1}{w}\sum_{c=1}^{w}u[\ell,c],
    \quad
    \sigma_\ell^2=\frac{1}{w}\sum_{c=1}^{w}{\big(u[\ell,c]-\mu_\ell\big)}^2,
    \qquad \ell=0,\ldots,d_x-1,
    \label{eq:implementation-layernorm}
\end{align}
with learned per-channel scale and shift 
$\bm{\gamma},\bm{\beta}\in\mathbb{R}^w$, 
and a small constant $\epsilon>0$.
The pseudo-time embedding $\bm{\phi}(\tau)\in\mathbb{R}^{d_\phi}$ enters the
velocity network twice: once as the broadcast input channels 
$\bm{\Phi}(\tau)=\bm{1}_{d_x}\bm{\phi}(\tau)^\top\in\mathbb{R}^{d_x\times d_\phi}$ 
concatenated into $\bm{h}_{t,\tau}^{\mathrm{in}}$ 
(see \cref{eq:implementation-trunk-input}), 
and again inside every residual block $m=1,\ldots,L$ 
as a feature-wise linear modulation (FiLM) \citep{perez2018film}. 
After $L=3$ blocks of LayerNorm, GELU, and periodic convolutions, 
the information of $\tau$ carried by the input channels 
$\bm{h}_{t,\tau}^{\mathrm{in}}$ 
is repeatedly renormalized and transformed, 
so the deeper blocks may no longer 
see an accurate and useful signal for $\tau$. 
Since $\bm{v}_\theta(\bm{z}_{t,\tau},\tau;\bm{c}_{t,\tau})$
should vary smoothly with $\tau$ at every depth
to trace the ODE in \cref{eq:flowda-analysis-ode},
we use FiLM to re-inject $\bm{\phi}(\tau)$ directly
into each block's hidden features,
which gives every block a direct dependence on $\tau$
rather than one that is only inherited through depth. 
Concretely, FiLM rescales the hidden features
$\bm{a}^{(m)}_{t,\tau}$ inside block $m$
(see \cref{eq:implementation-velocity-cnn} below) by a scale
$\bm{\alpha}^{(m)}(\tau)\in\mathbb{R}^w$ and shifts them by
$\bm{\delta}^{(m)}(\tau)\in\mathbb{R}^w$, 
which are computed from $\bm{\phi}(\tau)$ by a
block specific affine layer,
\begin{align}
    \operatorname{Dense}_m(\bm{\phi})
    &= \bm{W}^{(m)}_{\phi}\bm{\phi}+\bm{b}^{(m)}_{\phi},
    &
    \bm{W}^{(m)}_{\phi}&\in\mathbb{R}^{2w\times d_\phi},\;\;
    \bm{b}^{(m)}_{\phi}\in\mathbb{R}^{2w},
    \label{eq:implementation-film-dense}
    \\
    [\bm{\alpha}^{(m)}(\tau),\bm{\delta}^{(m)}(\tau)]
    &= \operatorname{Dense}_m\big(\bm{\phi}(\tau)\big),
    &
    m&=1,\ldots,L,
    \label{eq:implementation-film-split}
\end{align}
with weights $\bm{W}^{(m)}_{\phi}$ and bias $\bm{b}^{(m)}_{\phi}$ learned
separately for each block $m$.
The architecture of the velocity CNN is given below 
and its implementation for Kolmogorov flow is shown 
in \Cref{fig:implementation-velocity-cnn}:
\begin{align}
    \bm{h}_{t,\tau}^{\mathrm{in}}
    &=\left[\bm{z}_{t,\tau},\bm{\Phi}(\tau),
      \bm{c}_t^{\mathrm{static}},\bm{c}_{t,\tau}^{\mathrm{dyn}}\right]
    \in\mathbb{R}^{d_x\times(1+d_\phi+d_{\mathrm{static}}+d_{\mathrm{dyn}})},
    \label{eq:implementation-trunk-input}
    \\
    \bm{s}^{(0)}_{t,\tau}
    &= \operatorname{Conv}^{1\times1}_{w,\mathrm{in}}
       \big(\bm{h}_{t,\tau}^{\mathrm{in}}\big),
    \notag
    \\
    \bm{a}^{(m)}_{t,\tau}
    &= \operatorname{Conv}^{\mathrm{period}}_{5,w,m,1}
       \!\left(\operatorname{GELU}\!\left(
       \operatorname{LN}\big(\bm{s}^{(m-1)}_{t,\tau}\big)
       \right)\right),
    \\
    \bm{f}^{(m)}_{t,\tau}
    &= \big(\bm{1}+\bm{\alpha}^{(m)}(\tau)\big)
       \circ\bm{a}^{(m)}_{t,\tau}+\bm{\delta}^{(m)}(\tau),
    \\
    \bm{s}^{(m)}_{t,\tau}
    &= \bm{s}^{(m-1)}_{t,\tau}
       +\operatorname{Conv}^{\mathrm{period}}_{5,w,m,2}
       \!\left(\operatorname{GELU}\!\left(
       \operatorname{LN}\big(\bm{f}^{(m)}_{t,\tau}\big)
       \right)\right),
    \notag
    \\
    \bm{v}_\theta(\bm{z}_{t,\tau},\tau;\bm{c}_{t,\tau})
    &= \operatorname{Conv}^{1\times1}_{1,\mathrm{out}}
       \big(\bm{s}^{(L)}_{t,\tau}\big),
       \qquad m=1,\ldots,L.
    \label{eq:implementation-velocity-cnn}
\end{align}
The FiLM scale and shift are broadcast across grid points, whereas the
periodic convolutions and layer normalization are applied coordinatewise. 
The LayerNorm parameters $\bm{\gamma},\bm{\beta}$
and the FiLM weights $\bm{W}^{(m)}_{\phi},\bm{b}^{(m)}_{\phi}$ are part of
$\theta_{\mathrm{field}}$, so they are learned jointly with the convolution
weights by minimizing the flow matching loss.

\begin{figure}[!htb]
\centering
\includegraphics[width=\linewidth]{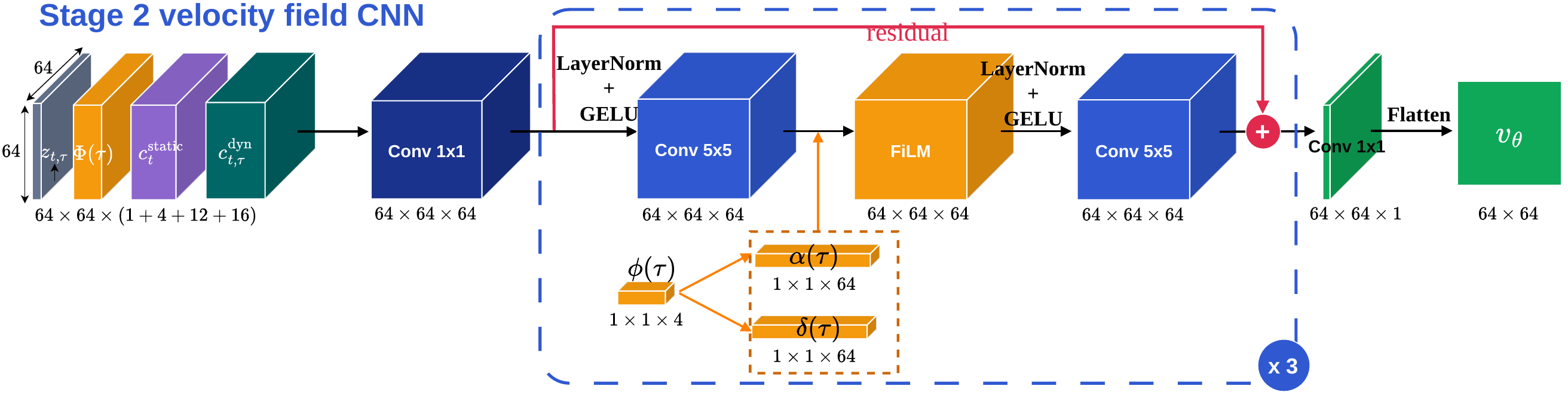}
\caption{Stage 2 velocity field CNN for Kolmogorov flow with $d_x=64 \times 64$.}
\label{fig:implementation-velocity-cnn}
\end{figure}

For Lorenz--96 and Kuramoto--Sivashinsky, 
all the inputs and hidden features are 1D tensors of length $d_x$, 
therefore we use circular padding in the 1D periodic convolutions 
with hidden width $w=32$ and kernel width $k=5$. 
For Kolmogorov flow, the inputs and hidden features are 2D tensors 
of shape $n_x\times n_x$ with $d_x=n_x^2$, 
so we use wrap padding in both spatial directions 
with hidden width $w=64$, and $5\times5$ kernels.

\subsection{Training}
\label{app:training}
At each iteration step in training, we sample a minibatch
$\mathcal{T}_{\mathrm{batch}}\subset\mathcal{T}_{\mathrm{train}}$ of $M=32$
assimilation times and one Monte Carlo draw $(\bm{z}_{t,0},\tau_t)$ for each
$t\in\mathcal{T}_{\mathrm{batch}}$, 
which generates the unbiased estimator 
of \cref{eq:finite-training-loss} as below:
\begin{align}
    \widehat{\mathcal{L}}_{\mathrm{train}}(\theta)
    =\frac{1}{M}\sum_{t\in\mathcal{T}_{\mathrm{batch}}}
      \big\lVert
        \bm{v}_\theta(\bm{z}_{t,\tau_t}^{\mathrm{train}},\tau_t;\bm{c}_{t,\tau_t})
        -(\bm{x}_t^\star-\bm{z}_{t,0})
      \big\rVert_2^2.
    \label{eq:implementation-minibatch-loss}
\end{align}
We then compute the gradient 
$\nabla_\theta\widehat{\mathcal{L}}_{\mathrm{train}}$ 
of this minibatch loss with respect to the 
trainable parameters $\theta$ of the velocity network, 
and update $\theta$ using the Adam optimizer over $100$ epochs 
with cosine decay learning rate schedule 
at initial learning rate $10^{-3}$, 
and gradient clipping at $\ell_2$ norm $5$. 
All the training settings 
are listed in \Cref{tab:implementation-settings}.

\subsection{Inference and calibration}
\label{app:inference}
At inference, we integrate the ODE in
\cref{eq:flowda-analysis-ode} with $K=10$ Euler steps, 
which means we compute the velocity field at pseudo-times 
$\tau_k=k/K$ for $k=0,\ldots,K-1$. 
The static conditioning tensor $\bm{c}_t^{\mathrm{static}}$ is
computed once per assimilation time from the 
forecast and analysis ensembles 
$\{ \bm{X}_t^f, \bm{X}_t^a \}$ from the baseline filter 
and the observation $\bm{y}_t$, 
whereas the dynamic conditioning $\bm{c}_{t,\tau_k}^{(i)}$ 
is recomputed at every Euler step $k$ from the current 
transported particle $\bm{z}_{t,\tau_k}^{(i)}$ 
and the observation $\bm{y}_t$, 
and the velocity field 
$\bm{v}_\theta(\bm{z}_{t,\tau_k}^{(i)},\tau_k;\bm{c}_{t,\tau_k}^{(i)})$ 
is then used to update the particle to the next pseudo-time step
$\bm{z}_{t,\tau_{k+1}}^{(i)}$ via \cref{eq:ode-euler}. 
The dynamic conditioning tracks how far 
the transported particle is from the observation 
in the intermediate pseudo-time steps $\tau_k$. 
The raw analysis ensemble is obtained at final pseudo-time $\tau_K=1$ as 
$\bm{Z}_{t,1}:=[\bm{z}_{t,1}^{(1)},\ldots,\bm{z}_{t,1}^{(N)}]$.

After obtaining the raw analysis ensemble
$\bm{Z}_{t,1}$ as described above, we apply
a multiplicative factor
$\alpha_{\mathrm{flow}}$ to $\bm{Z}_{t,1}$ to
improve the ensemble spread and reduce underdispersion.
On the held-out validation set $\mathcal{T}_{\mathrm{val}}$, 
we search over a grid of 
$\alpha_{\mathrm{flow}}\in\{0.80,0.81,\ldots,2.00\}$, 
which is a $121$-point grid with
step $\Delta\alpha_{\mathrm{flow}}=0.01$, 
for the value of $\alpha_{\mathrm{flow}}$ 
by minimizing the CRPS defined in \cref{eq:ensemble-crps},
\begin{align}
    \alpha_{\mathrm{flow}}
    =\argmin_{\alpha\in\{0.80,0.81,\ldots,2.00\}}
    \operatorname{CRPS}_{\mathrm{val}}(\alpha).
    \label{eq:implementation-alpha-calibration}
\end{align}
Finally, we apply the calibrated inflation factor $\alpha_{\mathrm{flow}}$ to
each member of the raw analysis ensemble about its mean,
\begin{align}
    \hat{\bm{x}}_t^{a,(i)}
    =
    \hat{\bar{\bm{x}}}_t^a
    +\alpha_{\mathrm{flow}}\big(\bm{z}_{t,1}^{(i)}-\hat{\bar{\bm{x}}}_t^a\big),
    \qquad
    \hat{\bar{\bm{x}}}_t^a
    =\frac{1}{N}\sum_{i=1}^{N}\bm{z}_{t,1}^{(i)},
    \qquad i=1,\ldots,N,
    \label{eq:implementation-inflation}
\end{align}
where $\hat{\bar{\bm{x}}}_t^a$ is the ensemble mean 
and isn't changed by inflation. 
The final calibrated analysis ensemble
$\hat{\bm{X}}_t^a=[\hat{\bm{x}}_t^{a,(1)},\ldots,\hat{\bm{x}}_t^{a,(N)}]$ is
then used for evaluation on the test set $\mathcal{T}_{\mathrm{test}}$.
All the inference and calibration settings are also listed in \Cref{tab:implementation-settings}.

\begin{table}[!ht]
\centering
\caption{FlowEF settings for the experiments.}
\label{tab:implementation-settings}
\small
\setlength{\tabcolsep}{4pt}
\renewcommand{\arraystretch}{1.08}
\newcolumntype{V}{>{\centering\arraybackslash}p{0.15\linewidth}}
\begin{tabular}{p{0.43\linewidth}VVV}
\toprule
\textbf{Setting} & \textbf{Lorenz--96} & \textbf{KS} & \textbf{Kolmogorov} \\
\midrule
\multicolumn{4}{l}{\textbf{Data}} \\
Trajectory length $n_{\mathrm{total}}$ & \multicolumn{3}{c}{$25{,}000$} \\
Spin-up time steps $n_{\mathrm{spinup}}$ & $200$ & $2{,}200$ & $700$ \\
Training size $n_{\mathrm{train}}$ & \multicolumn{3}{c}{$10{,}000$} \\
Validation size $n_{\mathrm{val}}$ & \multicolumn{3}{c}{$2{,}000$} \\
Test size $n_{\mathrm{test}}$ & \multicolumn{3}{c}{$2{,}000$} \\
Training source & \multicolumn{3}{c}{$q_{t,0}^{\mathrm{train}}\sim\mathcal{N}\big(\bar{\bm{x}}_t^f,\bm{C}_{\mathrm{GC}}(r_{\mathrm{loc}})\circ\bm{P}_t^f\big)$ (\cref{eq:locgauss-dist})} \\
Inference source & \multicolumn{3}{c}{$\bm{z}_{t,0}^{(i)}=\bm{x}_t^{f,(i)}$, $i=1,\ldots,N$ (\cref{eq:ordered-deploy-source})} \\
Data seeds & \multicolumn{3}{c}{$1000,1010,\ldots,1090$} \\
\midrule
\multicolumn{4}{l}{\textbf{Architecture}} \\
Static conditioning channels $d_{\mathrm{static}}$ & \multicolumn{3}{c}{$12$} \\
Dynamic conditioning channels $d_{\mathrm{dyn}}$ & \multicolumn{3}{c}{$16$} \\
Total conditioning channels $d_c$ & \multicolumn{3}{c}{$28$} \\
Pseudo-time embedding dimension $d_\phi$ & \multicolumn{3}{c}{$4$} \\
Residual-block count $L$ of $\bm{v}_\theta$ & $3$ & $3$ & $3$ \\
Width $w$ of $\bm{v}_\theta$ & $32$ & $32$ & $64$ \\
Conv.\ kernel width $k$ of $\bm{v}_\theta$ and $\bm{c}_{t,\tau}^{\mathrm{dyn}}$ & $5$ & $5$ & $5\times5$ \\
Padding of $\bm{v}_\theta$ and $\bm{c}_{t,\tau}^{\mathrm{dyn}}$ & circular & circular & wrap \\
\midrule
\multicolumn{4}{l}{\textbf{Training}} \\
Optimizer & \multicolumn{3}{c}{Adam} \\
Learning-rate schedule & \multicolumn{3}{c}{Cosine decay} \\
Initial learning rate & \multicolumn{3}{c}{$10^{-3}$} \\
Training epochs & \multicolumn{3}{c}{$100$} \\
Minibatch size $M=|\mathcal{T}_{\mathrm{batch}}|$ & \multicolumn{3}{c}{$32$} \\
Grad.\ clipping ($\ell_2$ norm of $\nabla_\theta\mathcal{L}_{\mathrm{train}}$) & \multicolumn{3}{c}{$5$} \\
\midrule
\multicolumn{4}{l}{\textbf{Inference and calibration}} \\
Euler steps $K$ & \multicolumn{3}{c}{$10$} \\
Inflation factor $\alpha_{\mathrm{flow}}$ search grid & \multicolumn{3}{c}{$\alpha_{\mathrm{flow}}\in\{0.80,0.81,\ldots,2.00\}$} \\
\bottomrule
\end{tabular}
\end{table}

\begin{algorithm}[!htbp]
\renewcommand{\algorithmicrequire}{\textbf{Inputs:}}
\renewcommand{\algorithmicensure}{\textbf{Outputs:}}
\caption{FlowEF training and analysis update}
\label{alg:flowda}
\begin{algorithmic}[1]
\REQUIRE $\mathsf{SSM}$ $(f,h,\bm{H}_t,\bm{R})$; baseline
ensembles $\{ \bm{X}_t^f,\bm{X}_t^a \}$;
observations $\bm{y}_t$;
ground truth $\bm{x}_t^\star$;
training/validation/test splits
$\{ \mathcal{T}_{\mathrm{train}},
\mathcal{T}_{\mathrm{val}},
\mathcal{T}_{\mathrm{test}} \}$;
ensemble size $N$; Euler steps $K$
\ENSURE learned analysis ensembles
${\big\{\hat{\bm{X}}_t^a\big\}}_{t\in\mathcal{T}_{\mathrm{test}}}$
\STATE \textbf{Training:} initialize $\theta$
\FOR{minibatches $\mathcal{T}_{\mathrm{batch}}\subset\mathcal{T}_{\mathrm{train}}$}
    \STATE $\bm{z}_{t,0}\sim q_{t,0}^{\mathrm{train}}(\cdot\mid\bm{X}_t^f)$,
    $\tau_t\sim\operatorname{Uniform}(0,1)$ for
    $t\in\mathcal{T}_{\mathrm{batch}}$ in parallel
    \algref{\cref{eq:locgauss-dist}}
    \STATE $\bm{z}_{t,\tau_t}^{\mathrm{train}}=(1-\tau_t)\bm{z}_{t,0}+\tau_t\bm{x}_t^\star$
    \STATE $\bm{c}_t^{\mathrm{static}}=C(\bm{X}_t^f,\bm{X}_t^a,\bm{y}_t)$
    \algref{\cref{eq:static-condition}}
    \STATE $\bm{r}_{t,\tau_t}=\bm{H}_t^\top\bm{R}^{-1/2}
    (\bm{y}_t-\bm{H}_t\bm{z}_{t,\tau_t}^{\mathrm{train}})$,
    \quad
    $\bm{\Phi}(\tau_t)=\bm{1}_{d_x}\bm{\phi}(\tau_t)^\top$
    \algref{\cref{eq:obs-residual-field,eq:time-embedding}}
    \STATE $\bm{c}_{t,\tau_t}^{\mathrm{dyn}}=\operatorname{CNN}^{\mathrm{period}}_{\theta_{\mathrm{obs}}}
    ([\bm{r}_{t,\tau_t},\bm{m}_t,\bm{\Phi}(\tau_t)])$
    \algref{\cref{eq:dynamic-conditioning}}
    \STATE $\bm{c}_{t,\tau_t}=[\bm{c}_t^{\mathrm{static}},
    \bm{c}_{t,\tau_t}^{\mathrm{dyn}}]$
    \STATE $\bm{v}_{\theta}(\bm{z}_{t,\tau_t}^{\mathrm{train}},\tau_t;\bm{c}_{t,\tau_t})=
    \operatorname{CNN}^{\mathrm{period}}_{\theta_{\mathrm{field}}}
    ([\bm{z}_{t,\tau_t}^{\mathrm{train}},\bm{\Phi}(\tau_t),\bm{c}_{t,\tau_t}])$
    \algref{\cref{eq:velocity-cnn}}
    \STATE Compute $\mathcal{L}_{\mathrm{train}}$
    \algref{\cref{eq:finite-training-loss}}
    \STATE $\theta\leftarrow\operatorname{Adam}
    (\theta,\nabla_\theta\mathcal{L}_{\mathrm{train}})$
\ENDFOR
\STATE $\hat{\theta} \leftarrow \theta$
\STATE \textbf{Calibration:} $\alpha_{\mathrm{flow}}=\argmin_{\alpha}
\operatorname{CRPS}_{\mathrm{val}}(\alpha)$
\algref{\cref{eq:ensemble-crps}}
\STATE \textbf{Inference:}
$\bm{z}_{t,0}^{(i)}=\bm{x}_t^{f,(i)}$,
$\bm{c}_t^{\mathrm{static}}=C(\bm{X}_t^f,\bm{X}_t^a,\bm{y}_t)$ for
$t\in\mathcal{T}_{\mathrm{test}}$, $i=1,\ldots,N$ in parallel
\FOR{$k=0,\ldots,K-1$}
    \STATE $\tau_k=k/K$
    \STATE $\bm{r}_{t,\tau_k}^{(i)}=\bm{H}_t^\top\bm{R}^{-1/2}
    (\bm{y}_t-\bm{H}_t\bm{z}_{t,\tau_k}^{(i)})$,
    \quad
    $\bm{\Phi}(\tau_k)=\bm{1}_{d_x}\bm{\phi}(\tau_k)^\top$
    \algref{\cref{eq:obs-residual-field,eq:time-embedding}}
    \STATE $\bm{c}_{t,\tau_k}^{\mathrm{dyn},(i)}=\operatorname{CNN}^{\mathrm{period}}_{\hat{\theta}_{\mathrm{obs}}}
    ([\bm{r}_{t,\tau_k}^{(i)},\bm{m}_t,\bm{\Phi}(\tau_k)])$
    \algref{\cref{eq:dynamic-conditioning}}
    \STATE $\bm{c}_{t,\tau_k}^{(i)}=[\bm{c}_t^{\mathrm{static}},
    \bm{c}_{t,\tau_k}^{\mathrm{dyn},(i)}]$
    \STATE $\bm{v}_{t,\tau_k}^{(i)}=\bm{v}_{\hat{\theta}}
    (\bm{z}_{t,\tau_k}^{(i)},\tau_k;\bm{c}_{t,\tau_k}^{(i)})$,
    \quad
    $\bm{z}_{t,\tau_{k+1}}^{(i)}=\bm{z}_{t,\tau_k}^{(i)}
    +K^{-1}\bm{v}_{t,\tau_k}^{(i)}$
    \algref{\cref{eq:flowda-analysis-ode}}
\ENDFOR
\STATE $\hat{\bar{\bm{x}}}_t^a = N^{-1}\sum_{i=1}^{N}\bm{z}_{t,1}^{(i)}$,
\quad
$\hat{\bm{x}}_t^{a,(i)} = \hat{\bar{\bm{x}}}_t^a+
\alpha_{\mathrm{flow}}(\bm{z}_{t,1}^{(i)}-\hat{\bar{\bm{x}}}_t^a)$
\algref{\cref{eq:flowda-inflation}}
\STATE \textbf{return}
$\big\{
    \hat{\bm{X}}_t^a =
    [\hat{\bm{x}}_t^{a,(1)},\ldots,\hat{\bm{x}}_t^{a,(N)}]
\big\}_{t\in\mathcal{T}_{\mathrm{test}}}$
\end{algorithmic}
\end{algorithm}

\subsection{Metrics}
\label{app:metrics}
In \Cref{sec:setup}, we evaluate the performance of FlowEF 
using four metrics that measure the 
different aspects of the quality 
of the analysis ensemble. 
The metrics are computed on the test set 
$\mathcal{T}_{\mathrm{test}}$ 
and averaged over $10$ seeds. 
Precisely, let 
$T_{\mathrm{test}}=|\mathcal{T}_{\mathrm{test}}|$, write
$\hat{x}_{t,j}^{a,(i)}$ for member $i$ and state coordinate $j$ 
of the final analysis ensemble at time $t\in\mathcal{T}_{\mathrm{test}}$,  
and set $\hat{\bar{x}}_{t,j}^a=N^{-1}\sum_{i=1}^N \hat{x}_{t,j}^{a,(i)}$. 
Each metric below is computed for a single seed 
and then aggregated as mean $\pm$ standard deviation 
over $10$ seeds under the experimental protocol 
described in \Cref{sec:setup}. 
The four metrics are RMSE, CRPS, SSR, and $\operatorname{Cov}_{95}$,
which are defined below.

\paragraph{RMSE}
RMSE measures the point accuracy of the analysis ensemble mean, 
and is defined as
\begin{align}
\operatorname{RMSE}
&=\frac{1}{T_{\mathrm{test}}}\sum_{t\in\mathcal{T}_{\mathrm{test}}}
\left[\frac{1}{d_x}\sum_{j=1}^{d_x}
\bigl(\hat{\bar{x}}_{t,j}^a-x_{t,j}^\star\bigr)^2\right]^{1/2},
\label{eq:test-rmse}
\end{align}
where $x_{t,j}^\star$ is the ground truth for state coordinate $j$ at time $t$. 

\paragraph{CRPS}
CRPS is a proper scoring rule on the full ensemble that jointly measures
accuracy and how tightly the predictive ensemble concentrates around
the truth. The CRPS between the analysis ensemble $\hat{\bm{X}}_t^a$ 
and the ground truth $\bm{x}_t^\star$ is defined as
\begin{align*}
\operatorname{CRPS}_t(\hat{\bm{X}}_t^a,\bm{x}_t^\star)
=\frac{1}{d_x}\sum_{\ell=1}^{d_x}
\left[
  \frac{1}{N}\sum_{i=1}^{N}
  \left|\hat{x}_{t,\ell}^{a,(i)}-x_{t,\ell}^\star\right|
  -\frac{1}{2N^2}\sum_{i=1}^{N}\sum_{j=1}^{N}
  \left|\hat{x}_{t,\ell}^{a,(i)}-\hat{x}_{t,\ell}^{a,(j)}\right|
\right],
\end{align*}
which is averaged uniformly over test times and state coordinates. 
Unlike RMSE, CRPS is minimized only when the 
ensemble's predictive distribution matches the 
true data generating distribution.

\paragraph{SSR}
SSR is a calibration diagnostic that compares the ensemble's own spread to
its actual error. It's defined as 
\begin{align}
\operatorname{SSR}
&=\left[
\frac{(T_{\mathrm{test}}d_x)^{-1}\sum_{t,j}s_{t,j}^2}
     {(T_{\mathrm{test}}d_x)^{-1}\sum_{t,j}(\hat{\bar{x}}_{t,j}^a-x_{t,j}^\star)^2}
\right]^{1/2},
\qquad
s_{t,j}^2
=\frac{1}{N}\sum_{i=1}^N
\bigl(\hat{x}_{t,j}^{a,(i)}-\hat{\bar{x}}_{t,j}^a\bigr)^2,
\label{eq:test-ssr}
\end{align}
which computes the square root of the ratio 
between the time- and coordinate-averaged 
ensemble variance $s_{t,j}^2$ and the 
time- and coordinate-averaged 
squared error of the ensemble mean. 
SSR${}=1$ is the reference value for a 
perfectly dispersed ensemble, 
SSR${}<1$ indicates underdispersion, 
which means the ensemble is overconfident 
relative to its actual error, 
and SSR${}>1$ indicates overdispersion, 
which means the ensemble is underconfident.

\paragraph{Coverage probability}
$\operatorname{Cov}_{95}$ checks whether
the truth actually falls inside the
ensemble's $95\%$ empirical uncertainty interval. 
It's defined as
\begin{align}
\operatorname{Cov}_{95}
&=\frac{1}{T_{\mathrm{test}}d_x}
\sum_{t=1}^{T_{\mathrm{test}}}\sum_{j=1}^{d_x}
\mathds{1} \left\{
\widehat q_{t,j}(0.025)\leq x_{t,j}^\star
\leq\widehat q_{t,j}(0.975)\right\},
\label{eq:test-cov95}
\end{align}
where $\mathds{1}\{\cdot\}$ is the indicator function and $\widehat q_{t,j}(p)$ is the
linearly interpolated empirical $p$-quantile of the 
$N$ analysis ensemble members $\{ \hat{\bm{x}}_t^{a,(i)} \}_{i=1}^N$. 
We compute quantity coverage 
at each test time and state coordinate independently 
and then averaged over all times and coordinates,
it's therefore a marginal and coordinatewise coverage 
rather than joint coverage of the full state vector at once. 
$\operatorname{Cov}_{95}=0.95$ is the reference value, 
$\operatorname{Cov}_{95}<0.95$ indicate undercoverage, 
which means the ensemble interval is too narrow, 
and values above indicate overcoverage, 
which means the ensemble interval is too wide.

\section{Benchmark comparison}
\label{app:benchmarks}

In this section, we compare FlowEF against all DA methods considered
in the experiments. 
\Cref{tab:benchmark-methods} lists the representative benchmark methods we used 
throughout the experiments, together with a brief description for each method. 
We show the comparison results for 
low-dimensional Lorenz--96 system in \Cref{tab:bench-lorenz96},
for high-dimensional Lorenz--96 system in \Cref{tab:bench-lorenz96-large}, 
for Kuramoto--Sivashinsky equation in \Cref{tab:bench-ks}, 
and for Kolmogorov flow in \Cref{tab:bench-kol}.

\begingroup
\footnotesize
\renewcommand{\baselinestretch}{1}\selectfont
\setlength{\tabcolsep}{3pt}
\renewcommand{\arraystretch}{1.2}
\newcommand{\catlab}[2]{\multirow{#1}{*}{\rotatebox[origin=c]{90}{\scriptsize#2}}}
\newcommand{\pad}[1]{\rule[-0.42\dimexpr#1\relax]{0pt}{#1}}
\begin{longtable}{%
  c%
  >{\raggedright\arraybackslash}p{0.36\linewidth}%
  >{\raggedright\arraybackslash\hspace{0pt}}p{0.56\linewidth}}
\caption{Representative benchmark methods for the experiments.}\label{tab:benchmark-methods}\\
\toprule
& Method & Description \\
\midrule
\endfirsthead
\toprule & Method & Description \\
\midrule
\endhead
\catlab{4}{Deterministic}
 & ADEnKF \citep{chen2022autodifferentiable} & Auto-differentiable ensemble Kalman filter. \\*
 & PSEF \citep{bach2026learning} & Proper scoring ensemble filter. \\*
 & DAN \citep{boudier2023data} & Data assimilation network. \\*
 & 4DVarNet \citep{fablet2021learning} & Four-dimensional variational network. \\
\midrule
\catlab{4}{Stochastic}
 & SDA \citep{rozet2023sda} & Score-based data assimilation. \\*
 & CCDDA \citep{binder2026ccdda} & Closed-form conditional diffusion data assimilation. \\*
 & DAISI \citep{andrae2026daisi} & Inverse sampling with stochastic interpolants. \\*
 & GenDA \citep{martin2025genda} & Generative data assimilation. \\
\midrule
\catlab{1}{Flow}
 & \pad{1.8em}FlowDAS \citep{chen2025flowdas} & Flow assimilation with stochastic interpolants. \\
\bottomrule
\end{longtable}
\endgroup

\begin{table}[!ht]
\centering
\caption{Benchmark comparison of FlowEF against all methods on the test
set (mean $\pm$ std over $10$ seeds for every method)
for the low-dimensional Lorenz--96 system ($d_x=40$, $d_y=10$, $N=30$).}
\label{tab:bench-lorenz96}
\footnotesize
\setlength{\tabcolsep}{1pt}
\renewcommand{\arraystretch}{0.92}
\resizebox{\linewidth}{!}{%
\begin{tabular}{lcccc}
\toprule
Method & RMSE $\downarrow$ & CRPS $\downarrow$ & SSR ($\to1$) & Cov$_{95}$ ($\to0.95$) \\
\midrule
Baseline (DEnKF) \citep{sakov2008deterministic} & $0.894 \pm 0.154$ & $0.438 \pm 0.069$ & $0.753 \pm 0.166$ & $0.893 \pm 0.026$ \\
\rowcolor{gray!15}
\textbf{FlowEF (ours)} & $\mathbf{0.873 \pm 0.127}$ & $\mathbf{0.420 \pm 0.053}$ & $\mathbf{0.948 \pm 0.145}$ & $\mathbf{0.905 \pm 0.020}$ \\
\midrule
ADEnKF \citep{chen2022autodifferentiable} & $0.891 \pm 0.152$ & $0.434 \pm 0.069$ & $0.770 \pm 0.256$ & $0.868 \pm 0.047$ \\
PSEF \citep{bach2026learning} & $0.986 \pm 0.168$ & $0.482 \pm 0.084$ & $0.779 \pm 0.181$ & $0.857 \pm 0.045$ \\
DAN \citep{boudier2023data} & $1.004 \pm 0.148$ & $0.545 \pm 0.074$ & $0.530 \pm 0.104$ & $0.627 \pm 0.028$ \\
4DVarNet \citep{fablet2021learning} & $1.114 \pm 0.130$ & $0.580 \pm 0.065$ & $0.476 \pm 0.076$ & $0.674 \pm 0.046$ \\
\midrule
SDA \citep{rozet2023sda} & $1.005 \pm 0.132$ & $0.659 \pm 0.066$ & $0.024 \pm 0.004$ & $0.082 \pm 0.007$ \\
CCDDA \citep{binder2026ccdda} & $1.118 \pm 0.141$ & $0.754 \pm 0.076$ & $0.186 \pm 0.033$ & $0.102 \pm 0.007$ \\
DAISI \citep{andrae2026daisi} & $0.939 \pm 0.147$ & $0.483 \pm 0.070$ & $0.581 \pm 0.120$ & $0.713 \pm 0.031$ \\
GenDA \citep{martin2025genda} & $0.938 \pm 0.149$ & $0.461 \pm 0.067$ & $0.747 \pm 0.145$ & $0.693 \pm 0.025$ \\
\midrule
FlowDAS \citep{chen2025flowdas} & $0.994 \pm 0.130$ & $0.622 \pm 0.063$ & $0.174 \pm 0.053$ & $0.259 \pm 0.056$ \\
\bottomrule
\end{tabular}%
}
\end{table}

\begin{table}[!ht]
\centering
\caption{Benchmark comparison of FlowEF against all methods on the test
set (mean $\pm$ std over $10$ seeds for every method)
for the high-dimensional Lorenz--96 system ($d_x=100$, $d_y=20$, $N=30$).}
\label{tab:bench-lorenz96-large}
\footnotesize
\setlength{\tabcolsep}{1pt}
\renewcommand{\arraystretch}{0.92}
\resizebox{\linewidth}{!}{%
\begin{tabular}{lcccc}
\toprule
Method & RMSE $\downarrow$ & CRPS $\downarrow$ & SSR ($\to1$) & Cov$_{95}$ ($\to0.95$) \\
\midrule
Baseline (DEnKF) \citep{sakov2008deterministic} & $2.150 \pm 0.102$ & $1.030 \pm 0.056$ & $0.758 \pm 0.234$ & $0.848 \pm 0.059$ \\
\rowcolor{gray!15}
\textbf{FlowEF (ours)} & $\mathbf{1.886 \pm 0.081}$ & $\mathbf{0.874 \pm 0.044}$ & $\mathbf{0.969 \pm 0.083}$ & $\mathbf{0.886 \pm 0.020}$ \\
\midrule
ADEnKF \citep{chen2022autodifferentiable} & $2.136 \pm 0.102$ & $1.014 \pm 0.057$ & $0.723 \pm 0.232$ & $0.815 \pm 0.061$ \\
PSEF \citep{bach2026learning} & $2.212 \pm 0.101$ & $1.070 \pm 0.063$ & $0.785 \pm 0.215$ & $0.826 \pm 0.064$ \\
DAN \citep{boudier2023data} & $2.211 \pm 0.096$ & $1.190 \pm 0.062$ & $0.419 \pm 0.122$ & $0.504 \pm 0.056$ \\
4DVarNet \citep{fablet2021learning} & $2.309 \pm 0.079$ & $1.217 \pm 0.049$ & $0.401 \pm 0.110$ & $0.557 \pm 0.055$ \\
\midrule
SDA \citep{rozet2023sda} & $2.218 \pm 0.092$ & $1.417 \pm 0.077$ & $0.045 \pm 0.005$ & $0.144 \pm 0.028$ \\
CCDDA \citep{binder2026ccdda} & $2.654 \pm 0.238$ & $1.735 \pm 0.189$ & $0.118 \pm 0.021$ & $0.043 \pm 0.004$ \\
DAISI \citep{andrae2026daisi} & $2.159 \pm 0.094$ & $1.052 \pm 0.054$ & $0.694 \pm 0.206$ & $0.773 \pm 0.070$ \\
GenDA \citep{martin2025genda} & $2.169 \pm 0.096$ & $1.032 \pm 0.054$ & $0.743 \pm 0.189$ & $0.701 \pm 0.053$ \\
\midrule
FlowDAS \citep{chen2025flowdas} & $2.144 \pm 0.082$ & $1.292 \pm 0.061$ & $0.215 \pm 0.107$ & $0.257 \pm 0.052$ \\
\bottomrule
\end{tabular}%
}
\end{table}

\begin{table}[!ht]
\centering
\caption{Benchmark comparison of FlowEF against all methods on the test
set (mean $\pm$ std over $10$ seeds for every method)
for Kuramoto--Sivashinsky ($d_x=128$, $d_y=16$, $N=20$), against the DEnKF
baseline.}
\label{tab:bench-ks}
\footnotesize\setlength{\tabcolsep}{1pt}\renewcommand{\arraystretch}{0.92}
\resizebox{\linewidth}{!}{\begin{tabular}{lcccc}
\toprule Method & RMSE $\downarrow$ & CRPS $\downarrow$ & SSR ($\to1$) & Cov$_{95}$ ($\to0.95$) \\
\midrule
Baseline (DEnKF) \citep{sakov2008deterministic} & $0.502 \pm 0.032$ & $0.257 \pm 0.018$ & $0.826 \pm 0.093$ & $0.840 \pm 0.023$ \\
\rowcolor{gray!15}\textbf{FlowEF (ours)} & $\mathbf{0.483 \pm 0.024}$ & $\mathbf{0.248 \pm 0.013}$ & $\mathbf{0.939 \pm 0.073}$ & $0.854 \pm 0.021$ \\
\midrule
ADEnKF \citep{chen2022autodifferentiable} & $0.504 \pm 0.031$ & $0.257 \pm 0.018$ & $0.870 \pm 0.075$ & $\mathbf{0.872 \pm 0.016}$ \\
PSEF \citep{bach2026learning} & $0.570 \pm 0.057$ & $0.292 \pm 0.030$ & $0.894 \pm 0.292$ & $0.843 \pm 0.022$ \\
DAN \citep{boudier2023data} & $0.551 \pm 0.032$ & $0.285 \pm 0.019$ & $0.650 \pm 0.059$ & $0.750 \pm 0.022$ \\
4DVarNet \citep{fablet2021learning} & $0.529 \pm 0.024$ & $0.271 \pm 0.014$ & $0.679 \pm 0.049$ & $0.771 \pm 0.023$ \\
\midrule
SDA \citep{rozet2023sda} & $0.527 \pm 0.027$ & $0.341 \pm 0.017$ & $0.059 \pm 0.004$ & $0.111 \pm 0.005$ \\
CCDDA \citep{binder2026ccdda} & $0.544 \pm 0.033$ & $0.305 \pm 0.020$ & $0.613 \pm 0.056$ & $0.607 \pm 0.018$ \\
DAISI \citep{andrae2026daisi} & $0.536 \pm 0.031$ & $0.277 \pm 0.018$ & $0.779 \pm 0.078$ & $0.784 \pm 0.024$ \\
GenDA \citep{martin2025genda} & $0.503 \pm 0.028$ & $0.255 \pm 0.016$ & $0.777 \pm 0.074$ & $0.795 \pm 0.023$ \\
\midrule
FlowDAS \citep{chen2025flowdas} & $0.508 \pm 0.021$ & $0.310 \pm 0.013$ & $0.160 \pm 0.008$ & $0.268 \pm 0.007$ \\
\bottomrule\end{tabular}}
\end{table}

\begin{table}[!ht]
\centering
\caption{Benchmark comparison of FlowEF against all methods on the test
set (mean $\pm$ std over $10$ seeds for every method)
for Kolmogorov flow ($d_x=64{\times}64=4096$, $d_y=8{\times}8=64$, $N=30$),
against the LETKF baseline.}
\label{tab:bench-kol}
\footnotesize
\setlength{\tabcolsep}{1pt}
\renewcommand{\arraystretch}{0.92}
\resizebox{\linewidth}{!}{\begin{tabular}{lcccc}
\toprule
Method & RMSE $\downarrow$ & CRPS $\downarrow$ & SSR ($\to1$) & Cov$_{95}$ ($\to0.95$) \\
\midrule
Baseline (LETKF) \citep{hunt2007efficient} & $0.547 \pm 0.209$ & $0.331 \pm 0.145$ & $0.548 \pm 0.331$ & $0.835 \pm 0.054$ \\
\rowcolor{gray!15}
\textbf{FlowEF (ours)} & $\mathbf{0.476 \pm 0.165}$ & $\mathbf{0.273 \pm 0.104}$ & $\mathbf{0.699 \pm 0.259}$ & $\mathbf{0.877 \pm 0.029}$ \\
\midrule
ADEnKF \citep{chen2022autodifferentiable} & $0.543 \pm 0.203$ & $0.328 \pm 0.143$ & $0.569 \pm 0.330$ & $0.843 \pm 0.051$ \\
PSEF \citep{bach2026learning} & $0.555 \pm 0.203$ & $0.335 \pm 0.143$ & $0.583 \pm 0.354$ & $0.840 \pm 0.051$ \\
DAN \citep{boudier2023data} & $0.577 \pm 0.212$ & $0.397 \pm 0.156$ & $0.146 \pm 0.085$ & $0.322 \pm 0.022$ \\
4DVarNet \citep{fablet2021learning} & $0.552 \pm 0.205$ & $0.377 \pm 0.150$ & $0.152 \pm 0.089$ & $0.334 \pm 0.018$ \\
\midrule
CCDDA \citep{binder2026ccdda} & $0.550 \pm 0.205$ & $0.333 \pm 0.145$ & $0.573 \pm 0.335$ & $0.847 \pm 0.050$ \\
DAISI \citep{andrae2026daisi} & $0.555 \pm 0.206$ & $0.338 \pm 0.146$ & $0.547 \pm 0.318$ & $0.815 \pm 0.050$ \\
\midrule
FlowDAS \citep{chen2025flowdas} & $0.514 \pm 0.168$ & $0.314 \pm 0.119$ & $0.402 \pm 0.165$ & $0.703 \pm 0.059$ \\
\bottomrule
\end{tabular}}
\end{table}

\begin{remark}[Interpretation of benchmark performance]
\Cref{tab:bench-lorenz96,tab:bench-lorenz96-large,tab:bench-ks,tab:bench-kol}
show that many benchmark methods perform worse than the tuned classical
baseline in both RMSE and CRPS. All benchmark methods are implemented through
the same one-step filtering framework: each method receives a forecast
ensemble from a baseline filter and an observation at an assimilation time,
and finally returns a learned analysis ensemble for that time. 
The results are therefore a comparison of adapted implementation
of the benchmark merhods under this common framework 
instead of exact reproductions of each method.

The tuned classical filters are a strong reference here: 
they use covariance localization and multiplicative prior inflation, 
with the localization radius and inflation factor 
jointly selected by validation CRPS. 
The experiments also use a linear, sparse observation model, 
which favors methods built for this 
setting over methods designed for nonlinear observations or strongly
non-Gaussian filtering distributions.
Deterministic learned methods such as ADEnKF \citep{chen2022autodifferentiable}
and PSEF \citep{bach2026learning} apply a deterministic learned analysis update;
DAN \citep{boudier2023data} additionally carries
a learned hidden state across assimilation times.
FlowEF instead recomputes the observation residual of the current particle
and uses it to update the conditioning variables,
so the analysis update can adjust dynamically along the flow path.
Training-free stochastic methods such as CCDDA \citep{binder2026ccdda}
also initialize from the forecast ensemble,
but they estimate a conditional score from finite ensemble members
by kernel density estimation.
This estimate can suffer from large approximation errors
when a small ensemble sparsely samples a high-dimensional state space.
Guided diffusion and flow samplers such as SDA
\citep{rozet2023sda}, DAISI \citep{andrae2026daisi}, GenDA
\citep{martin2025genda}, and FlowDAS \citep{chen2025flowdas} likewise
incorporate observations during sampling.
However, their architectures were developed around trajectory inference,
a stationary generative prior, snapshot reconstruction, or learned transition dynamics,
rather than modeling a baseline-specific forecast-to-analysis update.

Overall, many benchmark methods target different settings or objectives, 
including nonlinear or non-Gaussian distribution, 
reconstruction from complex observations, 
learning unknown dynamics, and computational scalability. 
They were not designed specifically for the forecast-to-analysis framework used here. 
The adapted implementations of these benchmark methods therefore 
face task mismatch and approximation error 
while competing with a localized and inflated classical filter 
tuned for this setting. 
FlowEF is designed for this framework and additionally
calibrates its output inflation by validation CRPS. 
Thus, these tables compare performance of the analysis 
update from fixed baseline ensemble forecast, 
instead of the recursive filtering performance or the 
methods in their native settings.
\end{remark}

\section{Convergence theory of FlowEF}
\label{app:theory}
In this section, we prove the error bound of
\Cref{thm:convergence-informal} in \Cref{sec:convergence}, 
which decomposes the error of FlowEF into six terms: 
conditioning bias, finite-ensemble feature error, 
flow matching error, Euler error, 
source mismatch, and inflation error. 
We introduce notation in \Cref{app:theory-setup}, 
state the assumptions in \Cref{app:theory-assumptions}, 
prove the supporting lemmas in \Cref{app:theory-lemmas}, 
and present the explicit bound in \Cref{app:theory-main}. 
Flow matching recovers its conditional target distribution 
exactly when the training source distribution matches the
inference source distribution, the velocity is learned exactly, and
the ODE is solved exactly \citep{lipman2023flow,tong2024improving}. 
\Cref{tab:theory-dictionary} explains the role of each error term 
and how they are controlled by the assumptions.

\subsection{Notations}
\label{app:theory-setup}
Given an assimilation time $t$, 
we define the notations and derive the results 
with respect to it. 
Let $\tau\in[0,1]$ index the pseudo-time of the flow matching transport. 
All random variables below are defined on one probability space
$(\Omega,\mathcal{F},\mathbb{P})$. 
For random variables $Z_1,\ldots,Z_k$, collected as
$Z:=(Z_1,\ldots,Z_k)$, we write $\sigma(Z)\subseteq\mathcal{F}$ for the
\emph{$\sigma$-algebra generated by} them, that is, the smallest
sub-$\sigma$-algebra of $\mathcal{F}$ 
with respect to which every $Z_j$ is measurable. 
For the finite-dimensional random variables, 
the Doob--Dynkin lemma gives
\begin{align}
    W\ \text{is }\sigma(Z)\text{-measurable}
    \iff
    W=g(Z)\ \text{for some Borel function } g ,
    \label{eq:theory-doob-dynkin}
\end{align}
therefore $\sigma(Z)$ represents the information contained in $Z$.
In particular $\sigma(V)\subseteq\sigma(Z)$ means that $V$
is a Borel-measurable function of $Z$, thus $Z$ carries at least as much
information as $V$.

\subsubsection{Filtering laws notations}
\label{app:theory-filtering-laws}

We denote every law on the state space by $\pi$. As in
\Cref{sec:bayesian-filtering}, let
\begin{align}
    \pi_t^f(d\bm{x})
    =
    \mathbb{P}(\bm{x}_t\in d\bm{x}\mid \bm{y}_{1:t-1}),
    \qquad
    \pi_t^a(d\bm{x})
    =
    \mathbb{P}(\bm{x}_t\in d\bm{x}\mid \bm{y}_{1:t})
    \label{eq:theory-forecast-analysis-laws}
\end{align}
denote the exact forecast and analysis laws. As in
\Cref{sec:ensemble-da}, the empirical laws of 
a baseline ensemble filter $B$ are
\begin{align}
    \widehat{\pi}_t^{B,f}
    =\frac1N\sum_{i=1}^N\delta_{\bm{x}_t^{f,(i)}},
    \qquad
    \widehat{\pi}_t^{B,a}
    =\frac1N\sum_{i=1}^N\delta_{\bm{x}_t^{a,(i)}},
    \label{eq:theory-baseline-laws}
\end{align}
where $\delta_x$ is the Dirac measure at $x$, 
and $\bm{x}_t^{f,(i)},\bm{x}_t^{a,(i)}$ are the 
forecast and analysis ensemble members 
of the baseline filter $B$. 
We define the FlowEF output laws as
\begin{align}
    \widetilde{\pi}_t^a := \frac1N\sum_{i=1}^N\delta_{\bm z_{t,1}^{(i)}},
    \qquad
    \widehat{\pi}_t^a := \frac1N\sum_{i=1}^N\delta_{\hat{\bm x}_t^{a,(i)}},
    \label{eq:theory-flowef-laws}
\end{align}
where $\bm z_{t,1}^{(i)}$ are transport endpoints, 
also called FlowEF raw analysis ensemble, 
and $\hat{\bm x}_t^{a,(i)}$ is the FlowEF analysis ensemble 
after inflation in \cref{eq:flowda-inflation}. 
Let $\bm{x}_t^{B,f,\infty}$ denote a forecast ensemble member of the
infinite-ensemble baseline filter $B$.  Its conditional forecast law is
$
\pi_t^{B,f,\infty}(d\bm{x})
:= \mathbb{P}\big(\bm{x}_t^{B,f,\infty}\in d\bm{x}
\mid \bm y_{1:t-1}\big)
$. 
For 
$\bm{x}_t^{B,f,\infty,(1)},\ldots,\bm{x}_t^{B,f,\infty,(N)}
\overset{\mathrm{i.i.d.}}{\sim}\pi_t^{B,f,\infty}$, 
we define the associated Monte Carlo empirical measure as
\begin{align}
    \widehat{\pi}_t^{B,f,\infty}
    := \frac1N\sum_{i=1}^N\delta_{\bm{x}_t^{B,f,\infty,(i)}}.
    \label{eq:theory-limiting-forecast-law}
\end{align}

\subsubsection{Conditional variables notations}
\label{app:theory-conditional-setup}

The static conditioning tensor
$\bm{c}_t^{\mathrm{static}}=C(\bm X_t^f,\bm X_t^a,\bm y_t)$ of
\cref{eq:condition-tensor} depends on the two ensembles only through their
coordinatewise first two moments: with $\bm\sigma_t^f,\bm\sigma_t^a$ the
coordinatewise forecast and analysis standard deviations of \Cref{sec:cond},
the channel definitions
\cref{eq:forecast-observation-condition,eq:baseline-action-condition} read
\begin{align}
    C(\bm X_t^f,\bm X_t^a,\bm y_t)
    = \Gamma_t\big(\bar{\bm x}_t^f,\bm\sigma_t^f,
      \bar{\bm x}_t^a,\bm\sigma_t^a,\bm y_t\big),
    \label{eq:theory-moment-map}
\end{align}
where $\Gamma_t$ is the deterministic map defined by those channel formulas
at the fixed design $\bm H_t,\bm R$. The infinite-population counterpart
evaluates the same $\Gamma_t$ at the moments of the exact laws
\cref{eq:theory-forecast-analysis-laws} in place of the ensemble moments,
\begin{align}
    \bm c_t^{\mathrm{static},\infty}
    := C_\infty(\pi_t^f,\pi_t^a,\bm{y}_t)
    := \Gamma_t\big(\bm\mu_t^f,\bm s_t^f,\bm\mu_t^a,\bm s_t^a,\bm{y}_t\big),
    \label{eq:theory-finite-conditioning}
\end{align}
where for $\bullet\in\{f,a\}$,
$\bm\mu_t^\bullet:=\mathbb{E}_{\pi_t^\bullet}[\bm x]$ is the mean and
$\bm s_t^\bullet:=\{\operatorname{diag}(\operatorname{Cov}_{\pi_t^\bullet}(\bm x))\}^{1/2}$
is the standard deviation under $\pi_t^\bullet$.
Thus $C$ and $C_\infty$ are the same map $\Gamma_t$ evaluated at ensemble
and at population moments respectively.
The superscript $\infty$ here means \emph{infinite-population} instead of 
\emph{infinite-ensemble}. $C_\infty$ is evaluated at the exact Bayesian laws
$\pi_t^f,\pi_t^a$ of \cref{eq:theory-forecast-analysis-laws}, not at the
$N\to\infty$ limit $\pi_t^{B,f,\infty}$ of the baseline filter $B$. 
Let 
\begin{align}
    \mathcal{G}_t:=\sigma\big(\bm X_t^f,\bm X_t^a,\bm y_t\big)\subseteq\mathcal{F}
    \label{eq:theory-baseline-info}
\end{align}
be the sub-$\sigma$-algebra generated by the baseline forecast ensemble, the
baseline analysis ensemble, and the observation. 
We call it the \emph{baseline information} at cycle $t$, 
since a random variable is $\mathcal{G}_t$-measurable 
when it is a deterministic function of 
$\bm X_t^f,\bm X_t^a,\bm y_t$. 
The static tensor $\bm c_t^{\mathrm{static}}$ is such a function by 
\cref{eq:theory-moment-map}, so
$\sigma(\bm c_t^{\mathrm{static}})\subseteq\mathcal{G}_t$.

The dynamic channel \cref{eq:dynamic-conditioning} is recomputed at each Euler
step from the residual \cref{eq:obs-residual-field},
we therefore carry the following variables as input for dynamic channels: 
\begin{align}
    \bm o_t := \big[\bm y_t,\bm H_t,\bm R\big].
    \label{eq:theory-exogenous-design}
\end{align}
We define the corresponding finite-ensemble and infinite-population 
conditioning variables as follows:
\begin{align}
    \mathcal{C}_t^N := \big[\bm c_t^{\mathrm{static}},\bm o_t\big],
    \qquad
    \mathcal{C}_t^\infty := \big[\bm c_t^{\mathrm{static},\infty},\bm o_t\big], 
    \label{eq:theory-conditioning-record}
\end{align}
which are both $\mathcal{G}_t$-measurable. 
Note that only static component carries ensemble estimation error, 
since $\bm o_t$ is not estimated from the $N$ ensemble members. 
Hence, the feature term of \Cref{thm:flowef} 
involves only $\bm c_t^{\mathrm{static}}$. 
Since $\sigma(\bm c_t^{\mathrm{static}})\subseteq\sigma(\mathcal{C}_t^N)$,
conditioning on $\mathcal{C}_t^N$ rather than 
$\bm c_t^{\mathrm{static}}$ alone can only reduce the bias of
\Cref{lem:bias}.

The velocity field network in \cref{eq:implementation-trunk-input} reads
$[\bm z_{t,\tau},\bm\Phi(\tau),\bm c_t^{\mathrm{static}},\bm
c_{t,\tau}^{\mathrm{dyn}}]$ as network inputs.  
Accordingly we distinguish the three generated $\sigma$-algebras
\begin{align}
    \mathcal{F}_\tau^{\mathrm{stat}}
    :=\sigma\big(\bm z_{t,\tau},\tau,\bm c_t^{\mathrm{static}}\big)
    \subseteq
    \mathcal{F}_\tau^{\mathrm{dyn}}
    :=\sigma\big(\bm z_{t,\tau},\tau,\bm c_t^{\mathrm{static}},\bm c_{t,\tau}^{\mathrm{dyn}}\big)
    \subseteq
    \mathcal{F}_\tau^{\mathrm{rec}}
    :=\sigma\big(\bm z_{t,\tau},\tau,\mathcal{C}_t^N\big).
    \label{eq:theory-filtrations}
\end{align}
The velocity network in \Cref{alg:flowda} 
is $\mathcal{F}_\tau^{\mathrm{dyn}}$-measurable. 
The CFM target is indexed by the record $\mathcal{C}_t^N$, not the dynamic feature. 
The dynamic feature is a $\tau$-dependent function of the 
transported state itself, so conditioning on it would define no probability path. 
$\mathcal{C}_t^N$ is the coarsest transport-invariant variable,  
which makes the dynamic feature a deterministic function of the state
(\cref{eq:theory-dyn-measurability}). 
\subsubsection{Training and inference notations}
\label{app:theory-transport-setup}
The training source \cref{eq:locgauss-dist} depends on the full ensemble
information, which is $\mathcal{G}_t$-measurable. 
In contrast, the CFM objective conditions only on $\mathcal{C}_t^N$, 
a lower-dimensional summary of that information. 
For $\mathcal{C}_t^N=[\bm c,\bm o]$, 
we define the conditional training source law by
\begin{align}
    q_{t,0}^{\mathrm{tr};\bm c,\bm o}(d\bm z)
    &:= \mathbb{P}\big(\bm z_{t,0}\in d\bm z\mid
    \bm c_t^{\mathrm{static}}=\bm c,\,\bm o_t=\bm o\big) \notag\\
    &= \mathbb{E}\Big[
        \mathcal{N}\big(\bar{\bm x}_t^f,
        \bm C_{\mathrm{GC}}(r_{\mathrm{loc}})\circ\bm P_t^f\big)
        \,\Big|\,\bm c_t^{\mathrm{static}}=\bm c,\,\bm o_t=\bm o\Big].
    \label{eq:theory-train-source}
\end{align} 
The inference source \cref{eq:ordered-deploy-source}
sets $\bm z_{t,0}^{(i)}=\bm x_t^{f,(i)}$, 
so it is exactly the baseline empirical forecast measure,
\begin{align}
    q_{t,0}^{\mathrm{inf};N}
    := \frac{1}{N}\sum_{i=1}^N \delta_{\bm x_t^{f,(i)}}
    = \widehat{\pi}_t^{B,f} .
    \label{eq:theory-inference-source}
\end{align}
Throughout, a superscript $\mathcal{C}_t^N$ in place of $[\bm c,\bm o]$ denotes
the same object evaluated at the random record, i.e.\
$q_{t,0}^{\mathrm{tr};\mathcal{C}_t^N}:=q_{t,0}^{\mathrm{tr};\bm
c_t^{\mathrm{static}},\bm o_t}$, and likewise for the flow and Euler maps
$\widehat\Phi^{\mathcal{C}_t^N}_{0,1}$ and $\widehat\Psi^{\mathcal{C}_t^N}_K$
defined below.

Let $(\bm Z,\bm X^\star)$ have the joint
law of $(\bm z_{t,0},\bm x_t^\star)$ given $\mathcal{C}_t^N=[\bm c,\bm o]$, 
then $\bm Z\sim q_{t,0}^{\mathrm{tr};\bm c,\bm o}$ and $\bm X^\star\sim p_t^{\bm
c,\bm o}$ where $p_t^{\bm c,\bm o}(d\bm
x):=\mathbb{P}(\bm x_t^\star\in d\bm x\mid \bm c_t^{\mathrm{static}}=\bm
c,\bm o_t=\bm o)$. Following \cref{eq:flowda-interpolant}, we define
\begin{align}
    \bm Z_\tau=(1-\tau)\bm Z+\tau\bm X^\star,
    \qquad
    \bm U=\bm X^\star-\bm Z,
    \label{eq:theory-interpolant-pair}
\end{align}
with \emph{ideal velocity}, the marginal velocity field of the conditional
flow-matching construction,
\begin{align}
    u_\tau^{\bm c,\bm o}(\bm z)
    :=\mathbb{E}\big[\bm U\mid\bm Z_\tau=\bm z,\,\bm c_t^{\mathrm{static}}=\bm c,\,\bm o_t=\bm o\big].
    \label{eq:theory-marginal-velocity}
\end{align}
The observation CNN maps $\bm{z}_{t,\tau}$, $\tau$, and 
$\bm{o}_t$ to the dynamic feature. 
We denote this deterministic map by $D_{\theta_{\mathrm{obs}}}$:
\begin{align}
    \bm{c}_{t,\tau}^{\mathrm{dyn}}
    = D_{\theta_{\mathrm{obs}}}(\bm{z}_{t,\tau},\tau;\bm o_t)
    := \operatorname{CNN}^{\mathrm{period}}_{\theta_{\mathrm{obs}}}
      \big([\bm H_t^\top\bm R^{-1/2}(\bm y_t-\bm H_t\bm z_{t,\tau}),
      \bm m_t,\bm\Phi(\tau)]\big).
    \label{eq:theory-exogenous-obs}
\end{align}
The dynamic feature depends on the observation record $\bm o_t$ 
as well as $\{\bm z_{t,\tau},\tau\}$, then 
\begin{align}
    \sigma\big(\bm c_{t,\tau}^{\mathrm{dyn}}\big)
    \subseteq \sigma\big(\bm z_{t,\tau},\tau,\bm o_t\big)
    \subseteq \sigma\big(\bm z_{t,\tau},\tau,\mathcal{C}_t^N\big).
    \label{eq:theory-dyn-measurability}
\end{align}
We define the \emph{induced velocity field}
\begin{align}
    \widetilde{\bm v}_{\theta,t}(\bm z,\tau;\bm c,\bm o)
    :=\bm v_{\theta_{\mathrm{field}}}
    \big(\bm z,\tau;[\bm c,D_{\theta_{\mathrm{obs}}}(\bm z,\tau;\bm o)]\big).
    \label{eq:theory-induced-velocity}
\end{align}
Because the dynamic channels are part of the velocity parameterization given
$\mathcal{C}_t^N$ rather than an additional conditioning variable, 
replacing original velocity field in \Cref{alg:flowda} 
with induced velocity field $\widetilde{\bm v}_{\theta,t}$ 
in \cref{eq:theory-induced-velocity} makes no changes. 

For a fixed record $[\bm c,\bm o]$ and pseudo-times
$0\le s\le\tau\le1$, let
$\widehat\Phi_{s,\tau}^{\bm c,\bm o}$ be the exact flow map of the induced
velocity field:
\begin{align}
    \bm z_\tau
    =\widehat\Phi_{s,\tau}^{\bm c,\bm o}(\bm z_s), 
    \quad
    \dot{\bm z}_\tau
    =\widetilde{\bm v}_{\hat\theta,t}
      (\bm z_\tau,\tau;\bm c,\bm o).
    \label{eq:theory-exact-flow-map}
\end{align}
Thus the subscripts $(s,\tau)$ denote the start and end pseudo-times.
On the grid $\tau_k=k/K$, let $\widehat\Psi_K^{\bm c,\bm o}$ be the
corresponding $K$ step Euler map:
\begin{align}
    & \bm z^{\mathrm E}_{\tau_0}
    =\bm z_0, 
    \quad
    \bm z^{\mathrm E}_{\tau_{k+1}}
    =\bm z^{\mathrm E}_{\tau_k}
      +K^{-1}\widetilde{\bm v}_{\hat\theta,t}
      (\bm z^{\mathrm E}_{\tau_k},\tau_k;\bm c,\bm o),
    \\
    & \widehat\Psi_K^{\bm c,\bm o}(\bm z_0)
    :=\bm z^{\mathrm E}_{\tau_K},
    \quad k=0,\ldots,K-1.
    \label{eq:theory-euler-map}
\end{align}
The superscript $[\bm c,\bm o]$ fixes the conditioning record, while the
superscript $\mathcal C_t^N$ evaluates the same map at the random record.
For a measurable map $T$ and probability law $\nu$, $T_\#\nu$ denotes the
pushforward law, defined by
\begin{align}
    (T_\#\nu)(A):=\nu\big(T^{-1}(A)\big)
    \qquad\text{for every measurable set }A.
    \label{eq:theory-pushforward}
\end{align}

Finally, we define the exact CFM target laws as follows
\begin{align}
    \pi_t^N(d\bm{x})
    &:= \mathbb{P}\big(\bm{x}_t^\star\in d\bm{x}\mid \mathcal{C}_t^N\big)
     = \mathbb{P}\big(\bm{x}_t^\star\in d\bm{x}\mid
        \bm c_t^{\mathrm{static}},\bm o_t\big), \\
    \pi_t^\infty(d\bm{x})
    &:= \mathbb{P}\big(\bm{x}_t^\star\in d\bm{x}\mid \mathcal{C}_t^\infty\big)
     = \mathbb{P}\big(\bm{x}_t^\star\in d\bm{x}\mid
        \bm c_t^{\mathrm{static},\infty},\bm o_t\big).
    \label{eq:theory-target-laws}
\end{align}

\subsection{Assumptions}
\label{app:theory-assumptions}

\subsubsection{Observation and baseline assumptions}
\label{app:theory-base-assumptions}

\begin{assumption}[Linear observation]
\label{asm:linear-obs}
$\bm y_t=\bm H_t\bm x_t+\bm\epsilon_t$ with $\bm H_t\in\mathbb{R}^{d_y\times
d_x}$ known and $\bm\epsilon_t\sim\mathcal{N}(\bm 0,\bm R)$, $\bm R\succ0$,
as in \cref{eq:linear-gaussian-analysis-model}. Moreover, 
$\bm R=\sigma_{\mathrm{obs}}^2\bm I_{d_y}$, 
and $\bm H_t\bm H_t^\top=\bm I_{d_y}$.
\end{assumption}

\begin{assumption}[Stationary assimilation cycles]
\label{asm:stationary}
The observation design is time-invariant, 
i.e., the observation operator $\bm H_t\equiv\bm H$, 
and observation noise covariance $\bm R$ is fixed.  
The baseline filter $B$ has reached a statistically steady 
and time-uniformly bounded regime: for some $C_B<\infty$,
$\mathbb E[\|\bm X_t^f\|_F^2+\|\bm X_t^a\|_F^2]\leq C_B$ for every $t$,
so that the tuples 
$\{\bm X_t^f,\bm X_t^a,\bm y_t,\bm x_t^\star\}$ for 
$t\in\mathcal{T}_{\mathrm{train}}\cup\mathcal{T}_{\mathrm{val}}
\cup\mathcal{T}_{\mathrm{test}}$ 
share a common marginal law $\mathcal{P}$.
The trained parameters $\hat\theta$ are shared across 
all assimilation cycles $t$.
\end{assumption}

\begin{assumption}[Baseline consistency]
\label{asm:baseline}
There exist rates $r_{f,N}, r_{a,N}\to0$ such that
\begin{align}
    \mathbb{E}\,W_2\big(\widehat{\pi}_t^{B,f},\pi_t^f\big) = O(r_{f,N}),
    \qquad
    \mathbb{E}\,W_2\big(\widehat{\pi}_t^{B,a},\pi_t^a\big) = O(r_{a,N}).
    \label{eq:theory-baseline-consistency}
\end{align}
\end{assumption}

By the triangle inequality,
\begin{align}
    \mathbb{E} W_2\big(\widehat{\pi}_t^{B,f},\pi_t^f\big)
    &\leq
    \underbrace{\mathbb{E} W_2\big(\widehat{\pi}_t^{B,f},
        \widehat{\pi}_t^{B,f,\infty}\big)}_{\text{(i) finite-ensemble approximation error}}
    \notag\\
    &\quad+
    \underbrace{\mathbb{E} W_2\big(\widehat{\pi}_t^{B,f,\infty},
        \pi_t^{B,f,\infty}\big)}_{\text{(ii) Monte Carlo error}}
    +\underbrace{W_2\big(\pi_t^{B,f,\infty},\pi_t^f\big)}_{\text{(iii) asymptotic approximation bias}},
    \label{eq:theory-baseline-decomposition}
\end{align}
and similarly for the analysis laws.
For each $t$, convergence theory of standard EnKF implies that  
the error between finite-ensemble members and their
infinite-ensemble counterparts 
can satisfy the $O(N^{-1/2})$ bound 
under suitable regularity conditions:
\begin{align}
    \max_{1\leq i\leq N}
    \left(\mathbb E\left\|\bm x_t^{f,(i)}-
    \bm x_t^{B,f,\infty,(i)}\right\|_2^2\right)^{1/2}
    \leq C_{f,t} N^{-1/2}.
    \label{eq:theory-limiting-member-rate}
\end{align}
This implies: 
\begin{align}
    \mathbb E\,W_2\big(\widehat\pi_t^{B,f},
    \widehat\pi_t^{B,f,\infty}\big)
    \leq
    \left\{\frac1N\sum_{i=1}^N
    \mathbb E\left\|\bm x_t^{f,(i)}-
    \bm x_t^{B,f,\infty,(i)}\right\|_2^2\right\}^{1/2}
    \leq C_{f,t}N^{-1/2},
    \label{eq:theory-limiting-measure-rate}
\end{align}
which is the $O(N^{-1/2})$ bound for term~(i).  
The constants $C_{f,t}$ depend on the assimilation time, 
we need additional stability assumptions \citep{mandel2011convergence} 
to obtain a time-uniform constant. 
The same bound applies to the analysis members. 
Term (ii) is dimension dependent.   
For a law $\mu$ on $\mathbb{R}^{d_x}$ with a finite moment of order $q>4$, 
and $\widehat\mu_N$ its empirical measure from $N$ i.i.d.\ draws,
we have $\mathbb{E}\,W_2(\widehat\mu_N,\mu)=O\big(r_{W_2}(N,d_x)\big)$ 
\citep{fournier2015rate} with  
\begin{align}
    r_{W_2}(N,d_x) =
    \begin{cases}
        N^{-1/4}, & d_x<4,\\
        N^{-1/4}\sqrt{\log(1+N)}, & d_x=4,\\
        N^{-1/d_x}, & d_x>4 .
    \end{cases}
    \label{eq:theory-ensemble-rate}
\end{align}
Term (ii) dominates term (i) for $d_x>4$, 
so the best convergence rate for term (i)+(ii) is 
$r_{f,N}=r_{W_2}(N,d_x)$.
Term (iii) does not depend on $N$ and vanishes only when the
infinite-ensemble baseline filter yields the 
exact Bayesian forecast law $\pi_t^f$ \citep{sarkka2023bayesian}. 
\Cref{asm:baseline} sets term (iii) to zero.

\subsubsection{Regularity assumptions}
\label{app:theory-regularity}

\begin{assumption}[Wasserstein conditioning stability]
\label{asm:lipschitz-kernel}
We define a kernel for conditional law as below: 
\begin{align}
    K_t(\bm c,\bm o)(d\bm x)
    &:= \mathbb{P}\big(\bm x_t^\star\in d\bm x
    \mid \bm c_t^{\mathrm{static}}=\bm c,\,\bm o_t=\bm o\big).
    \label{eq:theory-conditional-law-kernel}
\end{align}
Assume that this same kernel also satisfies
$\pi_t^\infty=K_t(\bm c_t^{\mathrm{static},\infty},\bm o_t)$ a.s. There exist
a norm $\|\cdot\|_c$, $L_{K,t}<\infty$, and $\beta_t\in(0,1]$, uniform in
$\bm o$, such that
\begin{align}
    W_2\big(K_t(\bm c,\bm o),K_t(\bm c',\bm o)\big)
    \le L_{K,t}\,\|\bm c-\bm c'\|_c^{\beta_t}
    \label{eq:theory-lipschitz-kernel}
\end{align}
for all $\bm c,\bm c'$ in the support of $\bm c_t^{\mathrm{static}}$ and
$\bm c_t^{\mathrm{static},\infty}$.
\end{assumption}

\begin{assumption}[Conditioning concentration]
\label{asm:feature-concentration}
Given $\|\cdot\|_c$ as above, there exist $\sigma_{c,t},D_{c,t}$ and a bias
$b_{c,N,t}$ such that 
\begin{align}
    \mathbb{E}\big\|\bm c_t^{\mathrm{static}}-\bm c_t^{\mathrm{static},\infty}\big\|_c^2
    \le \sigma_{c,t}^2\,\frac{D_{c,t}}{N}+b_{c,N,t}^2 .
    \label{eq:theory-dc-def}
\end{align}
\end{assumption}

$\sigma_{c,t}^2D_{c,t}/N$ is sampling fluctuation 
and $b_{c,N,t}$ is a systematic offset. 
By \cref{eq:theory-moment-map} the static channels are a smooth function
$\Gamma_t$ of $\bar{\bm x}_t^f,\bm\sigma_t^f,\bar{\bm x}_t^a,\bm
\sigma_t^a,\bm y_t$ alone, so $b_{c,N,t}\to0$ when the baseline filter
moments converge to their Bayesian counterparts $\bm\mu_t^\bullet,\bm
s_t^\bullet$ (\Cref{asm:baseline}). 
If the baseline filter is biased, then $b_{c,N,t}$ does not vanish and the
$N\to\infty$ limit of \cref{eq:theory-dc-def} is $b_{c,t}^2>0$.

\begin{assumption}[Ideal transport stability]
\label{asm:ideal-stability}
For every fixed observation record $\bm o$, the ideal velocity
$u_\tau^{\bm c,\bm o}$ of \cref{eq:theory-marginal-velocity} is
$L_{u,z,t}$-Lipschitz in its state argument and $L_{u,c,t}$-Lipschitz in
the static condition $\bm c$, uniformly on $\mathcal N_t$. Its ODE
transports $q_{t,0}^{\mathrm{tr};\bm c,\bm o}$ to
$K_t(\bm c,\bm o)$, and the conditional training sources satisfy
\begin{align}
    W_2\big(q_{t,0}^{\mathrm{tr};\bm c,\bm o},
        q_{t,0}^{\mathrm{tr};\bm c',\bm o}\big)
    \le L_{q,c,t}\|\bm c-\bm c'\|_c.
    \label{eq:theory-source-condition-stability}
\end{align}
\end{assumption}

\begin{assumption}[Transport regularity]
\label{asm:tube}
There is a set $\mathcal{N}_t\subset\mathbb{R}^{d_x}\times[0,1]$ containing
the support of the interpolant paths \cref{eq:flowda-interpolant}, the
exact trajectories of \cref{eq:flowda-analysis-ode}, and the $K$ Euler
iterates \cref{eq:ode-euler} initialized from the training source
$q_{t,0}^{\mathrm{tr};\mathcal{C}_t^N}$ and the inference source
$q_{t,0}^{\mathrm{inf};N}$. The induced velocity
$\widetilde{\bm v}_{\hat\theta,t}(\bm z,\tau;\bm c,\bm o)$ is $C^1$ in
$\tau$, $\widetilde{L}_{v,z,t}$-Lipschitz in $\bm z$, and
$\widetilde{L}_{v,c,t}$-Lipschitz in $\bm c$. 
That is, for all $\bm z,\bm z',\bm c,\bm c'$ and $\bm o$, we have
\begin{align}
    \big\|\widetilde{\bm v}_{\hat\theta,t}(\bm z,\tau;\bm c,\bm o)
    -\widetilde{\bm v}_{\hat\theta,t}(\bm z',\tau;\bm c',\bm o)\big\|_2
    \leq \widetilde{L}_{v,z,t}\|\bm z-\bm z'\|_2
    +\widetilde{L}_{v,c,t}\|\bm c-\bm c'\|_c,
    \label{eq:theory-transport-regularity}
\end{align}
\end{assumption}

\subsection{Supporting results}
\label{app:theory-lemmas}

\subsubsection{Induced velocity field}
\label{app:theory-induced}

\begin{lemma}[Induced velocity field]
\label{lem:induced}
Under \Cref{asm:linear-obs}, for 
$\theta=\{\theta_{\mathrm{obs}},\theta_{\mathrm{field}}\}$,
plugging in induced velocity
$\widetilde{\bm v}_{\theta,t}$ of \cref{eq:theory-induced-velocity} 
in \Cref{alg:flowda} yields  the following exact loss, 
Euler update, and excess loss identities. 
\begin{enumerate}
\item[(i)] the dynamic feature 
$\bm c_{t,\tau}^{\mathrm{dyn}}
=D_{\theta_{\mathrm{obs}}}(\bm z_{t,\tau},\tau;\bm o_t)$ of
\cref{eq:theory-exogenous-obs}. Hence
\cref{eq:theory-dyn-measurability} holds, and
$\bm c_{t,\tau}^{\mathrm{dyn}}$ is a deterministic function of
$(\bm z_{t,\tau},\tau)$ given fixed $\bm o_t$;
\item[(ii)] the CFM loss \cref{eq:cfm-loss} is exactly
\begin{align*}
    \mathcal{L}_{\mathrm{CFM}}(\widetilde{\bm v}_{\theta,t})
    &= \mathbb{E}\big\|
      \widetilde{\bm v}_{\theta,t}(\bm Z_\tau,\tau;
      \bm c_t^{\mathrm{static}},\bm o_t)-\bm U
      \big\|_2^2,
\end{align*}
with $(\bm Z_\tau,\bm U)$ as in \cref{eq:theory-interpolant-pair};
\item[(iii)] with $\tau_k=k/K$, the inference step in 
\cref{eq:ode-euler} d is exactly
\begin{align*}
    \bm z_{t,\tau_{k+1}}^{(i)}
    &= \bm z_{t,\tau_k}^{(i)}
      + \frac{1}{K}\widetilde{\bm v}_{\hat\theta,t}
      (\bm z_{t,\tau_k}^{(i)},\tau_k;
      \bm c_t^{\mathrm{static}},\bm o_t),
      \quad k=0,\ldots,K-1,
\end{align*}
the explicit Euler scheme for the induced velocity field;
\item[(iv)] for integrable
$\mathcal{F}_\tau^{\mathrm{rec}}$-measurable velocity fields, with
$\mathcal{F}_\tau^{\mathrm{rec}}$ as in \cref{eq:theory-filtrations}, 
the minimizer of \cref{eq:cfm-loss} is the ideal velocity of
\cref{eq:theory-marginal-velocity},
\begin{align*}
    u_\tau^{\bm c,\bm o}(\bm z)
    &:= \mathbb{E}\big[\bm U\mid
        \bm Z_\tau=\bm z,\,
        \bm c_t^{\mathrm{static}}=\bm c,\,
        \bm o_t=\bm o\big].
\end{align*}
For every $\theta$, the excess CFM loss of $\widetilde{\bm v}_{\theta,t}$
over the interpolate velocity $u$ equals its 
mean squared velocity error integrated over pseduo-time:
\begin{align}
    \mathcal{L}_{\mathrm{CFM}}(\widetilde{\bm v}_{\theta,t})
    - \mathcal{L}_{\mathrm{CFM}}(u)
    = \int_0^1 \mathbb{E}\big\|\widetilde{\bm v}_{\theta,t}(\bm{Z}_\tau,\tau;\bm c_t^{\mathrm{static}},\bm o_t)
      - u_\tau^{\bm c_t^{\mathrm{static}},\bm o_t}(\bm{Z}_\tau)\big\|_2^2\,d\tau .
    \label{eq:theory-loss-equivalence}
\end{align}
\end{enumerate}
\end{lemma}
\begin{proof}
(i) By \Cref{asm:linear-obs}, $\bm r_{t,\tau}=\bm H_t^\top\bm
R^{-1/2}(\bm y_t-\bm H_t\bm z)$ is jointly measurable in $(\bm z,\bm o_t)$
and affine in $\bm z$ for fixed $\bm o_t$, $\bm m_t$ is a function of $\bm
H_t$, and $\bm\Phi(\tau)$ a function of $\tau$, so the observation CNN output
\cref{eq:implementation-observation-cnn} is a measurable function of $(\bm
z,\tau,\bm o_t)$, which gives the first inclusion in
\cref{eq:theory-dyn-measurability}. 
The second inclusion in \cref{eq:theory-dyn-measurability} 
holds because $\bm o_t$ is a component of $\mathcal{C}_t^N$.

(ii) By \cref{eq:theory-exogenous-obs},
$\bm c_{t,\tau}=[\bm c_t^{\mathrm{static}},
D_{\theta_{\mathrm{obs}}}(\bm Z_\tau,\tau;\bm o_t)]$. Hence, the network in \cref{eq:cfm-loss} satisfies
\[
\bm v_\theta(\bm Z_\tau,\tau;\bm c_{t,\tau})
=\widetilde{\bm v}_{\theta,t}(\bm Z_\tau,\tau;
  \bm c_t^{\mathrm{static}},\bm o_t)
\]
by \cref{eq:theory-induced-velocity}. Substitution into
\cref{eq:cfm-loss}, with $\bm U=\bm x_t^\star-\bm z_{t,0}$ from
\cref{eq:theory-interpolant-pair}, gives (ii).

(iii) At Euler step $k$, \Cref{alg:flowda} forms
$\bm c_{t,\tau_k}^{\mathrm{dyn},(i)}$ from $\bm z_{t,\tau_k}^{(i)}$ and
evaluates $\bm v_{t,\tau_k}^{(i)}$. By (i),
\[
\bm v_{t,\tau_k}^{(i)}
=\widetilde{\bm v}_{\hat\theta,t}(\bm z_{t,\tau_k}^{(i)},\tau_k;
  \bm c_t^{\mathrm{static}},\bm o_t).
\]
Thus its update
$\bm z_{t,\tau_{k+1}}^{(i)}=\bm z_{t,\tau_k}^{(i)}
+K^{-1}\bm v_{t,\tau_k}^{(i)}$ is exactly the Euler update in
\cref{eq:ode-euler}.

(iv) By (ii), the CFM loss is the $L^2$ regression loss for $\bm U$ given
$(\bm Z_\tau,\tau,\mathcal{C}_t^N)$. By (i) and
\cref{eq:theory-dyn-measurability}, the dynamic feature
$\bm c_{t,\tau}^{\mathrm{dyn}}$ is already determined by these regressors:
\[
\sigma(\bm Z_\tau,\tau,\mathcal{C}_t^N,
\bm c_{t,\tau}^{\mathrm{dyn}})
=\sigma(\bm Z_\tau,\tau,\mathcal{C}_t^N).
\]
Thus
\[
u:=\mathbb{E}[\bm U\mid\bm Z_\tau,\tau,\mathcal{C}_t^N]
=u_\tau^{\bm c_t^{\mathrm{static}},\bm o_t}(\bm Z_\tau),
\]
where the second equality is \cref{eq:theory-marginal-velocity}. For every 
integrable admissible field $\bm w$, conditional expectation
orthogonality gives
\[
\mathbb{E}\|\bm w-\bm U\|_2^2
=\mathbb{E}\|\bm w-u\|_2^2+\mathbb{E}\|u-\bm U\|_2^2.
\]
Taking $\bm w=\widetilde{\bm v}_{\theta,t}$ and using
$\tau\sim\operatorname{Uniform}(0,1)$ gives
\cref{eq:theory-loss-equivalence}.
\end{proof}

\subsubsection{Static and dynamic conditioning}
\label{app:theory-dynamic}

The induced velocity field
$\widetilde{\bm v}_{\theta,t}(\bm z_{t,\tau},\tau;\bm c_t^{\mathrm{static}},\bm o_t)$
of \cref{eq:theory-induced-velocity}
is $\mathcal{F}_\tau^{\mathrm{dyn}}$-measurable, 
whereas the ideal velocity
$u_\tau^{\bm c_t^{\mathrm{static}},\bm o_t}(\bm z_{t,\tau})$ of
\cref{eq:theory-marginal-velocity} is the $L^2$ projection onto the larger $\mathcal{F}_\tau^{\mathrm{rec}}$ of
\cref{eq:theory-filtrations} by \Cref{lem:induced} (iv). 
The following lemma shows that, under \Cref{asm:linear-obs}, 
if dynamic feature $\bm c_{t,\tau}^{\mathrm{dyn}}$ of
\cref{eq:theory-exogenous-obs} determines the residual $\bm r_{t,\tau}$ of
\cref{eq:obs-residual-field}, we have
$\mathcal{F}_\tau^{\mathrm{dyn}}=\mathcal{F}_\tau^{\mathrm{rec}}$, 
so the dynamic feature loses no information 
relative to the record $\mathcal{C}_t^N$ of
\cref{eq:theory-conditioning-record}.

\begin{lemma}[Sufficiency of the dynamic channel]
\label{lem:dyn-sufficiency}
Let \Cref{asm:linear-obs} hold, then for each $t,\tau$:
\begin{enumerate}
\item[(i)] the residual $\bm r_{t,\tau}$ of \cref{eq:obs-residual-field} is
the whitened score of the observation likelihood at the transported state
$\bm z_{t,\tau}$,
\begin{align}
    \bm r_{t,\tau}
    = \sigma_{\mathrm{obs}}\,
      \nabla_{\bm z}\log p(\bm y_t\mid\bm z)\big|_{\bm z=\bm z_{t,\tau}},
    \qquad
    p(\bm y\mid\bm z)=\mathcal{N}(\bm y;\bm H_t\bm z,\bm R);
    \label{eq:theory-score-identity}
\end{align}
\item[(ii)] the observation $\bm y_t$ is recovered from the residual
$\bm r_{t,\tau}$ and the transported state $\bm z_{t,\tau}$,
\begin{align}
    \bm y_t
    = \sigma_{\mathrm{obs}}\bm H_t\bm r_{t,\tau}+\bm H_t\bm z_{t,\tau},
    \label{eq:theory-residual-inverse}
\end{align}
so $\sigma(\bm z_{t,\tau},\tau,\bm c_t^{\mathrm{static}},\bm
r_{t,\tau})=\mathcal{F}_\tau^{\mathrm{rec}}$, and the network input
$\bm h_{t,\tau}^{\mathrm{obs}}=[\bm r_{t,\tau},\bm m_t,\bm\Phi(\tau)]$ of
\cref{eq:implementation-observation-input} already carries all information of $\bm o_t$;
\item[(iii)] if $D_{\theta_{\mathrm{obs}}}$ of
\cref{eq:theory-exogenous-obs} is \emph{faithful} i.e., there is a
measurable $\psi$ such that 
$\psi(D_{\theta_{\mathrm{obs}}}(\bm z,\tau;\bm o),\bm z,\tau)=\bm
H_t^\top\bm R^{-1/2}(\bm y-\bm H_t\bm z)$ for all $(\bm z,\tau,\bm o)$,
where $\bm o=[\bm y,\bm H_t,\bm R]$ as in
\cref{eq:theory-exogenous-design}, then
$\mathcal{F}_\tau^{\mathrm{dyn}}=\mathcal{F}_\tau^{\mathrm{rec}}$ and the
ideal velocity $u_\tau^{\bm c,\bm o}$ of
\cref{eq:theory-marginal-velocity} is representable in the 
form \cref{eq:theory-induced-velocity}.
So there exists a measurable $G$ such that 
\begin{align}
    u_\tau^{\bm c,\bm o}(\bm z)
    = G\big(\bm z,\tau;[\bm c,D_{\theta_{\mathrm{obs}}}(\bm z,\tau;\bm o)]\big)
    \qquad\text{a.e. }(\bm z,\tau,\bm c,\bm o).
    \label{eq:theory-representable}
\end{align}
\end{enumerate}
\end{lemma}
\begin{proof}
(i) $\log p(\bm y\mid\bm z)=-\tfrac12(\bm y-\bm H_t\bm z)^\top\bm
R^{-1}(\bm y-\bm H_t\bm z)+\mathrm{const}$ gives $\nabla_{\bm z}\log
p(\bm y\mid\bm z)=\sigma_{\mathrm{obs}}^{-2}\bm H_t^\top(\bm y-\bm
H_t\bm z)$, while \cref{eq:obs-residual-field} gives $\bm
r_{t,\tau}=\sigma_{\mathrm{obs}}^{-1}\bm H_t^\top(\bm y_t-\bm H_t\bm
z_{t,\tau})$.

(ii) Multiplying $\bm r_{t,\tau}$ by $\bm H_t$ and using $\bm H_t\bm
H_t^\top=\bm I_{d_y}$ gives $\bm H_t\bm
r_{t,\tau}=\sigma_{\mathrm{obs}}^{-1}(\bm y_t-\bm H_t\bm z_{t,\tau})$,
which is \cref{eq:theory-residual-inverse}, so $\bm y_t$ is
$\sigma(\bm z_{t,\tau},\bm r_{t,\tau})$-measurable. Since $\bm H_t$ and
$\bm R$ are deterministic (\Cref{asm:linear-obs}), the whole design
$\bm o_t=[\bm y_t,\bm H_t,\bm R]$ of \cref{eq:theory-exogenous-design} is
then $\sigma(\bm z_{t,\tau},\bm r_{t,\tau})$-measurable, so
$\mathcal{C}_t^N=[\bm c_t^{\mathrm{static}},\bm o_t]$ is
$\sigma(\bm z_{t,\tau},\bm c_t^{\mathrm{static}},\bm r_{t,\tau})$-measurable,
giving
\begin{align*}
    \mathcal{F}_\tau^{\mathrm{rec}}
    =\sigma\big(\bm z_{t,\tau},\tau,\mathcal{C}_t^N\big)
    \subseteq
    \sigma\big(\bm z_{t,\tau},\tau,\bm c_t^{\mathrm{static}},\bm r_{t,\tau}\big).
\end{align*}
Conversely, $\bm r_{t,\tau}$ is a function of
$(\bm z_{t,\tau},\bm o_t)$ by \cref{eq:obs-residual-field}, so
\begin{align*}
    \sigma\big(\bm z_{t,\tau},\tau,\bm c_t^{\mathrm{static}},\bm r_{t,\tau}\big)
    \subseteq
    \sigma\big(\bm z_{t,\tau},\tau,\mathcal{C}_t^N\big)
    =\mathcal{F}_\tau^{\mathrm{rec}}.
\end{align*}

(iii) Evaluating the faithfulness identity at
$(\bm z,\tau,\bm o)=(\bm z_{t,\tau},\tau,\bm o_t)$ and combining with 
$\bm c_{t,\tau}^{\mathrm{dyn}}=D_{\theta_{\mathrm{obs}}}(\bm z_{t,\tau},\tau;\bm o_t)$
of \cref{eq:theory-exogenous-obs} gives
\begin{align*}
    \bm r_{t,\tau}
    =\psi\big(\bm c_{t,\tau}^{\mathrm{dyn}},\bm z_{t,\tau},\tau\big),
\end{align*}
so $\bm r_{t,\tau}$ is $\mathcal{F}_\tau^{\mathrm{dyn}}$-measurable and, by
the equality proved in (ii),
\begin{align*}
    \mathcal{F}_\tau^{\mathrm{rec}}
    =\sigma\big(\bm z_{t,\tau},\tau,\bm c_t^{\mathrm{static}},\bm r_{t,\tau}\big)
    \subseteq\mathcal{F}_\tau^{\mathrm{dyn}} .
\end{align*}
The reverse inclusion
$\mathcal{F}_\tau^{\mathrm{dyn}}\subseteq\mathcal{F}_\tau^{\mathrm{rec}}$ is
\cref{eq:theory-filtrations}. The
ideal velocity $u_\tau^{\bm c_t^{\mathrm{static}},\bm o_t}(\bm z_{t,\tau})$
is $\mathcal{F}_\tau^{\mathrm{rec}}$-measurable by \Cref{lem:induced} (iv),
hence it's $\mathcal{F}_\tau^{\mathrm{dyn}}$-measurable. We apply 
\cref{eq:theory-doob-dynkin} to
$Z=(\bm z_{t,\tau},\tau,\bm c_t^{\mathrm{static}},\bm c_{t,\tau}^{\mathrm{dyn}})$
and get a measurable $G$ satisfying \cref{eq:theory-representable}.
\end{proof}

\begin{proposition}[Approximation error of the induced velocity]
\label{prop:mediation}
Let
\begin{align}
    \mathcal{V}_{\mathrm{dyn}}
    :=\Big\{
      (\bm z,\tau,\bm c,\bm o)\mapsto
      F\big(\bm z,\tau;[\bm c,D(\bm z,\tau;\bm o)]\big)
      ;
      F,D\ \text{measurable and integrable}
    \Big\}
    \label{eq:theory-mediated-class}
\end{align}
be the class of all induced velocity fields in 
\cref{eq:theory-induced-velocity}, i.e.\ the fields obtained by composing an
observation encoder $D$ with a velocity network $F$, it contains
$\widetilde{\bm v}_{\theta,t}$ for every $\theta$. 
For a velocity field $\bm w$,  we denote its mean squared deviation from the ideal
velocity as
\begin{align}
    \mathcal{J}(\bm w)
    := \int_0^1\mathbb{E}\Big\|
       \bm w(\bm Z_\tau,\tau,\bm c_t^{\mathrm{static}},\bm o_t)
       -u_\tau^{\bm c_t^{\mathrm{static}},\bm o_t}(\bm Z_\tau)
       \Big\|_2^2\,d\tau ,
    \label{eq:theory-approx-objective}
\end{align}
and define the approximation error of the class as its smallest value over
the class,
\begin{align}
    \varepsilon_{\mathrm{approx}}
    := \inf_{\bm w\in\mathcal{V}_{\mathrm{dyn}}}\mathcal{J}(\bm w),
    \label{eq:theory-approx-error}
\end{align}
and let $\varepsilon_{\mathrm{est}}
:=\varepsilon_{\mathrm{vel}}-\varepsilon_{\mathrm{approx}}$, where
\begin{align}
    \varepsilon_{\mathrm{vel}}
    := \mathcal{L}_{\mathrm{CFM}}\big(\widetilde{\bm v}_{\hat\theta,t}\big)
     - \mathcal{L}_{\mathrm{CFM}}(u)
    = \mathcal{J}\big(\widetilde{\bm v}_{\hat\theta,t}\big)
    \label{eq:theory-vel-loss}
\end{align}
is the population excess CFM loss \cref{eq:cfm-loss} of the trained induced
field $\widetilde{\bm v}_{\hat\theta,t}$ over the ideal velocity
$u_\tau^{\bm c,\bm o}$ of \cref{eq:theory-marginal-velocity}, the second
equality is \cref{eq:theory-loss-equivalence}. Then:
\begin{enumerate}
\item[(i)] the decomposition
\begin{align}
    \varepsilon_{\mathrm{vel}}
    = \underbrace{\varepsilon_{\mathrm{approx}}}_{\text{approximation}}
    + \underbrace{\varepsilon_{\mathrm{est}}}_{\text{estimation and optimization}}
    \label{eq:theory-excess-split}
\end{align}
has both terms nonnegative: $\varepsilon_{\mathrm{approx}}$ is the smallest
error achievable by any induced velocity field, and
$\varepsilon_{\mathrm{est}}$ is what a finite network, a finite training set,
and the optimizer add on top of it;
\item[(ii)] under \Cref{asm:linear-obs},
$\varepsilon_{\mathrm{approx}}=0$.
\end{enumerate}
Consequently, under \Cref{asm:linear-obs} the induced parameterization
represents the ideal velocity exactly, so
$\varepsilon_{\mathrm{vel}}=\varepsilon_{\mathrm{est}}$ and the
flow-matching term
$\varepsilon_{\mathrm{fm}}=C_{\mathrm{FM}}\sqrt{\varepsilon_{\mathrm{vel}}}$
of \Cref{thm:flowef} is entirely estimation and optimization error:
passing the observation record $\bm o_t$ to the velocity network via 
learned dynamic feature $\bm c_{t,\tau}^{\mathrm{dyn}}$ contributes
nothing to $\varepsilon_{\mathrm{vel}}$.
\end{proposition}
\begin{proof}
(i) $\mathcal{J}$ of \cref{eq:theory-approx-objective} integrates a squared
norm, so $\mathcal{J}(\bm w)\ge0$ for every
$\bm w\in\mathcal{V}_{\mathrm{dyn}}$ and hence
$\varepsilon_{\mathrm{approx}}=\inf_{\bm w\in\mathcal{V}_{\mathrm{dyn}}}
\mathcal{J}(\bm w)\ge0$. For $\varepsilon_{\mathrm{est}}\ge0$, the trained
field satisfies $\widetilde{\bm v}_{\hat\theta,t}\in\mathcal{V}_{\mathrm{dyn}}$
and $\varepsilon_{\mathrm{vel}}=\mathcal{J}(\widetilde{\bm v}_{\hat\theta,t})$
by \cref{eq:theory-vel-loss}, so
\begin{align*}
    \varepsilon_{\mathrm{vel}}
    = \mathcal{J}\big(\widetilde{\bm v}_{\hat\theta,t}\big)
    \ge \inf_{\bm w\in\mathcal{V}_{\mathrm{dyn}}}\mathcal{J}(\bm w)
    = \varepsilon_{\mathrm{approx}} .
\end{align*}

(ii) By (i), it suffices to construct one $\bm w\in\mathcal{V}_{\mathrm{dyn}}$
with $\mathcal{J}(\bm w)=0$. Take the encoder
\begin{align*}
    D(\bm z,\tau;\bm o)
    := \big[\bm H_t^\top\bm R^{-1/2}(\bm y-\bm H_t\bm z),\,\bm 0\big],
\end{align*}
i.e.\ $D^{(1)}(\bm z_{t,\tau},\tau;\bm o_t)=\bm r_{t,\tau}\in\mathbb{R}^{d_x}$
of \cref{eq:obs-residual-field} and $D^{(j)}=\bm 0$ for
$j=2,\dots,d_{\mathrm{dyn}}$, so that
$D(\bm z,\tau;\bm o)\in\mathbb{R}^{d_x\times d_{\mathrm{dyn}}}$ has the shape
of $\bm c_{t,\tau}^{\mathrm{dyn}}$ in \cref{eq:theory-exogenous-obs}.
It is measurable and integrable, and, with
$\psi(\bm d,\bm z,\tau):=\bm d^{(1)}$ the projection onto the first channel,
\begin{align*}
    \psi\big(D(\bm z,\tau;\bm o),\bm z,\tau\big)
    = \bm H_t^\top\bm R^{-1/2}(\bm y-\bm H_t\bm z).
\end{align*}
Therefore, $D$ is faithful in the sense of \Cref{lem:dyn-sufficiency} (iii). 
By \Cref{lem:dyn-sufficiency} (iii), we have a measurable $G$ such that
\begin{align*}
    u_\tau^{\bm c,\bm o}(\bm z)
    = G\big(\bm z,\tau;[\bm c,D(\bm z,\tau;\bm o)]\big)
    \qquad\text{a.e. }(\bm z,\tau,\bm c,\bm o),
\end{align*}
and $G$ is integrable because
$u_\tau^{\bm c,\bm o}=\mathbb{E}[\bm U\mid\bm Z_\tau,\bm c_t^{\mathrm{static}},\bm o_t]$
with $\bm U\in L^2$. Hence $F:=G$ gives
$\bm w:=G(\cdot,\cdot;[\cdot,D(\cdot,\cdot;\cdot)])\in\mathcal{V}_{\mathrm{dyn}}$
with $\bm w=u$ a.e., so $\mathcal{J}(\bm w)=0$ and
$\varepsilon_{\mathrm{approx}}=0$.
\end{proof}

\subsubsection{Transport endpoint}
\label{app:theory-endpoint-errors}

\begin{lemma}[CFM endpoint]
\label{lem:optimal}
For $\mathcal{C}_t^N=[\bm c,\bm o]$, $(\bm Z_\tau,\bm U)$ and
$u_\tau^{\bm c,\bm o}$ as in
\cref{eq:theory-interpolant-pair,eq:theory-marginal-velocity}. If
\begin{enumerate}
\item[(i)] the induced velocity field achieves zero excess loss, i.e., 
$\varepsilon_{\mathrm{vel}}=0$ in \cref{eq:theory-vel-loss};
\item[(ii)] training and inference share the same source distribution, i.e., 
$q_{t,0}^{\mathrm{tr};\bm c,\bm o} = q_{t,0}^{\mathrm{inf}}$; 
\item[(iii)] the ODE \cref{eq:flowda-analysis-ode} driven by
$\widetilde{\bm v}_{\hat\theta,t}$ is solved exactly,
\end{enumerate}
then its endpoint law is exact, i.e., 
$\mathbb{P}\big(\bm z_{t,1}\in d\bm x\mid\mathcal{C}_t^N\big)
=\pi_t^N(d\bm x)$ of \cref{eq:theory-target-laws}.
\end{lemma}
\begin{proof}
Let $\rho_\tau^{\bm c,\bm o}(d\bm z)
:=\mathbb{P}\big(\bm Z_\tau\in d\bm z\mid\mathcal{C}_t^N=[\bm c,\bm o]\big)$,
so that
$\rho_0^{\bm c,\bm o}=q_{t,0}^{\mathrm{tr};\bm c,\bm o}$ and
$\rho_1^{\bm c,\bm o}=p_t^{\bm c,\bm o}$. Each interpolant path in 
\cref{eq:theory-interpolant-pair} is a straight line ending at $\bm U=\bm X^\star-\bm Z$. 
Write $\gamma^{\bm c,\bm o}$ for the joint law of $(\bm Z,\bm X^\star)$ given
$\mathcal{C}_t^N=[\bm c,\bm o]$ and
$\rho_\tau^{\bm z,\bm x}:=\delta_{(1-\tau)\bm z+\tau\bm x}$ for the path with
endpoints $(\bm z,\bm x)$, then 
$\rho_\tau^{\bm c,\bm o}=\int\rho_\tau^{\bm z,\bm x}
\gamma^{\bm c,\bm o}(d\bm z,d\bm x)$ and
$\partial_\tau\rho_\tau^{\bm z,\bm x}
+\nabla\cdot\big(\rho_\tau^{\bm z,\bm x}(\bm x-\bm z)\big)=0$. Integrating
the latter against $\gamma^{\bm c,\bm o}$ and combining with
\cref{eq:theory-marginal-velocity}, we get 
\begin{align*}
    \partial_\tau\rho_\tau^{\bm c,\bm o}
    &= \int\partial_\tau\rho_\tau^{\bm z,\bm x}\,
       \gamma^{\bm c,\bm o}(d\bm z,d\bm x)
     = -\nabla\cdot\int\rho_\tau^{\bm z,\bm x}(\bm x-\bm z)\,
       \gamma^{\bm c,\bm o}(d\bm z,d\bm x)
     = -\nabla\cdot\big(\rho_\tau^{\bm c,\bm o}u_\tau^{\bm c,\bm o}\big),
\end{align*}
which is the continuity equation of conditional flow matching
\citep[Thm.~1]{lipman2023flow} extended to general couplings by
\citet{tong2024improving}. Hence, the flow of $u^{\bm c,\bm o}$ transports
$q_{t,0}^{\mathrm{tr};\bm c,\bm o}$ to $p_t^{\bm c,\bm o}$. By (i) and
\Cref{lem:induced} (iv), the learned and ideal velocities agree along the
interpolant distribution:
\begin{align*}
    \widetilde{\bm v}_{\hat\theta,t}(\cdot,\tau;\bm c,\bm o)
    = u_\tau^{\bm c,\bm o}
    \qquad \text{along }\rho_\tau^{\bm c,\bm o},\quad \tau\in[0,1].
\end{align*}
Thus \cref{eq:flowda-analysis-ode} has the same flow. By (ii) and (iii), 
\begin{align*}
    \mathbb{P} \left(
        \bm z_{t,1}\in\cdot
        \mid\mathcal C_t^N
    \right)
    = p_t^{\bm c_t^{\mathrm{static}},\bm o_t}
    = \mathbb{P} \left(
        \bm x_t^\star\in\cdot\mid\mathcal C_t^N
       \right)
     = \pi_t^N,
\end{align*}
where the final equality is from \cref{eq:theory-target-laws}.
\end{proof}

\subsubsection[]{Conditioning bias ($\varepsilon_{\mathrm{bias}}$)}
\label{app:theory-error-bias}

\begin{lemma}[Conditioning bias]
\label{lem:bias}
Define $B_{c,t} := \{\mathbb{E}\, W_2^2(\pi_t^\infty,\pi_t^a)\}^{1/2}$. If
every conditional law $\pi_t^\infty(\cdot\mid \mathcal{C}_t^\infty=[\bm c,\bm o])$
satisfies the uniform Talagrand $T_2(\kappa_t)$ transport--entropy inequality
\citep{talagrand1996transportation},
\begin{align}
    W_2^2\big(\nu, \pi_t^\infty(\cdot\mid \mathcal{C}_t^\infty=[\bm c,\bm o])\big)
    \le \frac{2}{\kappa_t}\,
    \mathrm{KL}\big(\nu \,\|\, \pi_t^\infty(\cdot\mid \mathcal{C}_t^\infty=[\bm c,\bm o])\big),
    \label{eq:theory-t2}
\end{align}
then 
\begin{align}
    \mathbb{E}\, W_2^2(\pi_t^\infty,\pi_t^a)
    = 
    B_{c,t}^2 \le \frac{2}{\kappa_t} \mathcal{I}_{\mathrm{repr},t}, 
    \quad
    \mathcal{I}_{\mathrm{repr},t}:=\mathbb{E}\,\mathrm{KL}(\pi_t^a\,\|\,\pi_t^\infty).
\end{align}
In particular $B_{c,t}=0$ whenever $\pi_t^\infty=\pi_t^a$ a.s.
\end{lemma}
\begin{proof}
Apply $\nu=\pi_t^a$ to \cref{eq:theory-t2} and average over $\mathcal{C}_t^\infty$
and $\pi_t^a$.
\end{proof}

\subsubsection[]{Finite-ensemble features ($\varepsilon_{\mathrm{feat}}$)}
\label{app:theory-error-features}

\begin{lemma}[Conditional stability]
\label{lem:lipschitz}
Under \Cref{asm:ideal-stability}, the endpoint conditional laws follow
\begin{align}
    W_2\big(K_t(\bm c,\bm o),K_t(\bm c',\bm o)\big)
    \le L_{K,t}\|\bm c-\bm c'\|_c, 
    \quad
    L_{K,t}
    := e^{L_{u,z,t}}L_{q,c,t}
    + L_{u,c,t}\frac{e^{L_{u,z,t}}-1}{L_{u,z,t}}.
    \label{eq:theory-lipschitz-const}
\end{align}
Then \Cref{asm:lipschitz-kernel} holds with $\beta_t=1$. 
When $L_{u,z,t}=0$, we interpret $(e^{L_{u,z,t}}-1)/L_{u,z,t}$ 
as its continuous limit $1$.
\end{lemma}
\begin{proof}
Choose random initial states $(\bm Z_0,\bm Z_0')$ with laws
$q_{t,0}^{\mathrm{tr};\bm c,\bm o}$ and
$q_{t,0}^{\mathrm{tr};\bm c',\bm o}$ respectively such that
\begin{align*}
    \big\{\mathbb E\|\bm Z_0-\bm Z_0'\|_2^2\big\}^{1/2}
    = W_2\big(q_{t,0}^{\mathrm{tr};\bm c,\bm o},
        q_{t,0}^{\mathrm{tr};\bm c',\bm o}\big).
\end{align*}
Let $\bm Z_\tau$ and $\bm Z_\tau'$ solve the ODEs exactly with static
conditions $\bm c$ and $\bm c'$ respectively, and set
$\Delta_\tau:=\bm Z_\tau-\bm Z_\tau'$. The Lipschitz bounds in
\Cref{asm:ideal-stability} give
\begin{align*}
    \frac{d}{d\tau}\|\Delta_\tau\|_2
    \le L_{u,z,t}\|\Delta_\tau\|_2
      +L_{u,c,t}\|\bm c-\bm c'\|_c.
\end{align*}
By Gr\"onwall's inequality, at $\tau=1$,
\begin{align*}
    \|\Delta_1\|_2
    \le e^{L_{u,z,t}}\|\Delta_0\|_2
      +L_{u,c,t}\frac{e^{L_{u,z,t}}-1}{L_{u,z,t}}
        \|\bm c-\bm c'\|_c.
\end{align*}
Taking $L^2$ norms, using the optimality of the initial coupling, 
and the source bound in \Cref{asm:ideal-stability} yields
\begin{align*}
    \big\{\mathbb E\|\Delta_1\|_2^2\big\}^{1/2}
    &\le e^{L_{u,z,t}}
       W_2\big(q_{t,0}^{\mathrm{tr};\bm c,\bm o},
                q_{t,0}^{\mathrm{tr};\bm c',\bm o}\big)
      +L_{u,c,t}\frac{e^{L_{u,z,t}}-1}{L_{u,z,t}}
       \|\bm c-\bm c'\|_c \\
    &\le L_{K,t}\|\bm c-\bm c'\|_c.
\end{align*}
By \Cref{asm:ideal-stability}, the two endpoint laws are
$K_t(\bm c,\bm o)$ and $K_t(\bm c',\bm o)$. Therefore the definition of
$W_2$ gives the stated bound.
\end{proof}

\subsubsection[]{Flow matching and Euler integration ($\varepsilon_{\mathrm{fm}}$, $\varepsilon_{\mathrm{ode}}$)}
\label{app:theory-error-numerical}

\begin{lemma}[Population conditional error]
\label{lem:optimal-rate}
Under \Cref{asm:lipschitz-kernel,asm:feature-concentration} with
$\beta_t=1$,
\begin{align}
    \big\{\mathbb{E}\, W_2^2(\pi_t^N,\pi_t^a)\big\}^{1/2}
    \le
    B_{c,t} + L_{K,t}\,\sigma_{c,t}\sqrt{\frac{D_{c,t}}{N}}
    + L_{K,t}\,b_{c,N,t}.
    \label{eq:theory-optimal-bound}
\end{align}
\end{lemma}
\begin{proof}
The triangle inequality followed by Minkowski's inequality gives
\begin{align}
    \big\{\mathbb E W_2^2(\pi_t^N,\pi_t^a)\big\}^{1/2}
    &\le
    \big\{\mathbb E W_2^2(\pi_t^N,\pi_t^\infty)\big\}^{1/2}
    + \big\{\mathbb E W_2^2(\pi_t^\infty,\pi_t^a)\big\}^{1/2} \\
    &= \big\{\mathbb E W_2^2(\pi_t^N,\pi_t^\infty)\big\}^{1/2}
    + B_{c,t}.
    \label{eq:theory-conditional-error-split}
\end{align}
Because $\mathcal C_t^N$ and $\mathcal C_t^\infty$ have the same
observation record $\bm o_t$, the kernel representation in
\Cref{asm:lipschitz-kernel} gives
\begin{align*}
    \pi_t^N
    = K_t(\bm c_t^{\mathrm{static}},\bm o_t), 
    \quad
    \pi_t^\infty
    = K_t(\bm c_t^{\mathrm{static},\infty},\bm o_t), 
    \quad
    W_2(\pi_t^N,\pi_t^\infty)
    \le L_{K,t}
       \|\bm c_t^{\mathrm{static}}
          -\bm c_t^{\mathrm{static},\infty}\|_c.
\end{align*}
Therefore, by \Cref{asm:feature-concentration},
\begin{align}
    \big\{\mathbb E W_2^2(\pi_t^N,\pi_t^\infty)\big\}^{1/2}
    &\le L_{K,t}
       \big\{\mathbb E\|\bm c_t^{\mathrm{static}}
          -\bm c_t^{\mathrm{static},\infty}\|_c^2\big\}^{1/2} \\
    &\le L_{K,t}
       \sqrt{\sigma_{c,t}^2\frac{D_{c,t}}{N}+b_{c,N,t}^2} \\
    &\le L_{K,t}\left(
       \sigma_{c,t}\sqrt{\frac{D_{c,t}}{N}}+b_{c,N,t}\right).
    \label{eq:theory-conditional-feature-bound}
\end{align}
Combining \cref{eq:theory-conditional-error-split,eq:theory-conditional-feature-bound}
yields \cref{eq:theory-optimal-bound}.
\end{proof}

\begin{lemma}[CFM and Euler error]
\label{lem:numerical}
Let the induced velocity $\widetilde{\bm v}_{\hat\theta,t}$ have population
excess CFM loss as follows
\begin{align}
    \varepsilon_{\mathrm{vel}}
    := \mathcal L_{\mathrm{CFM}}(\widetilde{\bm v}_{\hat\theta,t})
      -\mathcal L_{\mathrm{CFM}}(u) 
    =\int_0^1\mathbb E\big\|
      \widetilde{\bm v}_{\hat\theta,t}
      (\bm Z_\tau,\tau;\bm c_t^{\mathrm{static}},\bm o_t)
      -u_\tau^{\bm c_t^{\mathrm{static}},\bm o_t}(\bm Z_\tau)
      \big\|_2^2\,d\tau .
    \label{eq:theory-numerical-excess-loss}
\end{align}
Under \Cref{asm:tube}:
\begin{enumerate}
\item[(i)] The  CFM error for exact flow satisfies
\begin{align}
    \bigg\{\mathbb E\,W_2^2\big(
        \widehat\Phi^{\mathcal{C}_t^N}_{0,1\#}
        q_{t,0}^{\mathrm{tr};\mathcal{C}_t^N},\,\pi_t^N
    \big)\bigg\}^{1/2}
    &\le e^{\widetilde{L}_{v,z,t}}\sqrt{\varepsilon_{\mathrm{vel}}}
    =: C_{\mathrm{FM}}\sqrt{\varepsilon_{\mathrm{vel}}};
    \label{eq:theory-vel-bound}
\end{align}
\item[(ii)] for every source law $\nu$ supported in $\mathcal N_t$, the
$K$-step Euler error satisfies
\begin{align}
    W_2\big(\widehat\Psi^{\mathcal{C}_t^N}_{K\#}\nu,\,\widehat\Phi^{\mathcal{C}_t^N}_{0,1\#}\nu\big)
    &\le C_{\mathrm{Euler}}K^{-1}.
    \label{eq:theory-euler-bound}
\end{align}
\end{enumerate}
\end{lemma}
\begin{proof}
\begin{enumerate}
\item[(i)] Consider fixed $[\bm c,\bm o]$, we define
\begin{align*}
    \bm X_0=\bm Y_0
    &\sim q_{t,0}^{\mathrm{tr};\bm c,\bm o}, 
    \quad
    \dot{\bm X}_\tau
    =u_\tau^{\bm c,\bm o}(\bm X_\tau), 
    \quad
    \dot{\bm Y}_\tau
    =\widetilde{\bm v}_{\hat\theta,t}
      (\bm Y_\tau,\tau;\bm c,\bm o), 
    \\
    e_\tau(\bm z)
    &:=\big\|u_\tau^{\bm c,\bm o}(\bm z)
      -\widetilde{\bm v}_{\hat\theta,t}
      (\bm z,\tau;\bm c,\bm o)\big\|_2.
\end{align*}
The trajectory difference satisfies
\begin{align*}
    \frac{d}{d\tau}(\bm X_\tau-\bm Y_\tau)
    &={\underbrace{u_\tau^{\bm c,\bm o}(\bm X_\tau)
      -\widetilde{\bm v}_{\hat\theta,t}
      (\bm X_\tau,\tau;\bm c,\bm o)}_{\text{CFM velocity error}}} \\
    &\quad+{\underbrace{\widetilde{\bm v}_{\hat\theta,t}
      (\bm X_\tau,\tau;\bm c,\bm o)
      -\widetilde{\bm v}_{\hat\theta,t}
      (\bm Y_\tau,\tau;\bm c,\bm o)}_{\text{state stability term}}}, \\
    \frac{d}{d\tau}\|\bm X_\tau-\bm Y_\tau\|_2
    &\le e_\tau(\bm X_\tau)
      +\widetilde{L}_{v,z,t}\|\bm X_\tau-\bm Y_\tau\|_2,
\end{align*}
where the second line uses \Cref{asm:tube}.
By Gr\"onwall's inequality,
\begin{align}
    \|\bm X_1-\bm Y_1\|_2
    &\le e^{\widetilde{L}_{v,z,t}}
       \int_0^1e_s(\bm X_s)\,ds, \\
    \mathbb{E}\|\bm X_1-\bm Y_1\|_2^2
    &\le e^{2\widetilde{L}_{v,z,t}}
       \mathbb{E}\!\left[\left(\int_0^1e_s(\bm X_s)\,ds\right)^2\right] \\
    &\le e^{2\widetilde{L}_{v,z,t}}
       \int_0^1\mathbb{E}\big[e_s(\bm X_s)^2\big]\,ds
    = e^{2\widetilde{L}_{v,z,t}}\,\varepsilon_{\mathrm{vel}},
    \label{eq:theory-gronwall-vel}
\end{align}
where we apply Jensen's inequality, 
and \cref{eq:theory-gronwall-vel} so that
\begin{align*}
    \mathcal L(\bm X_s\mid\mathcal C_t^N=[\bm c,\bm o])
    =\rho_s^{\bm c,\bm o}, 
    \quad
    \int_0^1\mathbb E\big[e_s(\bm X_s)^2\big]\,ds
    =\varepsilon_{\mathrm{vel}}.
\end{align*}
The endpoint laws are
\begin{align*}
    \mathcal L(\bm X_1\mid\mathcal C_t^N)
    =\pi_t^N, 
    \quad
    \mathcal L(\bm Y_1\mid\mathcal C_t^N)
    =\widehat\Phi_{0,1\#}^{\mathcal C_t^N}
      q_{t,0}^{\mathrm{tr};\mathcal C_t^N}.
\end{align*}
Therefore,
\begin{align*}
    \mathbb E\,W_2^2\big(
      \widehat\Phi_{0,1\#}^{\mathcal C_t^N}
      q_{t,0}^{\mathrm{tr};\mathcal C_t^N},\pi_t^N
    \big)
    \le \mathbb E\|\bm X_1-\bm Y_1\|_2^2 
    \le e^{2\widetilde L_{v,z,t}}\varepsilon_{\mathrm{vel}}.
\end{align*}

\item[(ii)] Along every exact trajectory $\bm z_\tau$ of the learned ODE,
\Cref{asm:tube} gives a uniform constant $M_{2,t}<\infty$ such that
\begin{align*}
    \sup_{0\le\tau\le1}\|\ddot{\bm z}_\tau\|_2
    \le M_{2,t},
    \quad
    \ddot{\bm z}_\tau
    =\partial_\tau\widetilde{\bm v}_{\hat\theta,t}
    (\bm z_\tau,\tau;\bm c,\bm o)
    +\nabla_{\bm z}\widetilde{\bm v}_{\hat\theta,t}
    (\bm z_\tau,\tau;\bm c,\bm o)
    \,\widetilde{\bm v}_{\hat\theta,t}
    (\bm z_\tau,\tau;\bm c,\bm o).
\end{align*}
The standard global error bound for explicit Euler with step size
$h=K^{-1}$ \citep[Ch.~II.3]{hairer1993solving} gives
\begin{align}
    \big\|\widehat\Psi^{\bm c,\bm o}_{K}(\bm x)-\widehat\Phi^{\bm c,\bm o}_{0,1}(\bm x)\big\|_2
    \le \frac{M_{2,t}}{2\widetilde{L}_{v,z,t}}
        \big(e^{\widetilde{L}_{v,z,t}}-1\big)\,K^{-1}
    =: C_{\mathrm{Euler}}K^{-1},
\end{align}
For every $\nu$ supported in $\mathcal N_t$, we have
\begin{align*}
    W_2^2\big(
      \widehat\Psi^{\mathcal C_t^N}_{K\#}\nu,
      \widehat\Phi^{\mathcal C_t^N}_{0,1\#}\nu
    \big)
    \le \int\big\|
      \widehat\Psi_K^{\mathcal C_t^N}(\bm x)
      -\widehat\Phi_{0,1}^{\mathcal C_t^N}(\bm x)
      \big\|_2^2\,\nu(d\bm x) 
    \le C_{\mathrm{Euler}}^2K^{-2}.
\end{align*}
\end{enumerate}
\end{proof}

\subsubsection[]{Source mismatch ($\varepsilon_{\mathrm{src}}$)}
\label{app:theory-error-source}

\begin{lemma}[Source mismatch]
\label{lem:source}
Define
\begin{align}
    \varepsilon_{\mathrm{src}}
    :=\mathbb{E}\,W_2\big(
      \widehat\Phi^{\mathcal{C}_t^N}_{0,1\#}q_{t,0}^{\mathrm{inf};N},\,
      \widehat\Phi^{\mathcal{C}_t^N}_{0,1\#}q_{t,0}^{\mathrm{tr};\mathcal{C}_t^N}\big).
    \label{eq:theory-source-mismatch}
\end{align}
Under \Cref{asm:tube,asm:baseline},
\begin{align}
    \varepsilon_{\mathrm{src}}
    \le e^{\widetilde{L}_{v,z,t}}
    \big(\underbrace{O(r_{f,N})}_{\text{baseline sampling}}
    +\underbrace{\delta_{\mathrm{gauss},t}}_{\text{source model gap}}\big),
    \qquad
    \delta_{\mathrm{gauss},t}
    :=\mathbb{E}\,W_2\big(\pi_t^f,q_{t,0}^{\mathrm{tr};\mathcal{C}_t^N}\big).
    \label{eq:theory-source-split}
\end{align}
\end{lemma}
\begin{proof}
Fix the record $\mathcal{C}_t^N$. By \Cref{asm:tube} and Gr\"onwall's
inequality, two exact trajectories of \cref{eq:flowda-analysis-ode} started at
$\bm x,\bm x'$ satisfy
\begin{align*}
    \big\|\widehat\Phi^{\mathcal{C}_t^N}_{0,1}(\bm x)
    -\widehat\Phi^{\mathcal{C}_t^N}_{0,1}(\bm x')\big\|_2
    \le e^{\widetilde{L}_{v,z,t}}\|\bm x-\bm x'\|_2 ,
\end{align*}
so the flow map is $e^{\widetilde{L}_{v,z,t}}$-Lipschitz. 
According to \citep[Ch.~6]{villani2008optimal}, 
pushing forward through a Lipschitz map inflates $W_2$ 
by at most its Lipschitz constant, 
hence for all $\mu,\nu$ supported in
$\mathcal{N}_t$,
\begin{align*}
    W_2\big(\widehat\Phi^{\mathcal{C}_t^N}_{0,1\#}\mu,
    \widehat\Phi^{\mathcal{C}_t^N}_{0,1\#}\nu\big)
    \le e^{\widetilde{L}_{v,z,t}}\,W_2(\mu,\nu).
\end{align*}
We take $\mu=q_{t,0}^{\mathrm{inf};N}=\widehat{\pi}_t^{B,f}$ of
\cref{eq:theory-inference-source} and
$\nu=q_{t,0}^{\mathrm{tr};\mathcal{C}_t^N}$, 
then apply triangle inequality through $\pi_t^f$ 
and take expectations to obtain
\begin{align*}
    \varepsilon_{\mathrm{src}}
    &\le e^{\widetilde{L}_{v,z,t}}\,
       \mathbb{E}\,W_2\big(\widehat{\pi}_t^{B,f},
       q_{t,0}^{\mathrm{tr};\mathcal{C}_t^N}\big) \\
    &\le e^{\widetilde{L}_{v,z,t}}\Big(
       \underbrace{\mathbb{E}\,W_2\big(\widehat{\pi}_t^{B,f},\pi_t^f\big)}
         _{=O(r_{f,N})\ \text{by \Cref{asm:baseline}}}
       +\underbrace{\mathbb{E}\,W_2\big(\pi_t^f,
         q_{t,0}^{\mathrm{tr};\mathcal{C}_t^N}\big)}
         _{=\delta_{\mathrm{gauss},t}}\Big).
\end{align*}
\end{proof}

\subsubsection[]{Output calibration ($\varepsilon_{\mathrm{cal},N}$)}
\label{app:theory-error-calibration}

\begin{lemma}[Inflation]
\label{lem:ensemble}
Let $\widetilde\pi_t^a,\widehat\pi_t^a$ be as defined in
\cref{eq:theory-flowef-laws}, and let
\begin{align}
    S_{t,N}^2
    := \frac1N\sum_{i=1}^N\big\|\bm z_{t,1}^{(i)}
       -\hat{\bar{\bm x}}_t^a\big\|_2^2,
    \qquad
    \hat{\bar{\bm x}}_t^a
    = \frac1N\sum_{i=1}^N\bm z_{t,1}^{(i)},
    \label{eq:theory-raw-spread}
\end{align}
be the raw endpoint ensemble variance and mean. Then
\begin{align}
    \varepsilon_{\mathrm{cal},N}
    :=\mathbb{E}\,W_2\big(\widehat{\pi}_t^a,\widetilde{\pi}_t^a\big)
    \le |\alpha_{\mathrm{flow}}-1|\;\mathbb{E}\,S_{t,N},
    \label{eq:theory-inflation-bound}
\end{align}
where equality holds when pairing each $\bm z_{t,1}^{(i)}$ with
$\hat{\bm x}_t^{a,(i)}$ is $W_2$-optimal, i.e., 
\begin{align*}
    W_2^2\big(\widehat\pi_t^a,\widetilde\pi_t^a\big)
    = \frac1N\sum_{i=1}^N
      \big\|\hat{\bm x}_t^{a,(i)}-\bm z_{t,1}^{(i)}\big\|_2^2.
\end{align*}
$\varepsilon_{\mathrm{cal},N}=0$ iff $\alpha_{\mathrm{flow}}=1$, and
$|\alpha_{\mathrm{flow}}-1|\le1$ over the search grid $[0.80,2.00]$ of
\Cref{tab:implementation-settings}.
\end{lemma}
\begin{proof}
Inflation \cref{eq:flowda-inflation} leaves the ensemble mean unchanged and
rescales the anomalies about it, so
\begin{align*}
    \hat{\bm x}_t^{a,(i)}-\bm z_{t,1}^{(i)}
    = (\alpha_{\mathrm{flow}}-1)
      \big(\bm z_{t,1}^{(i)}-\hat{\bar{\bm x}}_t^a\big),
    \qquad i=1,\dots,N.
\end{align*}
By \cref{eq:theory-flowef-laws}, $\widetilde\pi_t^a$ and $\widehat\pi_t^a$ both
put mass $1/N$ on each member, so the joint measure pairing them index by
index,
\begin{align*}
    \gamma := \frac1N\sum_{i=1}^N
    \delta_{(\bm z_{t,1}^{(i)},\,\hat{\bm x}_t^{a,(i)})},
\end{align*}
belongs to the set $\Pi(\widehat\pi_t^a,\widetilde\pi_t^a)$ of couplings, i.e.,
of joint measures with marginals $\widehat\pi_t^a$ and $\widetilde\pi_t^a$.
Since
\begin{align*}
    W_2^2(\mu,\nu)
    = \inf_{\gamma'\in\Pi(\mu,\nu)}
      \int\|\bm x-\bm x'\|_2^2\,\gamma'(d\bm x,d\bm x'),
\end{align*}
we evaluate at $\gamma$ and use \cref{eq:theory-raw-spread}, and then obtain
\begin{align*}
    W_2^2\big(\widehat\pi_t^a,\widetilde\pi_t^a\big)
    &\le \frac1N\sum_{i=1}^N
       \big\|\hat{\bm x}_t^{a,(i)}-\bm z_{t,1}^{(i)}\big\|_2^2 \\
    &= (\alpha_{\mathrm{flow}}-1)^2\,\frac1N\sum_{i=1}^N
       \big\|\bm z_{t,1}^{(i)}-\hat{\bar{\bm x}}_t^a\big\|_2^2
     = (\alpha_{\mathrm{flow}}-1)^2 S_{t,N}^2 .
\end{align*}
\end{proof}

\subsection{Main results}
\label{app:theory-main}

\Cref{thm:flowef} below is the formal version of
\Cref{thm:convergence-informal} in the main text, and its bound
\cref{eq:theory-full-bound} restates \cref{eq:convergence-overview}.
\Cref{tab:theory-dictionary} collects the error terms of that bound: each
$\varepsilon$ term is the contribution underbraced in
\cref{eq:theory-full-bound}, together with the step of \Cref{alg:flowda} that
produces it and the condition under which it vanishes.

\begin{theorem}[FlowEF convergence bound]
\label{thm:flowef}
Let \Cref{asm:stationary,asm:linear-obs,asm:baseline,asm:lipschitz-kernel,%
asm:feature-concentration,asm:tube} hold with $\beta_t=1$, and let the induced
velocity field $\widetilde{\bm v}_{\hat\theta,t}$ of
\cref{eq:theory-induced-velocity} have population excess CFM loss
$\varepsilon_{\mathrm{vel}}$ of \cref{eq:theory-vel-loss}. Run
\Cref{alg:flowda} with
\begin{enumerate}
\item[(i)] the forecast ensemble members as inference source 
as in \cref{eq:ordered-deploy-source}, i.e., 
$q_{t,0}^{\mathrm{inf};N}=\widehat{\pi}_t^{B,f}$ of
\cref{eq:theory-inference-source};
\item[(ii)] $K$ explicit Euler steps as in \cref{eq:ode-euler}; 
\item[(iii)] calibration inflation with factor $\alpha_{\mathrm{flow}}$
\cref{eq:flowda-inflation},
\end{enumerate}
and let $\widetilde{\pi}_t^a$ and $\widehat{\pi}_t^a$ of
\cref{eq:theory-flowef-laws} be the raw transport law and reported
FlowEF analysis law, and $\pi_t^a$ of
\cref{eq:theory-forecast-analysis-laws} the exact analysis law. Then
\begin{align}
    \mathbb{E}\, W_2\big(\widehat{\pi}_t^a,\pi_t^a\big)
    \;\le\;
    &\underbrace{B_{c,t}}_{\varepsilon_{\mathrm{bias}}}
    + \underbrace{L_{K,t}\big(\sigma_{c,t}\sqrt{D_{c,t}/N}+b_{c,N,t}\big)}_{\varepsilon_{\mathrm{feat}}}
    + \underbrace{C_{\mathrm{FM}}\sqrt{\varepsilon_{\mathrm{vel}}}}_{\varepsilon_{\mathrm{fm}}} \notag\\
    &+ \underbrace{C_{\mathrm{Euler}}K^{-1}}_{\varepsilon_{\mathrm{ode}}}
    + \underbrace{e^{\widetilde{L}_{v,z,t}}\big(O(r_{f,N})+\delta_{\mathrm{gauss},t}\big)}_{\varepsilon_{\mathrm{src}}}
    + \underbrace{|\alpha_{\mathrm{flow}}-1|\,\mathbb{E}S_{t,N}}_{\varepsilon_{\mathrm{cal},N}} .
    \label{eq:theory-full-bound}
\end{align}
The six terms are, in order: the conditioning bias (\Cref{lem:bias}), the
static-feature estimation error (\Cref{lem:optimal-rate}), the
flow-matching error (\Cref{lem:numerical}), the Euler discretization error (\Cref{lem:numerical}), 
the source distribution mismatch between the forecast
ensemble and the localized Gaussian training source (\Cref{lem:source}),
and the inflation discrepancy (\Cref{lem:ensemble}).
\end{theorem}

\begin{table}[H]
\centering
\caption{Error terms in \Cref{thm:flowef}.}
\label{tab:theory-dictionary}
\small
\setlength{\tabcolsep}{5pt}
\renewcommand{\arraystretch}{1.15}
\resizebox{\linewidth}{!}{%
\begin{tabular}{lllll}
\toprule
\textbf{Term} & \textbf{Contribution in the bound} & \textbf{Meaning / source} & \textbf{Vanishes when} & \textbf{Reference} \\
\midrule
$\varepsilon_{\mathrm{bias}}$ & $B_{c,t}$ & static conditioning in \cref{eq:static-condition} & $\mathcal{C}_t^\infty$ sufficient & \Cref{lem:bias} \\
$\varepsilon_{\mathrm{feat}}$ & $L_{K,t}(\sigma_{c,t}\sqrt{D_{c,t}/N}+b_{c,N,t})$ & finite-ensemble conditioning channels & $N\to\infty$, baseline exact & \Cref{lem:optimal-rate} \\
$\varepsilon_{\mathrm{fm}}$ & $C_{\mathrm{FM}}\sqrt{\varepsilon_{\mathrm{vel}}}$ & learned velocity / CFM loss & $\varepsilon_{\mathrm{vel}}\to0$ & \Cref{lem:numerical} \\
$\varepsilon_{\mathrm{ode}}$ & $C_{\mathrm{Euler}}K^{-1}$ & $K$-step Euler integration in \cref{eq:ode-euler} & $K\to\infty$ & \Cref{lem:numerical} \\
$\varepsilon_{\mathrm{src}}$ & $e^{\widetilde L_{v,z,t}}(O(r_{f,N})+\delta_{\mathrm{gauss},t})$ & training source vs. inference source & sources matched & \Cref{lem:source} \\
$\varepsilon_{\mathrm{cal},N}$ & $|\alpha_{\mathrm{flow}}-1|\,\mathbb{E}S_{t,N}$ & output inflation in \cref{eq:flowda-inflation} & $\alpha_{\mathrm{flow}}=1$ & \Cref{lem:ensemble} \\
\midrule
\multicolumn{5}{l}{\textbf{Subterms for the error terms}} \\
\midrule
$B_{c,t}$ & $\{\mathbb{E}W_2^2(\pi_t^\infty,\pi_t^a)\}^{1/2}$ & \multicolumn{3}{l}{conditioning bias} \\
$L_{K,t}$ & $W_2(K_t(\bm c,\bm o),K_t(\bm c',\bm o))\le L_{K,t}\|\bm c-\bm c'\|_c$ & \multicolumn{3}{l}{endpoint law sensitivity} \\
$\sigma_{c,t}$ & $\|\bm c_t^{\mathrm{static}}-\mathbb{E}\bm c_t^{\mathrm{static}}\|_{L^2}\le\sigma_{c,t}\sqrt{D_{c,t}/N}$ & \multicolumn{3}{l}{feature fluctuation scale} \\
$b_{c,N,t}$ & $\|\mathbb{E}\bm c_t^{\mathrm{static}}-\bm c_t^{\mathrm{static},\infty}\|_c$ & \multicolumn{3}{l}{systematic feature bias} \\
$D_{c,t}$ & $\text{effective feature dimension}$ & \multicolumn{3}{l}{complexity of the estimated conditioning summary} \\
$C_{\mathrm{FM}}$ & $e^{\widetilde L_{v,z,t}}$ & \multicolumn{3}{l}{flow-matching stability constant} \\
$C_{\mathrm{Euler}}$ & $\frac{M_{2,t}}{2\widetilde L_{v,z,t}}(e^{\widetilde L_{v,z,t}}-1)$ & \multicolumn{3}{l}{Euler discretization constant} \\
$r_{f,N}$ & $\mathbb{E}W_2(\widehat\pi_t^{B,f},\pi_t^f)=O(r_{f,N})$ & \multicolumn{3}{l}{baseline forecast-law rate} \\
$\delta_{\mathrm{gauss},t}$ & $\mathbb{E}W_2(\pi_t^f,q_{t,0}^{\mathrm{tr};\mathcal{C}_t^N})$ & \multicolumn{3}{l}{training--inference source mismatch} \\
$S_{t,N}$ & $\{N^{-1}\sum_i\|\bm z_{t,1}^{(i)}-\hat{\bar{\bm x}}_t^a\|_2^2\}^{1/2}$ & \multicolumn{3}{l}{raw endpoint spread} \\
\bottomrule
\end{tabular}%
}
\end{table}

\begin{proof}
By triangle inequality, we have 
\begin{align}
    W_2\big(\widehat{\pi}_t^a,\pi_t^a\big)
    &\le \underbrace{W_2\big(\widehat{\pi}_t^a,\widetilde{\pi}_t^a\big)}_{\text{(i) inflation}}
    + \underbrace{W_2\big(\widetilde{\pi}_t^a,
      \widehat\Phi^{\mathcal{C}_t^N}_{0,1\#}q_{t,0}^{\mathrm{inf};N}\big)}_{\text{(ii) Euler}} \notag\\
    &\quad+ \underbrace{W_2\big(\widehat\Phi^{\mathcal{C}_t^N}_{0,1\#}q_{t,0}^{\mathrm{inf};N},
      \widehat\Phi^{\mathcal{C}_t^N}_{0,1\#}q_{t,0}^{\mathrm{tr};\mathcal{C}_t^N}\big)}_{\text{(iii) source}}
    + \underbrace{W_2\big(\widehat\Phi^{\mathcal{C}_t^N}_{0,1\#}q_{t,0}^{\mathrm{tr};\mathcal{C}_t^N},
      \pi_t^N\big)}_{\text{(iv) flow matching}}
    + \underbrace{W_2\big(\pi_t^N,\pi_t^a\big)}_{\text{(v) conditioning}}.
    \label{eq:theory-chain}
\end{align}
We then bound the five error terms as follows.
\begin{enumerate}
\item[(i)] \Cref{lem:ensemble} implies the pathwise bound
$W_2(\widehat{\pi}_t^a,\widetilde{\pi}_t^a)
\le|\alpha_{\mathrm{flow}}-1|S_{t,N}$, whose expectation is
$\varepsilon_{\mathrm{cal},N}$.
\item[(ii)] By \cref{eq:ordered-deploy-source} the $N$ forecast members are
transported exactly once each, so
$\widetilde{\pi}_t^a=\widehat\Psi^{\mathcal{C}_t^N}_{K\#}
q_{t,0}^{\mathrm{inf};N}$ with no Monte Carlo error between the source and
the reported members. Then, with \cref{eq:theory-euler-bound} with
$\nu=q_{t,0}^{\mathrm{inf};N}$, we have
\begin{align*}
    W_2\big(\widetilde{\pi}_t^a,
    \widehat\Phi^{\mathcal{C}_t^N}_{0,1\#}q_{t,0}^{\mathrm{inf};N}\big)
    \le C_{\mathrm{Euler}}K^{-1}
    = \varepsilon_{\mathrm{ode}} .
\end{align*}
\item[(iii)] \Cref{lem:source} implies, in expectation,
$\varepsilon_{\mathrm{src}}
\le e^{\widetilde{L}_{v,z,t}}(O(r_{f,N})+\delta_{\mathrm{gauss},t})$.
\item[(iv)] \cref{eq:theory-vel-bound} implies 
\begin{align*}
    \Big\{\mathbb{E}\,W_2^2\big(
    \widehat\Phi^{\mathcal{C}_t^N}_{0,1\#}q_{t,0}^{\mathrm{tr};\mathcal{C}_t^N},
    \pi_t^N\big)\Big\}^{1/2}
    \le C_{\mathrm{FM}}\sqrt{\varepsilon_{\mathrm{vel}}}
    = \varepsilon_{\mathrm{fm}}.
\end{align*}
\item[(v)] \cref{eq:theory-optimal-bound} implies 
\begin{align*}
    \big\{\mathbb{E}\,W_2^2\big(\pi_t^N,\pi_t^a\big)\big\}^{1/2}
    \le B_{c,t}+L_{K,t}\Big(\sigma_{c,t}\sqrt{D_{c,t}/N}+b_{c,N,t}\Big)
    = \varepsilon_{\mathrm{bias}}+\varepsilon_{\mathrm{feat}} .
\end{align*}
\end{enumerate}
Then we take expectations in \cref{eq:theory-chain}, 
bound each term by the above, 
and apply $\mathbb{E}\,W_2\le\{\mathbb{E}\,W_2^2\}^{1/2}$ to (iv) and (v), 
to obtain the results in \cref{eq:theory-full-bound}.
\end{proof}

\begin{corollary}[Conditioning-limited bound]
\label{cor:matched}
Under the assumptions in \Cref{thm:flowef} and, in addition,
\begin{enumerate}
\item[(i)] zero excess CFM loss, $\varepsilon_{\mathrm{vel}}=0$;
\item[(ii)] exact integration of \cref{eq:flowda-analysis-ode};
\item[(iii)] no inflation, i.e., $\alpha_{\mathrm{flow}}=1$; 
\item[(iv)] same sources, i.e.\ inference draws from
\cref{eq:locgauss-dist} rather than \cref{eq:ordered-deploy-source}, so that
$q_{t,0}^{\mathrm{inf};N}=q_{t,0}^{\mathrm{tr};\mathcal{C}_t^N}$.
\end{enumerate}
Define 
$q_{\theta^\star,t}^N(d\bm{x}):=\mathbb{P}(\bm{z}_{t,1}\in d\bm{x}\mid
\mathcal{C}_t^N)$ for the resulting endpoint law, then we have
\begin{align}
    \mathbb{E}\, W_2\big(q_{\theta^\star,t}^N,\pi_t^a\big)
    \le \underbrace{B_{c,t}}_{\varepsilon_{\mathrm{bias}}}
    + \underbrace{L_{K,t}\Big(\sigma_{c,t}\sqrt{D_{c,t}/N}
      +b_{c,N,t}\Big)}_{\varepsilon_{\mathrm{feat}}} .
    \label{eq:theory-special-bound}
\end{align}
If moreover $\pi_t^\infty=\pi_t^a$ a.s.\ and the baseline is consistent, so
that $B_{c,t}=0$ and $b_{c,N,t}\to0$, and we have 
\begin{align*}
    \mathbb{E}\,W_2\big(q_{\theta^\star,t}^N,\pi_t^a\big)
    = O\Big(L_{K,t}\sigma_{c,t}\sqrt{D_{c,t}/N}\Big).
\end{align*}
\end{corollary}
\begin{proof}
Assumptions (i)--(iv) remove the first four error terms of
\cref{eq:theory-chain}:
\begin{align*}
    \varepsilon_{\mathrm{cal},N}
    =|\alpha_{\mathrm{flow}}-1|\,\mathbb{E}S_{t,N}=0,
    \qquad
    \varepsilon_{\mathrm{ode}}=0,
    \qquad
    \varepsilon_{\mathrm{src}}=0,
    \qquad
    \varepsilon_{\mathrm{fm}}
    =C_{\mathrm{FM}}\sqrt{\varepsilon_{\mathrm{vel}}}=0.
\end{align*}
By (i), (ii) and (iv),
\Cref{lem:optimal} identifies the endpoint law with $\pi_t^N$, i.e.\
$q_{\theta^\star,t}^N=\pi_t^N$, so only term (v) of \cref{eq:theory-chain}
remains nonzero and \cref{eq:theory-optimal-bound} gives
\cref{eq:theory-special-bound}. 
\end{proof}

\end{document}